\documentclass[twoside,11pt]{article}
\RequirePackage{xcolor}

\usepackage{blindtext}

\usepackage[preprint]{jmlr2e}
\usepackage{amsmath}
\usepackage{mathtools}
\mathtoolsset{showonlyrefs} 
\usepackage{verbatim}
\usepackage{tabularx}
\usepackage[export]{adjustbox}
\graphicspath{{fig/}}
\usepackage{wrapfig}
\usepackage{float}

\DeclareMathOperator*{\tr}{tr}

\DeclareMathOperator*{\diag}{diag}

\DeclareMathOperator*{\VEC}{vec}

\renewcommand{\d}[1]{\ensuremath{\operatorname{d}\!{#1}}}
\newcommand{\FR}{\operatorname{FR}}

\usepackage{lastpage}
\usepackage{hyperref}

\newtheorem{assumption}{Assumption}

\usepackage{algorithm}
\usepackage[noend]{algpseudocode}
\algnewcommand{\Output}{\item[\algorithmicoutput]}
\algnewcommand{\algorithmicoutput}{\textbf{Output:}}

\newcommand{\calA}{{\cal A}}
\newcommand{\calB}{{\cal B}}

\newcommand{\calI}{{\cal I}}

\newcommand{\calM}{{\cal M}}
\newcommand{\calN}{{\cal N}}

\newcommand{\calP}{{\cal P}}

\newcommand{\calZ}{{\cal Z}}

\newcommand{\bfb}{\mathbf{b}}

\newcommand{\bff}{\mathbf{f}}

\newcommand{\bfl}{\mathbf{l}}
\newcommand{\bfm}{\mathbf{m}}
\newcommand{\bfn}{\mathbf{n}}

\newcommand{\bfu}{\mathbf{u}}

\newcommand{\bfw}{\mathbf{w}}
\newcommand{\bfx}{\mathbf{x}}
\newcommand{\bfy}{\mathbf{y}}
\newcommand{\bfz}{\mathbf{z}}

\newcommand{\bfalpha}{\boldsymbol{\alpha}}

\newcommand{\bfgamma}{\boldsymbol{\gamma}}

\newcommand{\bfeta}{\boldsymbol{\eta}}
\newcommand{\bftheta}{\boldsymbol{\theta}}

\newcommand{\bfmu}{\boldsymbol{\mu}}

\newcommand{\bfphi}{\boldsymbol{\phi}}

\newcommand{\bfxi}{\boldsymbol{\xi}}

\newcommand{\bbE}{\mathbb{E}}

\newcommand{\bbI}{\mathbb{I}}

\newcommand{\bbR}{\mathbb{R}}

\firstpageno{1}

\begin{document}

\title{Variational Formulation of Particle Flow}

\author{\name Yinzhuang Yi \email yiyi@ucsd.edu \\
       \addr Contextual Robotics Institute\\
       University of California, San Diego\\
       La Jolla, CA 92093, USA
       \AND
       \name Jorge Cort\'es \email cortes@ucsd.edu \\
       \addr Contextual Robotics Institute\\
       University of California, San Diego\\
       La Jolla, CA 92093, USA
       \AND
       \name Nikolay Atanasov \email natanasov@ucsd.edu \\
       \addr Contextual Robotics Institute\\
       University of California, San Diego\\
       La Jolla, CA 92093, USA}

\editor{Editor.}

\maketitle


\begin{abstract}
    This paper provides a formulation of the log-homotopy particle flow from the perspective of variational inference. We show that the transient density used to derive the particle flow follows a time-scaled trajectory of the Fisher-Rao gradient flow in the space of probability densities. The Fisher-Rao gradient flow is obtained as a continuous-time algorithm for variational inference, minimizing the Kullback-Leibler divergence between a variational density and the true posterior density. When considering a parametric family of variational densities, the function space Fisher-Rao gradient flow simplifies to the natural gradient flow of the variational density parameters. By adopting a Gaussian variational density, we derive a Gaussian approximated Fisher-Rao particle flow and show that, under linear Gaussian assumptions, it reduces to the Exact Daum and Huang particle flow. Additionally, we introduce a Gaussian mixture approximated Fisher-Rao particle flow to enhance the expressive power of our model through a multi-modal variational density. Simulations on low- and high-dimensional estimation problems illustrate our results.
\end{abstract}

\begin{keywords}
    Variational Inference, Fisher-Rao Gradient Flow, Particle Flow, Bayesian Inference, Nonlinear Filtering
\end{keywords}


\section{Introduction}
This work aims to unveil a relationship between variational inference \citep{jordan1999introduction} and log-homotopy particle flow \citep{daum2007nonlinear, daum2009nonlinear} by providing a variational formulation of particle flow. These techniques consider Bayesian inference problems in which we start with a prior density function $p(\bfx)$, and given an observation $\bfz$ and associated likelihood density function $p(\bfz | \bfx)$, we aim to compute a posterior density function $p(\bfx|\bfz)$, following Bayes' rule. 

Variational inference \citep{jordan1999introduction, wainwright2008graphical, blei2017variational} is a method for approximating complex posterior densities $p(\bfx|\bfz)$ by optimizing a simpler variational density $q(\bfx)$ to closely match the true posterior. The log-homotopy particle flow \citep{daum2007nonlinear, daum2009nonlinear, daum2010exact} approximates the posterior density $p(\bfx|\bfz)$ using a set of discrete samples $\{\bfx_i\}_i$ called particles. The particles are initially sampled from the prior density $p(\bfx)$ and propagated according to a particle dynamics function satisfying the Fokker-Planck equation \citep{jazwinski2007stochastic}. 

Uncovering a variational formulation of particle flow offers several advantages. It provides an expression for the approximate density after particle propagation, thereby enhancing the expressive capability of the particle flow. It offers an alternative to linearization for handling nonlinear observation models. Additionally, by using a mixture density as the variational density $q(\bfx)$, the variational formulation can capture multi-modal posterior densities, making it well-suited for inference tasks that need to capture the likelihood of several possible outcomes. To put these advantages in context and clarify the relationship with existing methods, we review traditional Bayesian filtering approaches next.

\subsection{Related Work}

\paragraph{Filtering}

Traditional filtering approaches to Bayesian inference include the celebrated Kalman filter \citep{kalman1960new}, which makes linear Gaussian assumptions, and its extensions to nonlinear observation models --- the extended Kalman filter (EKF) \citep{anderson2005optimal}, which linearizes the observation model, and the unscented Kalman filter (UKF) \citep{julier1995new}, which propagates a set of discrete samples from the prior, called sigma points, through the nonlinear observation model. While the EKF and UKF consider nonlinear observation models, they use a single Gaussian to approximate the posterior and, hence, cannot capture multi-modal behavior. To handle multi-modal posteriors, the Gaussian sum filter \citep{alspach1972nonlinear} approximates the posterior density using a weighted sum of Gaussian components. Alternatively, particle filters, or sequential Monte Carlo methods, have been proposed, in which the posterior is approximated by a set of particles. The bootstrap particle filter \citep{gordon1993novel} is a classical example, which draws particles from a proposal distribution, commonly chosen to be the prior, and updates their weights using the observation likelihood. These methods are developed for online Bayesian inference, where timely execution of the update steps is required to ensure fast performance. Among the filtering approaches to Bayesian inference, particle filters are of particular interest, as they employ the same particle-based approximation of the posterior as the particle flow method studied in this paper. However, particle filters are known to suffer from particle degeneracy, particularly when the state dimension is high or the measurements are highly informative \citep{bickel2008sharp}. This challenge can be alleviated using posterior sampling methods, which we introduce next.

\paragraph{Posterior Sampling}

Posterior sampling \citep{uhlenbeck1930brownian, metropolis1953equation, hastings1970monte, alan1990sampling, luke1994markov, neal2011mcmc, betancourt2015hamiltonian, ma2015complete} aims to generate statistically accurate samples from the posterior distribution. Unlike filter approaches, posterior sampling approaches typically do not require real-time performance. Traditional posterior sampling approaches include the Metropolis–Hastings sampler \citep{metropolis1953equation, hastings1970monte}, which employs an accept–reject mechanism based on a proposal distribution whose choice strongly affects efficiency, particularly in high-dimensional settings, and the Gibbs sampler \citep{alan1990sampling}, which generates samples from full conditional distributions but can suffer from slow mixing when strong dependencies exist among variables. An early and comprehensive treatment of Markov chain methods for posterior sampling is provided by \citet{luke1994markov}. Recent developments in sampling have led to samplers based on continuous dynamics, including Hamiltonian Monte Carlo (HMC) \citep{neal2011mcmc, betancourt2015hamiltonian}, which generates particles by simulating a hypothetical Hamiltonian system, leveraging the first-order gradient information of the posterior, and Langevin diffusion \citep{uhlenbeck1930brownian}, which constructs a stochastic differential equation whose stationary distribution coincides with the target posterior. Discretizations of Langevin diffusion lead to gradient-based sampling algorithms such as the unadjusted Langevin algorithm (ULA) \citep{grenander1994representations} and the Metropolis-adjusted Langevin algorithm (MALA) \citep{gareth1996exponential}, which incorporate gradient information to improve mixing relative to random-walk proposals. \citet{ma2015complete} provides a unifying SDE framework for a broad class of continuous-dynamics sampling methods, including Langevin- and Hamiltonian-based samplers, by demonstrating that they arise as special cases of a parametrized class of stochastic differential equations with a prescribed invariant distribution. The log-homotopy particle flow \citep{daum2007nonlinear, daum2009nonlinear, daum2010exact}, which transports particles via a particle dynamics function satisfying the Fokker–Planck equation \citep{jazwinski2007stochastic}, can also be viewed as a sampling method based on continuous dynamics. The samples generated by posterior sampling methods can be used to enhance the performance of particle filters. The particle flow particle filter \citep{li2017particle} uses a set of particles obtained via the log-homotopy particle flow as samples from a proposal distribution and updates the particle weights based on the unnormalized posterior and the proposal distribution. This guarantees asymptotic consistency even if the particle ensemble provides an inaccurate estimation of the posterior. However, the particle flow particle filter's reliance on a Gaussian prior density limits its effectiveness for multi-modal densities. To address this, Gaussian mixture models have been incorporated into particle flow particle filter variants \citep{pal2017gaussian, pal2018particle, li2019invertible, zhang2024multisensor}.

\paragraph{Variational Inference}
The pioneering work of \citet{jordan1998fokker} established a connection between variational inference and the Fokker-Planck equation \citep{jazwinski2007stochastic}, which serves as the theoretical foundation for the posterior sampling methods discussed above. Variational inference views Bayesian inference as an optimization problem aiming to minimize the Kullback–Leibler divergence between a variational density and the posterior density. The optimization problem can be solved with a closed-form coordinate descent algorithm only when the likelihood function is suitably structured \citep{wainwright2008graphical}. For more complex models, computing the necessary gradients can become intractable but automatic differentiation and stochastic gradient descent can be used to overcome this challenge \citep{hoffman2013stochastic, ranganath2014black}. To improve the efficiency of variational inference, natural gradient descent can be applied \citep{hoffman2013stochastic, lin2019fast, martens2020insight, huang2022inverse, khan2023bayesian}. Variational inference can be formulated directly on the space of probability density functions, i.e., instead of a parametric variational density, we consider the optimization variable as a density function belonging to an infinite-dimensional manifold. In this case, the Wasserstein gradient flow \citep{jordan1998fokker, ambrogioni2018wvi, lambert2022wass} is used to formulate the gradient dynamics on the space of probability densities with respect to the Wasserstein metric. Particle-based variational inference methods represent the approximating distribution using a set of particles. This includes the Stein variational gradient descent \citep{liu2016stein} and diffusion-based variational inference \citep{doucet2022scorebased, geffner2023langevin}. Our work here formulates the gradient dynamics on the space of probability densities with respect to the Fisher-Rao metric \citep{amari2000methods}. The gradient dynamics introduces a time rate of change of the variational density, which is further utilized to derive the corresponding particle dynamics function.

\paragraph{Feedback Particle Filter}

Variational inference provides an optimization-based perspective on Bayesian inference, which can be leveraged to derive novel filtering methods. The feedback particle filter \citep{feedbackpf2011acc, feedbackpf2011cdc, feedbackpf2013tac, feedbackpf2015poisson} formulates Bayesian inference as an optimal control problem on the space of probability densities, aiming to obtain a particle dynamics function satisfying the Kushner-Stratonovich equation \citep{jazwinski2007stochastic} specified by the evolution of the time-varying posterior density. Utilizing this optimal control formulation, \citet{feedbackpf2017control} introduced a controlled particle filter, where the time-varying posterior density coincides with the ones used to derive the log homotopy particle flow \citep{daum2007nonlinear, daum2009nonlinear}. The density evolution induced by the controlled particle filter can be interpreted as the gradient flow for the expected negative log likelihood density with respect to the Fisher-Rao metric. The variational interpretation introduced in \citet{feedbackpf2017control} can be viewed as a special case of the variational inference formulation studied in this paper, obtained when the initial variational density is chosen as the prior and an appropriate time-scaling function is introduced. Under this construction, the gradient flow in \citet{feedbackpf2017control} progressively conditions the likelihood on an initial density. Consequently, if the initial density is not chosen to coincide with the prior, this gradient flow will not, in general, converge to the posterior. In contrast, the gradient flow introduced here drives any initial density toward the true posterior density and therefore ensures convergence regardless of the initial density.

\subsection{Paper Contributions and Organization}

The main contribution of this paper is the development of a variational formulation of the log-homotopy particle flow \citep{daum2007nonlinear, daum2009nonlinear, daum2010exact}. We show that the time-varying posterior used in the particle flow derivation can be viewed as a time-scaled trajectory of the Fisher-Rao gradient flow minimizing the Kullback-Leibler divergence. This variational approach removes the Gaussian assumptions on the prior and the likelihood, allowing for flexibility to approach estimation problems with any prior-likelihood pair. Utilizing this flexibility, we propose an approximated Gaussian mixture particle flow to capture multi-modal behavior in the posterior. Finally, we formulate several identities enabling inverse- and derivative-free implementation of the particle flow induced by the Fisher-Rao gradient flow.

The paper is organized as follows. In Section~\ref{sec: prelim}, we provide the necessary background on particle flow and variational inference. Section~\ref{sec: transient_and_fisher} presents our first main result, identifying the transient density used to derive particle flow as a time-scaled trajectory of the Fisher-Rao gradient flow. Building on this, Section~\ref{sec: gaussian_flow} specializes the Fisher-Rao gradient flow to Gaussian variational densities and introduces the associated Gaussian Fisher-Rao particle flow, and demonstrates that, under linear Gaussian assumptions, the Gaussian Fisher-Rao particle flow reduces to the Exact Daum and Huang particle flow \citep{daum2010exact}. Section~\ref{sec: gaussian_mixture_approx}, then, proposes an approximated Gaussian mixture Fisher-Rao particle flow by choosing the variational density as a Gaussian mixture density. Section~\ref{sec: derivative_inverse_free} develops an inverse- and derivative-free formulation of the proposed particle flows and shows that they preserve Gauss-Hermite particles and the Mahalanobis distance. Section~\ref{sec: Evaluation} provides simulations on low- and high-dimensional estimation problems to illustrate the performance of the proposed particle flows. Finally, Section~\ref{sec: Application} generalizes the proposed method to non-Gaussian variational densities by combining the Fisher-Rao particle flow with normalizing flows \citep{rezende2015variational}, and presents numerical experiments demonstrating the effectiveness of the resulting approach.

\paragraph{Notation}
We let $\bbR$ denote the real numbers. We use boldface letters to denote vectors and capital letters to denote matrices. We use $|\cdot|$ to denote the absolute value. We denote the probability density function (PDF) of a Gaussian random vector with mean $\bfmu$ and covariance $\Sigma$ by $p_{\calN}( \ \cdot \ ; \bfmu, \Sigma)$. We use the notation $f \in C^{\infty}$ to indicate that $\bfx \mapsto f(\bfx)$ is a smooth (infinitely differentiable) function. We use $\nabla_{\bfx} f(\bfx)$ to denote the gradient of $f(\bfx)$ and $\nabla_{\bfx} \cdot f(\bfx)$ to denote the divergence of $f(\bfx)$. For a random variable $\bfy$ with PDF $p(\bfy)$, we use $\bbE_{p(\bfy)}[ f(\bfy) ]$ to denote the expectation of $f(\bfy)$. We follow the numerator layout when calculating derivatives. For example, for a function $f: \bbR^n \to \bbR^m$, we have $\frac{\partial f(\bfx)}{\partial \bfx} \in \bbR^{m \times n}$. For a matrix $A \in \bbR^{n \times n}$, we use $\tr(A)$ to denote its trace and $\det(A)$ to denote its determinant. We let $\bbI_k(\omega)$ denote the indicator function, which is $1$ when $\omega = k$ and $0$ otherwise.


\section{Background}
\label{sec: prelim}

We consider a Bayesian estimation problem where $\bfx \in \bbR^n$ is a random variable with a prior PDF $p(\bfx)$. Given a measurement $\bfz \in \bbR^m$ with a known measurement likelihood $p(\bfz | \bfx)$, the posterior PDF of $\bfx$ conditioned on $\bfz$ is determined by Bayes' theorem~\citep{bayes1763}:
\begin{equation}
\label{prelim: bayes_rule}
	p(\bfx | \bfz) = \frac{p(\bfz | \bfx) p (\bfx)}{p(\bfz)}, 
\end{equation}
where $p(\bfz)$ is the marginal measurement PDF computed as $p(\bfz) = \int p(\bfz | \bfx) p(\bfx) \d \bfx$. Calculating the posterior PDF in \eqref{prelim: bayes_rule} is often intractable since the marginal measurement PDF $p(\bfz)$ cannot typically be computed in closed form, except when the prior $p(\bfx)$ and the likelihood $p(\bfz|\bfx)$ are a conjugate pair \citep{gelman1995bayesian}. To calculate the posterior of non-conjugate prior and likelihood, approximation methods are needed. This problem can be approached using two classes of methods, which we review next: particle-based methods \citep{doucet2009tutorial}, where the posterior is approximated using a set of weighted particles, and variational inference methods \citep{blei2017variational}, where the posterior is approximated using a parametric variational PDF. 

\subsection{Particle-Based Inference}

We review two particle-based inference methods: discrete-time and continuous-time particle filters. The major difference between these methods, as their name suggests, is the type of time evolution of their update step. A concrete example of a discrete-time particle filter method is the bootstrap particle filter \citep{gordon1993novel}, which employs a single-step update on the particle weights once a measurement is received. On the other hand, the particle flow \citep{daum2007nonlinear} updates the particle location following an ordinary differential equation upon receiving a measurement.

\subsubsection{Discrete-Time Formulation: Particle Filter}
We review the bootstrap particle filter as an example of a discrete-time particle-based inference method but note that other methods exist \citep{sarkka2023bayesian}. The bootstrap particle filter is a classical particle-based inference method that approximates the Bayes' posterior \eqref{prelim: bayes_rule} using a set of weighted particles, where we draw $N$ independent identically distributed particles $\{ \bfx_i \}_{i=1}^N$ from the prior $\bfx_i \sim p(\bfx)$. The particle weights are determined by the following equation:
\begin{equation}
   w_i = \frac{p(\bfz | \bfx_i)}{\sum_{j=1}^N p(\bfz | \bfx_j)}.
\end{equation}
The empirical mean and covariance of the posterior can be calculated using the weighted particles:
\begin{equation}
    \tilde{\bfmu} = \sum_{i=1}^{N} w_i \bfx_i, \qquad \tilde{\Sigma} = \sum_{i=1}^{N} w_i (\bfx_i - \tilde{\bfmu}) (\bfx_i - \tilde{\bfmu})^{\top}. 
\end{equation}
One would immediately notice that since the particles are sampled from the prior and their positions remain unchanged, if the high-probability region of the posterior differs significantly from the prior, then the particle set may either poorly represent the posterior or result in most particles having negligible weights. This challenge is known as \emph{particle degeneracy} \citep{bickel2008sharp}. The issue of particle degeneracy arises from a significant mismatch between the density from which the particles are drawn and the Bayes' posterior. To alleviate this issue, particles should be sampled from a distribution that more closely aligns with the posterior. This leads to the sequential importance sampling particle filter \citep{genshiro1996monte}, where particles are sampled from a proposal density. However, constructing an effective proposal density that accurately aligns with the posterior is a challenging task. An alternative is to move particles sampled from the prior to regions with high likelihood, which leads to the particle flow method, explained next.

\subsubsection{Continuous-Time Formulation: Particle Flow}
\label{subsec: pfpf}
We review the particle flow proposed in \cite{daum2007nonlinear, daum2009nonlinear,daum2010exact} as an example of a continuous-time particle-based inference method. This method seeks to avoid the particle degeneracy faced by the bootstrap particle filter. The particles are moved based on a deterministic or stochastic differential equation, while the particle weights stay unchanged. We follow the exposition in \cite{crouse2019consideration}, where the interested reader can find a detailed derivation. Before introducing the particle flow, it is important to understand the Liouville equation \citep{wibisono2017information}, as it provides the necessary background for its derivation. Consider a random process $\bfx_t \in \bbR^n$ that satisfies a drift-diffusion stochastic differential equation:
\begin{equation} \label{prelim: particle_fpe}
    \d \bfx_t = \phi(\bfx_t, t) \d t + B(\bfx_t, t) \d \bfw_t, 
\end{equation}
where $\phi(\bfx_t, t) \in \bbR^n$ is a drift function, $B(\bfx_t, t) \in \bbR^{n \times m}$ is a diffusion matrix function and $\bfw_t$ is standard Brownian motion. The PDF $p(\bfx; t)$ of $\bfx_t$ evolves according to the Fokker–Planck equation \citep{jazwinski2007stochastic}:
\begin{equation} \label{prelim: fpe}
    \frac{\partial p(\bfx; t)}{\partial t} = - \nabla_{\bfx} \cdot \big( p(\bfx; t) \phi(\bfx, t) \big) + \frac{1}{2} \nabla_{\bfx} \cdot \Big( \nabla_{\bfx} \cdot \big( B(\bfx, t) B(\bfx, t)^{\top} p(\bfx; t) \big) \Big).
\end{equation}
In this work, we only consider the case when the diffusion term $B$ is zero. In such a case, the random process is described by the following ordinary differential equation:
\begin{equation}
\label{prelim: particle_dynamics}
	\frac{\d \bfx_t}{\d t} = \phi(\bfx_t, t),
\end{equation}
where we term $\phi(\bfx_t, t)$ the \emph{particle dynamics function}. The Fokker-Planck equation then reduces to the Liouville equation \citep{wibisono2017information}:
\begin{equation}
\label{prelim: liouville}
	\frac{\partial p(\bfx; t)}{\partial t} = - \nabla_{\bfx} \cdot \big( p(\bfx; t) \phi(\bfx, t) \big).
\end{equation}
In other words, for a particle evolving according to \eqref{prelim: particle_dynamics} with initialization $\bfx_0$ sampled from PDF $p_0(\bfx)$, at a later time $t$, the particle is distributed according to the PDF $p(\bfx; t)$, which satisfies \eqref{prelim: liouville} with $p(\bfx; t = 0) = p_0(\bfx)$. Conversely, if the time evolution of a PDF $p(\bfx; t)$ is specified, the corresponding particle dynamics can be determined by finding a particle dynamics function $\phi(\bfx, t)$ that satisfies \eqref{prelim: liouville}. This observation facilitates the development of an alternative approach to particle filtering, where particles are actively moved towards regions of higher probability. We make the following assumptions on the particle dynamics function \citep{ambrosio2005gradient}.
\begin{assumption}[Regularity and Mass Conservation Assumptions]\label{assum: regularity}
    The particle dynamics function $\bfphi(\bfx, t)$ in \eqref{prelim: liouville} satisfies the following regularity assumptions for every compact set $B \subset \bbR^n$ and time horizon $T$:
    \begin{equation}
        \int_{0}^{T} \text{Lip}(\bfphi(\bfx, t), B) + \sup_{\bfx \in B} \| \bfphi(\bfx, t) \| \d t < \infty, \quad \int_{0}^{T} \int \| \bfphi(\bfx, t) \|^2 p(\bfx; t) \; \d \bfx \d t < \infty,
    \end{equation}
    where $\text{Lip}(\bfphi(\bfx, t), B)$ denotes the Lipschitz constant for function $\bfphi(\bfx, t)$ in set $B$. In addition, we assume the following non-flux boundary condition holds for all $t \in [0, T]$:
    \begin{equation}
        \lim_{r \to \infty} \oint_{\partial \calB(r)} p(\bfx; t)  \bfphi^{\top}(\bfx, t) \hat{\bfn}_{r}(\bfx) \d \bfx = 0,
    \end{equation}
    where $\calB(r)$ is the Euclidean ball centered at $\mathbf{0}$ with radius $r$ and $\hat{\bfn}_{r}(\bfx)$ is the outward unit normal vector to the boundary $\partial \calB(r)$.
\end{assumption}
This assumption corresponds to the hypothesis in \citet[Proposition~8.1.8]{ambrosio2005gradient}, which ensures the well-posedness of the associated particle dynamics and the Liouville equation on the interval $[0, T]$. The non-flux boundary condition guarantees conservation of total mass, so that $p(\bfx;t)$ remains a probability density function for all $t \in [0, T]$. To derive the particle flow, we first take the natural logarithm of both sides of Bayes' theorem~\eqref{prelim: bayes_rule}:
\begin{equation}
    \log(p(\bfx | \bfz)) = \log(p(\bfz | \bfx)) + \log((p(\bfx)) - \log(p(\bfz)). 
\end{equation}
Our objective is to find a particle dynamics function $\phi(\bfx, t)$ such that, if a particle is sampled from the prior and evolves according to \eqref{prelim: particle_dynamics}, then it will be distributed according to the Bayes' posterior at a certain later time. Building on the discussion above, finding this particle dynamics function requires specifying a transformation from the prior to the posterior. To specify the transformation, we introduce a pseudo-time parameter $0 \leq \lambda \leq 1$ and define a log-homotopy of the form:
\begin{equation}
    \log(p(\bfx | \bfz; \lambda)) = \lambda \log(p(\bfz | \bfx)) + \log(p(\bfx)) - \log(p(\bfz; \lambda)), 
\end{equation}
where the marginal measurement density $p(\bfz; \lambda)$ has been parameterized by the pseudo-time $\lambda$ so that $p(\bfx | \bfz; \lambda)$ is a valid PDF for all values of $\lambda$. This defines the following \emph{transient density} describing the particle flow process:
\begin{equation}
\label{prelim: log_bayes_rule}
	p(\bfx | \bfz; \lambda) = \frac{p(\bfz | \bfx)^{\lambda} p (\bfx)}{p(\bfz; \lambda)}, \qquad p(\bfz; \lambda) = \int p(\bfz | \bfx)^{\lambda} p(\bfx) \d \bfx.
\end{equation}
The transient density \eqref{prelim: log_bayes_rule} defines a smooth and continuous transformation from the prior $p(\bfx | \bfz; 0) = p(\bfx)$ to the posterior $p(\bfx | \bfz; 1) = p(\bfx | \bfz)$. Based on our discussion above, in order to obtain a particle flow describing the evolution of the transient density from the prior to the posterior, we need to find a particle dynamics function $\bfphi(\bfx, \lambda)$ such that the following equation holds:
\begin{equation}
\label{prelim: liouville_condition}
	\frac{\partial p(\bfx | \bfz; \lambda)}{\partial \lambda} = - \nabla_{\bfx} \cdot \big( p(\bfx | \bfz; \lambda) \bfphi(\bfx, \lambda) \big), 
\end{equation}
where $p(\bfx | \bfz; \lambda)$ is the transient density given in \eqref{prelim: log_bayes_rule}. 

\begin{assumption}[Linear Gaussian Assumptions] \label{assumption:linear_Gaussian} 
    The prior PDF is Gaussian, given by $p(\bfx) = p_{\calN}(\bfx; \hat{\bfx}, P)$ with mean $\hat{\bfx} \in \bbR^n$ and covariance $P \in \bbR^{n \times n}$, and the likelihood is also Gaussian, given by $p(\bfz | \bfx) = p_{\calN} (\bfz; H \bfx, R)$ with mean $H \bfx \in \bbR^m$ and covariance $R \in \bbR^{m \times m}$.
\end{assumption}

Finding a particle dynamics function $\bfphi(\bfx, \lambda)$ satisfying \eqref{prelim: liouville_condition} for a general transient density is challenging. This is why we consider Assumption~\ref{assumption:linear_Gaussian}. Under linear Gaussian assumptions, the transient density is Gaussian with PDF given by $p(\bfx | \bfz; \lambda) = p_{\calN}(\bfx; \bfmu_{\lambda}, \Sigma_{\lambda})$, where 
\begin{equation}
\label{prelim: transient_para}
    \bfmu_{\lambda} = \Sigma_{\lambda} (P^{-1} \hat{\bfx} + \lambda H^{\top} R^{-1} \bfz), \quad \Sigma_{\lambda} = P - \lambda P H^{\top} (R + \lambda HPH^{\top})^{-1} HP.
\end{equation}
\citet{daum2010exact} provide a closed-form expression for the particle dynamics function $\bfphi(\bfx, \lambda)$ satisfying \eqref{prelim: liouville}:
\begin{equation}
\label{prelim: exact_flow}
	\bfphi(\bfx, \lambda) = A_{\lambda} \bfx + \bfb_{\lambda}, 
\end{equation}
with
\begin{equation*}
A_{\lambda} = -\frac{1}{2} PH^{\top} (R + \lambda HPH^{\top})^{-1} H , \qquad
\bfb_{\lambda} = (I + 2 \lambda A_{\lambda})(A_{\lambda} \hat{\bfx} + (I + \lambda A_{\lambda}) P H ^{\top} R^{-1} \bfz).     
\end{equation*}

Within the particle flow literature, the solution above is termed the Exact Daum and Huang (EDH) flow. A detailed derivation of the EDH flow is missing in \cite{daum2010exact}. In subsequent works, the separation of variables method has been widely adopted to derive the EDH flow \citep{crouse2019consideration, ward2022information, zhang2024multisensor}, where the pseudo-time rate of change of the marginal density $p(\bfz; \lambda)$ is ignored due to its independence of $\bfx$, indicating the EDH flow satisfies \eqref{prelim: liouville} up to some constant. For completeness, we show in the following result that the EDH flow satisfies \eqref{prelim: liouville} exactly.

\begin{lemma}[EDH Flow is Exact \citep{daum2007nonlinear, daum2009nonlinear}]\label{lemma: exact_flow_exact}    
    Consider the \textit{transient density} $p(\bfx | \bfz; \lambda)$ given by \eqref{prelim: log_bayes_rule} with $p(\bfx) = p_{\calN}(\bfx; \hat{\bfx}, P)$ and $p(\bfz | \bfx) = p_{\calN} (\bfz; H \bfx, R)$. Then, the particle dynamics function $\phi(\bfx, \lambda)$ given by the EDH flow \eqref{prelim: exact_flow} satisfies \eqref{prelim: liouville_condition}.
\end{lemma}
\begin{proof}
    Expanding both sides of equation \eqref{prelim: liouville_condition} leads to:
    \begin{equation}
    \begin{split}
        \frac{\partial p(\bfx | \bfz; \lambda)}{\partial \lambda} &= p(\bfx | \bfz; \lambda) \left( \log(p(\bfz | \bfx)) - \frac{\partial \log(p(\bfz; \lambda))}{\partial \lambda} \right), 
        \\
        - \nabla_{\bfx} \cdot \left( p(\bfx | \bfz; \lambda) (A_{\lambda} \bfx + \bfb_{\lambda}) \right) &= - p(\bfx | \bfz; \lambda) \tr(A_{\lambda}) - \frac{\partial p(\bfx | \bfz; \lambda)}{\partial \bfx} \left( A_{\lambda} \bfx + \bfb_{\lambda} \right). 
    \end{split}
    \end{equation}
    Since $p(\bfx | \bfz; \lambda)$ is a Gaussian density, we have:
    \begin{equation}
        \frac{\partial p(\bfx | \bfz; \lambda)}{\partial \bfx} = - p(\bfx | \bfz; \lambda) \left( (\bfx - \bfmu_{\lambda})^{\top} \Sigma_{\lambda}^{-1} \right).
    \end{equation}
    As a result, we only need to show that the following equality holds:
    \begin{equation}
        \log(p(\bfz | \bfx)) - \frac{\partial \log(p(\bfz; \lambda))}{\partial \lambda} = \left( \bfx - \bfmu_{\lambda} \right)^{\top} \Sigma_{\lambda}^{-1} \left( A_{\lambda} \bfx + \bfb_{\lambda} \right) - \tr(A_{\lambda}).
    \end{equation}
    Expanding both sides and re-arranging leads to 
     \begin{equation}
        \frac{\partial \log(p(\bfz; \lambda))}{\partial \lambda} + \frac{1}{2} \log \left( \| 2\pi R \| \right) + \frac{1}{2} \bfz^{\top} R^{-1} \bfz = \tr(A_{\lambda}) + \lambda \bfb_{\lambda}^{\top} H^{\top} R^{-1} \bfz + \bfb_{\lambda}^{\top} P^{-1} \hat{\bfx}. 
     \end{equation}
     Expanding the $\bfb_{\lambda}$ term and re-arranging, we have:
     \begin{equation}
     \begin{split}
         \lambda \bfb_{\lambda}^{\top} H^{\top} R^{-1} \bfz + \bfb_{\lambda}^{\top} P^{-1} \hat{\bfx} = -\frac{1}{2} \bfmu_{\lambda}^{\top} H^{\top} R^{-1} H \bfmu_{\lambda} + \bfmu_{\lambda}^{\top} H^{\top} R^{-1} \bfz.
     \end{split}
     \end{equation}
     The term $\tr(A_{\lambda})$ has the following expression:
     \begin{equation}
     \begin{split}
         \tr(A_{\lambda}) = - \frac{1}{2} \bbE_{p(\bfx | \bfz; \lambda)} \left[ \bfx^{\top} H^{\top} R^{-1} H \bfx \right] + \frac{1}{2} \bfmu_{\lambda}^{\top} H^{\top} R^{-1} H \bfmu_{\lambda}.
    \end{split}
     \end{equation}
     By the definition of $\log(p(\bfz; \lambda))$, we have
     \begin{equation}
         \frac{\partial \log(p(\bfz; \lambda))}{\partial \lambda} + \frac{1}{2} \log\left( | 2 \pi R | \right) + \frac{1}{2} \bfz^{\top} R^{-1} \bfz = \bfmu_{\lambda}^{\top} H^{\top} R^{-1} \bfz - \frac{1}{2} \bbE_{p(\bfx | \bfz; \lambda)} \left[ \bfx^{\top} H^{\top} R^{-1} H \bfx \right].
     \end{equation}
     Finally, the result is obtained by combining the equations above with the transient parameters $\bfmu_{\lambda}$ and $\Sigma_{\lambda}$ given in \eqref{prelim: transient_para}.
\end{proof}

\subsection{Variational Inference}
Variational inference (VI) methods approximate the posterior in \eqref{prelim: bayes_rule} using a PDF $q(\bfx)$ with an explicit expression \citep[Ch.~10]{bishop2006pattern}, termed \textit{variational density}. To perform the approximation, VI minimizes the Kullback-Leibler (KL) divergence between the true posterior $p(\bfx | \bfz)$ and the variational density $q(\bfx)$:
\begin{equation}
\label{prelim: kld}
    D_{KL}\left( q(\bfx) \| p(\bfx | \bfz) \right) = \int q(\bfx) \log \frac{q(\bfx)}{p(\bfx | \bfz)} \d \bfx.
\end{equation}
Formally, given the statistical manifold $\calP$ of probability measures with smooth positive densities,
\begin{equation}\label{prelim: space_of_probability_measures}
    \calP = \Big \{ p(\bfx) \in C^{\infty}: \int p(\bfx) \d \bfx = 1, p(\bfx) > 0 \Big \},
\end{equation}
we seek a variational density $q(\bfx) \in \calP$ that solves the optimization problem:
\begin{equation}\label{min_kld}
    \min_{q(\bfx) \in \calP} D_{KL}(q(\bfx) \| p(\bfx | \bfz) ).
\end{equation}
If $p(\bfx | \bfz) \in \calP$, then the optimal variational density $q^*(\bfx)$ is given by $q^*(\bfx) = p(\bfx | \bfz)$. Conversely, if $p(\bfx | \bfz) \notin \calP$, the optimal variational density $q^*(\bfx)$ satisfies $D_{KL}(q^*(\bfx) \| p(\bfx | \bfz) ) \leq D_{KL}(q(\bfx) \| p(\bfx | \bfz) )$, for all $q(\bfx) \in \calP$. 

Optimizing directly over $\calP$ is challenging due to its infinite-dimensional nature. VI methods usually select a parametric family of variational densities $q(\bfx; \bftheta)$, with parameters $\bftheta$ from an admissible parameter set $\Theta$. Popular choices of variational densities include multivariate Gaussian \citep{opper2009gaussian} and Gaussian mixture \citep{lin2019fast}.
This makes the optimization problem in \eqref{min_kld} finite dimensional:
\begin{equation}
\label{prelim: parametric_vi}
    \min_{\bftheta \in \Theta} D_{KL}(q(\bfx; \bftheta) \| p(\bfx | \bfz) ).
\end{equation}
We distinguish between discrete-time and continuous-time VI techniques. Discrete-time VI performs gradient descent on $D_{KL}$ to update (the parameters of) the variational density, while continuous-time VI uses gradient flow.

\subsubsection{Discrete-Time Formulation: Gradient Descent}

A common approach to optimize the parameters in \eqref{prelim: parametric_vi} is to use gradient descent:
\begin{equation}
\begin{split}
\label{prelim: vi_update}
    \bftheta_{t + 1} &= \bftheta_t - \beta_t \nabla_{\bftheta} D_{KL}\left( q(\bfx; \bftheta) \| p(\bfx | \bfz) \right) \bigm |_{\bftheta = \bftheta_t} \\
    &= \bftheta_t - \beta_t \bbE_{q(\bfx; \bftheta)} \left[ \nabla_{\bftheta} \log(q(\bfx; \bftheta)) \left( \log \left( \frac{q(\bfx; \bftheta)}{p(\bfx | \bfz)} \right) + 1 \right) \right] \biggm |_{\bftheta = \bftheta_t}, 
\end{split}
\end{equation} 
where $\beta_t > 0$ is the step size. Starting with initial parameters $\bftheta_0$, gradient descent updates~$\bftheta_t$ in the direction of the negative gradient of the KL divergence. However, for a non-conjugate prior and likelihood pair, the KL gradient is challenging to obtain due to the expectation in \eqref{prelim: vi_update}. Stochastic gradient methods \citep{hoffman2013stochastic, ranganath2014black} have been proposed to overcome this challenge. Stochastic gradient methods use samples to approximate the expectation over $q(\bfx; \bftheta)$ in \eqref{prelim: vi_update}. The KL divergence is convex with respect to the variational density $q(\bfx; \bftheta)$ but it is not necessarily convex with respect to the density parameters $\bftheta$. As a result, local convergence is to be expected. The convergence analysis of the general gradient descent method is available in \citet{curry1944method}.

\subsubsection{Continuous-Time Formulation: Fischer-Rao Gradient Flow}
\label{subsec: fisher_rao_flow}
Next, we describe an alternative method for solving the optimization problem~\eqref{min_kld} formulated by the VI method. Following the exposition in \cite{chen2023gradient}, we consider a specific geometry over the space of probability densities and present a continuous-time gradient flow approach as an alternative to the usual discrete-time gradient descent method. At a point $q \in \calP$, consider the associated tangent space $T_q\calP$ of the probability space $\calP$ in \eqref{prelim: space_of_probability_measures}. Note that
\begin{equation}
    T_q\calP \subseteq \Big \{ \sigma \in C^{\infty}: \int \sigma(\bfx) \d \bfx = 0 \Big \}. 
\end{equation}
The cotangent space $T^*_q\calP$ is the dual of $T_{q}\calP$. We can introduce a bilinear map $\langle \cdot, \cdot \rangle$ as the duality pairing $T^*_q\calP \times T_{q}\calP \to \bbR$. For any $\psi \in T^*_q\calP$ and $\sigma \in T_{q}\calP$, the duality pairing between $T^*_q\calP$ and $T_{q}\calP$ can be identified in terms of $L^2$ integration as $\langle \psi, \sigma \rangle = \int \psi \sigma \d \bfx$. The most important element in $T^*_q\calP$ we consider is the first variation of the KL divergence $\frac{\delta D_{KL}(q || p)}{\delta q}$,
\begin{equation} \label{fisher_rao: first_variation}
    \left \langle \frac{\delta D_{KL}(q || p)}{\delta q}, \sigma \right \rangle = \lim_{\epsilon \to 0} \frac{D_{KL}(q + \epsilon \sigma || p) - D_{KL}(q || p)}{\epsilon}, 
\end{equation}
for $\sigma \in T_q\calP$. Given a metric tensor at $q$, denoted by $M(q) : T_{q}\calP \to T^*_q\calP$,
we can express the Riemannian metric $g_q : T_{q}\calP \times T_{q}\calP \to \bbR$ as $g_q(\sigma_1, \sigma_2) = \langle M(q)\sigma_1, \sigma_2 \rangle$. Consider the Fisher-Rao Riemannian metric \citep{amari2016information}:
\begin{equation}
\label{fisher_rao: metric}
    g^{\FR}_q(\sigma_1, \sigma_2) = \int \frac{\sigma_1(\bfx) \sigma_2(\bfx)}{q(\bfx)} \d \bfx, \quad \text{for} \ \sigma_1, \sigma_2 \in T_q\calP. 
\end{equation}
The choice of elements in $T^*_q\calP$ is not unique under $L^2$ integration, since $\langle \psi, \sigma \rangle = \langle \psi + c, \sigma \rangle$ for all $\psi \in T^*_q\calP, \ \sigma \in T_q\calP$ and any constant $c$. However, a unique representation can be obtained by requiring $\int q \psi \d \bfx = 0$, and in that case, the metric tensor associated with the Fisher-Rao Riemannian metric $g^{\FR}_q$ is the Fisher-Rao metric tensor $M^{\FR}(q)$, which satisfies the following equation \citep{chen2023sampling}:
\begin{equation}\label{fisher_rao: fr_grad_cond}
    M^{\FR}(q)^{-1} \psi = q \psi \in T_q\calP, \quad \forall \psi \in T^*_q\calP.
\end{equation}
The gradient of the KL divergence under the Fisher-Rao Riemannian metric, denoted by $\nabla^{\FR}_{q} D_{KL}$, is defined according to the following condition:
\begin{equation}\label{fisher_rao: fr_grad_def}
    g^{\FR}_q \left( \nabla^{\FR}_{q} D_{KL}, \sigma \right) = \left \langle \frac{\delta D_{KL}(q || p)}{\delta q}, \sigma \right \rangle, \quad \forall \sigma \in T_q\calP.
\end{equation}
\begin{proposition}[Gradient of KL Divergence]
    The Fisher-Rao Riemannian gradient of the KL divergence $D_{KL}(q(\bfx) \| p(\bfx | \bfz) )$ is given by 
    \begin{equation}
    \label{fisher_rao: fisher_rao_gradient}
        \nabla^{\FR}_{q} D_{KL}(q(\bfx) \| p(\bfx | \bfz) ) = q(\bfx) \left( \log \left( \frac{q(\bfx)}{p(\bfx | \bfz)} \right) - \bbE_{q(\bfx)} \left[ \log \left( \frac{q(\bfx)}{p(\bfx | \bfz)} \right) \right] \right), 
    \end{equation}
    where $p(\bfx | \bfz)$ is the Bayes' posterior given by \eqref{prelim: bayes_rule}.
\end{proposition}
\begin{proof}
    The Fr\'echet derivative of the KL divergence $D_{KL}(q(\bfx) \| p(\bfx | \bfz) )$ is given by
    \begin{equation}
        \lim_{\epsilon \to 0} \frac{D_{KL}(q(\bfx) + \epsilon \sigma(\bfx) || p(\bfx | \bfz)) - D_{KL}(q(\bfx) || p(\bfx | \bfz))}{\epsilon} = \int \sigma(\bfx) \left( 1 + \log \left( \frac{q(\bfx)}{p(\bfx | \bfz)} \right) \right) \d \bfx.
    \end{equation}
    Substituting this expression in~\eqref{fisher_rao: first_variation} and using that $\int \sigma(\bfx) \d \bfx= 0$ (since $\sigma \in T_q \calP$), we have 
    \begin{equation} \label{fisher_rao: first_variation_cond}
        \int \frac{\delta D_{KL}(q(\bfx) \| p(\bfx | \bfz) )}{\delta q(\bfx)} \sigma(\bfx) \d \bfx = \int \sigma(\bfx) \log \left( \frac{q(\bfx)}{p(\bfx | \bfz)} \right) \d \bfx. 
    \end{equation}
    However, the first variation of the KL divergence is not uniquely determined because
    \begin{equation}
        \int \sigma(\bfx) \left( \log \left( \frac{q(\bfx)}{p(\bfx | \bfz)} \right) + c \right) \d \bfx = \int \sigma(\bfx) \log \left( \frac{q(\bfx)}{p(\bfx | \bfz)} \right) \d \bfx, \quad \forall c \in \bbR.
    \end{equation}
    Based on our discussion above, the first variation of the KL divergence can be uniquely identified by further requiring that
    \begin{equation}
        \int q(\bfx) \left( \log \left( \frac{q(\bfx)}{p(\bfx | \bfz)} \right) + c \right) \d \bfx = 0,
    \end{equation}
    which leads to the selection $c = -\bbE_{q(\bfx)} \left[ \log \left( \frac{q(\bfx)}{p(\bfx | \bfz)} \right) \right]$. Therefore, the first variation of the KL divergence is given by 
    \begin{equation}\label{eq: kl_first_variation}
        \frac{\delta D_{KL}(q(\bfx) \| p(\bfx | \bfz) )}{\delta q(\bfx)} = \log \left( \frac{q(\bfx)}{p(\bfx | \bfz)} \right) - \bbE_{q(\bfx)} \left[ \log \left( \frac{q(\bfx)}{p(\bfx | \bfz)} \right) \right].
    \end{equation}
    Using~\eqref{fisher_rao: fr_grad_cond}, the Fisher-Rao gradient of the KL divergence $D_{KL}(q(\bfx) \| p(\bfx | \bfz) )$ is then given by
    \begin{equation}
        \nabla^{\FR}_{q} D_{KL}(q(\bfx) \| p(\bfx | \bfz) ) = q(\bfx) \frac{\delta D_{KL}(q(\bfx) \| p(\bfx | \bfz) )}{\delta q(\bfx)}.
    \end{equation}
    Substituting \eqref{eq: kl_first_variation} into the equation above, we get the desired result.
\end{proof}
The Fisher-Rao gradient flow in the space of probability measures $\calP$ takes the following form:
\begin{equation}
\label{fisher_rao: fisher_rao_flow}
    \frac{\partial q(\bfx; t)}{\partial t} = - \nabla^{\FR}_{q} D_{KL}(q(\bfx; t) \| p(\bfx | \bfz) ).
\end{equation}
Further, we consider the case where the variational density is parameterized by $\bftheta$, and denote the space of $\bftheta$-parameterized positive probability densities as:
\begin{equation}
    \calP_{\bftheta} = \{ p(\bfx; \bftheta) \in \calP: \bftheta \in \Theta \subseteq \bbR^l \} \subset \calP
\end{equation}
The basis of $T_{p(\bfx; \bftheta)} \calP_{\bftheta}$ is given by
\begin{equation}
    \left \{ \frac{\partial p(\bfx; \bftheta)}{\partial \theta_1}, \frac{\partial p(\bfx; \bftheta)}{\partial \theta_2}, \ldots, \frac{\partial p(\bfx; \bftheta)}{\partial \theta_l} \right \}, 
\end{equation}
where $\theta_i$ denotes the $i$th element of $\bftheta$. As a result, the Fisher-Rao metric tensor $M^{FR}(\bftheta): \Theta \to \bbR^{l \times l}$ under parametrization $\bftheta$ is given as follows \citep{nielsen2020ig}:
\begin{equation}
\label{fisher_rao: fisher_rao_tensor}
    M^{FR}(\bftheta) \coloneqq \bbE_{p(\bfx; \bftheta)} \left[\nabla_{\bftheta}\log (p(\bfx; \bftheta)) \nabla^{\top}_{\bftheta}\log (p(\bfx; \bftheta))\right],
\end{equation}
which is identical to the Fisher Information Matrix (FIM), denoted $\calI(\bftheta) = M^{FR}(\bftheta)$. Restricting the variational density to the space of $\bftheta$-parameterized densities, we consider the finite-dimensional constrained optimization problem \eqref{prelim: parametric_vi}. Given the metric tensor \eqref{fisher_rao: fisher_rao_tensor}, the Fisher-Rao parameter flow is given by
\begin{equation}
\label{fisher_rao: fr_para_flow}
    \frac{\d \bftheta_t}{\d t} = - \calI^{-1}(\bftheta_t) \nabla _{\bftheta_t} D_{KL}(q(\bfx; \bftheta_t) \| p(\bfx | \bfz)).
\end{equation}
Notice that the functional space Fisher-Rao gradient flow \eqref{fisher_rao: fisher_rao_flow} and the parameter space Fisher-Rao parameter flow \eqref{fisher_rao: fr_para_flow} can be interpreted as the Riemannian gradient flow induced by the Fisher-Rao Riemannian metric \eqref{fisher_rao: metric}.
\section{Transient Density as a Solution to Fisher-Rao Gradient Flow}
\label{sec: transient_and_fisher}

We aim to identify a VI formulation that yields the transient density~\eqref{prelim: log_bayes_rule} as its solution, thus establishing a connection between the transient density \eqref{prelim: log_bayes_rule} in the particle flow formulation and the Fisher-Rao gradient flow \eqref{fisher_rao: fisher_rao_flow} in the VI formulation. In order to do this, we must address the following challenges.

\begin{description}
    \item[Time parameterizations:]  As shown in Section~\ref{subsec: pfpf}, the particle flow is derived using a particular parameterization of Bayes' rule, which introduces a pseudo-time parameter $\lambda$. However, the Fisher-Rao gradient flow \eqref{fisher_rao: fisher_rao_flow} derived in Section~\ref{subsec: fisher_rao_flow} does not possess such a pseudo-time parameter. In fact, we have that the transient density trajectory satisfies $p(\bfx | \bfz; \lambda) \to p(\bfx | \bfz)$ as $\lambda \to 1$, while the variational density trajectory defined by the Fisher-Rao gradient flow satisfies $q(\bfx; t) \to p(\bfx | \bfz)$ as $t \to \infty$. 
    
    \item[Initialization:]  Another key difference is that the transient density trajectory $p(\bfx | \bfz; \lambda)$ defines a transformation from the prior to the Bayes' posterior, while the variational density trajectory defines a transformation from any density to the Bayes' posterior. 
\end{description}

As our result below shows, we can obtain the transient density as a solution to the Fisher-Rao gradient flow by initializing the variational density with the prior and introducing a time scaling function. The time scaling ensures that the time rate of change of the variational density matches the pseudo-time rate of change of the transient density, which is then used to derive the particle dynamics function governing the EDH flow. 

\begin{theorem}[Transient Density as a Solution to Fisher-Rao Gradient Flow]
\label{theorem: fr_sol}
    The transient density \eqref{prelim: log_bayes_rule} with pseudo-time scaling function $\lambda(t) = 1 - \exp(-t)$ is a solution to the Fisher-Rao gradient flow~\eqref{fisher_rao: fisher_rao_flow} of the KL divergence  with $q(\bfx; 0) = p(\bfx)$.
\end{theorem}

\begin{proof}
    We have to verify that $
        q^*(\bfx; \lambda(t)) = {p(\bfz | \bfx)^{\lambda(t)} p(\bfx)}/{p(\bfz; \lambda(t))}$
    satisfies
    \begin{equation}\label{eq:to-verify}
        \frac{\partial q^*(\bfx; \lambda(t))}{\partial t} = - \nabla^{\FR}_{q^*} D_{KL}(q^*(\bfx; \lambda(t)) \| p(\bfx | \bfz) ).
    \end{equation}
    The derivative of $q^*(\bfx; \lambda(t))$ with respect to time $t$ in the left-hand side above can be written as:
    \begin{equation} \label{thm_proof: sol_derivative}
        \frac{\partial q^*(\bfx; \lambda(t))}{\partial t} = (1 - \lambda(t)) q^*(\bfx; \lambda(t)) \left( \log \left( p(\bfz | \bfx) \right) - \frac{1}{p(\bfz; \lambda(t))} \frac{\partial p(\bfz; \lambda(t))}{\partial \lambda(t)} \right).
    \end{equation}
    %
    By the definition of $p(\bfz; \lambda(t))$ in~\eqref{prelim: log_bayes_rule}, we have:
    \begin{equation}
        \frac{\partial p(\bfz; \lambda(t))}{\partial \lambda(t)} = \int p(\bfz | \bfx)^{\lambda(t)} p(\bfx) \log \left( p(\bfz | \bfx) \right) \d \bfx.
    \end{equation}
    As a result, it holds that
    \begin{equation}
        \frac{1}{p(\bfz; \lambda(t))} \frac{\partial p(\bfz; \lambda(t))}{\partial \lambda(t)} = \int \frac{p(\bfz | \bfx)^{\lambda(t)} p(\bfx)}{p(\bfz; \lambda(t))} \log \left( p(\bfz | \bfx) \right) \d \bfx = \bbE_{q^*(\bfx; \lambda(t))} \left[ \log(p(\bfz | \bfx)) \right]. 
    \end{equation}
    Substituting into \eqref{thm_proof: sol_derivative}, we obtain
    \begin{equation}
        \frac{\partial q^*(\bfx; \lambda(t))}{\partial t} = (1 - \lambda(t)) q^*(\bfx; \lambda(t)) \left( \log \left( p(\bfz | \bfx) \right) - \bbE_{q^*(\bfx; \lambda(t))} \left[ \log(p(\bfz | \bfx)) \right] \right). 
    \end{equation}
    On the other hand, regarding the right-hand side of~\eqref{eq:to-verify}, using the definition of $q^*(\bfx; \lambda(t))$ and~\eqref{prelim: bayes_rule}, we obtain
    \begin{equation}
        \log \left( \frac{q^*(\bfx; \lambda(t))}{p(\bfx | \bfz)} \right) = (\lambda(t) - 1) \log(p(\bfz | \bfx)) + \log \left( \frac{p(\bfz)}{p(\bfz; \lambda(t))} \right).
    \end{equation}
    Substituting this expression into~\eqref{fisher_rao: fisher_rao_gradient}, the negative Fisher-Rao gradient of the KL divergence can be expressed as
    \begin{equation} \label{thm_proof: n_fr_grad}
        - \nabla^{\FR}_{q^*} D_{KL}(q^*(\bfx; \lambda(t)) \| p(\bfx | \bfz) ) = (1 - \lambda(t)) q^*(\bfx; \lambda(t)) \left( \log(p(\bfz | \bfx)) - \bbE_{q^*(\bfx; \lambda(t))} \left[ \log(p(\bfz | \bfx)) \right] \right),
    \end{equation}
    which matches the expression for $\frac{\partial q^*(\bfx; \lambda(t))}{\partial t} $.
    %
    %
    %
\end{proof}

%
%
Theorem~\ref{theorem: fr_sol} shows that a Fisher-Rao particle flow can be derived by finding a particle dynamics function $\bfphi$ such that the following equation holds:
\begin{equation}
\label{fisher_rao: liouville_cond}
    \nabla_{\bfx} \cdot \left( q(\bfx; t) \bfphi(\bfx, t) \right) = \nabla^{FR}_{q} D_{KL}(q(\bfx; t) \| p(\bfx | \bfz) ),
\end{equation}
with the initial variational density set to the prior $q(\bfx; 0) = p(\bfx)$. Obtaining a closed-form expression for this particle dynamics function is challenging in general. To alleviate this, we can restrict the variational density to have a specific parametric form. In the next section, we restrict the variational density to a single Gaussian and show that, under linear Gaussian assumptions, the corresponding particle dynamics function is a time-scaled version of the particle dynamics function governing the EDH flow. 



\section{Gaussian Fisher-Rao Flows}
\label{sec: gaussian_flow}
This section focuses on the case where the variational density is selected as Gaussian and establishes a connection between the Gaussian Fisher-Rao gradient flow and the particle dynamics in the particle flow. Since Gaussian densities can be specified by mean $\bfmu$ and covariance $\Sigma$ parameters, instead of working in the space of Gaussian densities, we work in the space of parameters:
\begin{equation}
\label{gaussian_approx: param_space}
    \calA \coloneqq \left \{ \left( \bfmu, \VEC(\Sigma^{-1}) \right): \bfmu \in \bbR^n, \Sigma^{-1} \in \bbR^{n \times n} ,  \Sigma  \succ 0  \right \},
\end{equation}
where $\VEC( \cdot )$ is the vectorization operator that converts a matrix to a vector by stacking its columns. This parametrization results in the same update for the inverse covariance matrix as the half-vectorization parametrization, which accounts for the symmetry of the inverse covariance matrix \citep{barfoot2020multivariate}. For $\bfalpha \in \calA$, the inverse FIM \eqref{fisher_rao: fisher_rao_tensor} evaluates to:
\begin{equation}
\label{gaussian_approx: fisher_information}
    \calI^{-1}(\bfalpha) = \begin{bmatrix}
                                \Sigma & \mathbf{0} \\
                                \mathbf{0} & 2 \left( \Sigma^{-1} \otimes \Sigma^{-1} \right)
                            \end{bmatrix},
\end{equation}
where $\otimes$ denotes the Kronecker product. To simplify the notation, define:
\begin{equation}
\label{fisher_rao: kl_kernel}
    V(\bfx; \bfalpha) = \log(q(\bfx; \bfalpha)) - \log(p(\bfx, \bfz)), 
\end{equation}
where $q(\bfx; \bfalpha)$ is the variational density and $p(\bfx, \bfz) = p(\bfz | \bfx) p(\bfx)$ is the joint density. The derivative of the KL divergence with respect to the Gaussian parameters is given by \citep{opper2009gaussian}:
\begin{align}
    \nabla_{\bfmu} D_{KL}(q(\bfx; \bfalpha) \| p(\bfx | \bfz)) &= \bbE_{q(\bfx; \bfalpha)} \left[ \nabla_{\bfx} V(\bfx; \bfalpha) \right],
    \\
    \nabla_{\Sigma^{-1}} D_{KL}(q(\bfx; \bfalpha) \| p(\bfx | \bfz)) &= -\frac{1}{2} \Sigma \bbE_{q(\bfx; \bfalpha)} \left[ \nabla^2_{\bfx} V(\bfx; \bfalpha)\right] \Sigma,  \label{gaussian_approx: kl_derivative}
\end{align}
with $V(\bfx; \bfalpha)$ defined in \eqref{fisher_rao: kl_kernel}. Inserting \eqref{gaussian_approx: fisher_information} and \eqref{gaussian_approx: kl_derivative} into \eqref{fisher_rao: fr_para_flow}, and using the fact $\VEC(ABC) = (C^{\top} \otimes A) \VEC(B)$, the Gaussian Fisher-Rao parameter flow takes the form:
\begin{equation}
\label{gaussian_approx: approx_fr_para_flow}
    \frac{\d \bfmu_{t}}{\d t} = -\Sigma_t \bbE_{q(\bfx; \bfalpha_t)}\left[\nabla_{\bfx} V(\bfx; \bfalpha_t)\right], \quad \frac{\d \Sigma^{-1}_{t}}{\d t} = \bbE_{q(\bfx; \bfalpha_t)}\left[\nabla^2_{\bfx} V(\bfx; \bfalpha_t)\right], 
\end{equation} 
where $q(\bfx; \bfalpha_t) = p_{\calN}(\bfx; \mu_t, \Sigma_t)$. As a result, the Gaussian Fisher-Rao parameter flow \eqref{gaussian_approx: approx_fr_para_flow} defines a time rate of change of the variational density $q(\bfx; \bfalpha_t)$, which can be captured using the Liouville equation. This observation is formally stated in the following result.

\begin{lemma}[Gaussian Fisher-Rao Particle Flow]
\label{lemma: natural_particle_flow}
    The time rate of change of the variational density $q(\bfx; \bfalpha_t)$ induced by the Gaussian Fisher-Rao parameter flow~\eqref{gaussian_approx: approx_fr_para_flow} is captured by the  Liouville equation:
    \begin{equation}
    \label{gaussian_approx: natural_para_liouville}
    \begin{split}
        \frac{\d q(\bfx; \bfalpha_t)}{\d t} = - \nabla_{\bfx} \cdot \big( q(\bfx; \bfalpha_t) \tilde{\phi}(\bfx, t) \big), 
    \end{split}
    \end{equation} 
    with particle dynamics function $ \tilde{\bfphi}(\bfx, t) = \tilde{A}_t \bfx + \tilde{\bfb}_t$, where
    \begin{align}\label{gaussian_approx: approx_particle_flow}
     \tilde{A}_t  = -\frac{1}{2} \Sigma_t \bbE_{q(\bfx; \bfalpha_t)}\left[\nabla^2_{\bfx}V(\bfx; \bfalpha_t)\right] , \qquad
     \tilde{\bfb}_t = -\Sigma_t \bbE_{q(\bfx; \bfalpha_t)}\left[\nabla_{\bfx}V(\bfx; \bfalpha_t)\right] - \tilde{A}_t \bfmu_t.
    \end{align}
\end{lemma}
\begin{proof}
    Using the chain rule and~\eqref{gaussian_approx: approx_fr_para_flow}, we can write:
    \begin{align}
        &\frac{\d q(\bfx; \bfalpha_t)}{\d t} = \frac{\partial q(\bfx; \bfalpha_t)}{\partial \bfmu_{t}} \frac{\d \bfmu_{t}}{\d t} + \tr\left( \frac{\partial q(\bfx; \bfalpha_t)}{\partial \Sigma^{-1}_{t}} \frac{\d \Sigma^{-1}_{t}}{\d t} \right)
        \\
        &= q(\bfx; \bfalpha_t) (\bfmu_t - \bfx)^{\top} \bbE_{q(\bfx; \bfalpha_t)}\left[\nabla_{\bfx}V(\bfx; \bfalpha_t)\right] + \frac{1}{2} q(\bfx; \bfalpha_t) \tr \left( \Sigma_t \bbE_{q(\bfx; \bfalpha_t)}\left[\nabla^2_{\bfx}V(\bfx; \bfalpha_t)\right] \right) 
        \\
        & \quad - \frac{1}{2} q(\bfx; \bfalpha_t) (\bfx - \bfmu_t)^{\top} \bbE_{q(\bfx; \bfalpha_t)}\left[\nabla^2_{\bfx}V(\bfx; \bfalpha_t)\right] (\bfx - \bfmu_t), 
        \label{eq:to-match}
    \end{align}
    where we have employed Jacobi's formula~\citep{petersen2008matrix} to write the derivatives of the Gaussian density.
    Substituting the particle dynamics function defined by~\eqref{gaussian_approx: approx_particle_flow} into \eqref{gaussian_approx: natural_para_liouville}, we obtain
    \begin{align*}
        &- \nabla_{\bfx} \cdot \big( q(\bfx; \bfalpha_t) (\tilde{A}_t \bfx + \tilde{\bfb}_t) \big) = -q(\bfx; \bfalpha_t) \tr(\tilde{A}_t) + \frac{\partial q(\bfx; \bfalpha_t)}{\partial \bfmu_t} (\tilde{A}_t \bfx + \tilde{\bfb}_t) \\
        &= \frac{1}{2}  q(\bfx; \bfalpha_t)  \tr \left( \Sigma_t \bbE_{q(\bfx; \bfalpha_t)}\left[\nabla^2_{\bfx}V(\bfx; \bfalpha_t)\right] \right) + q(\bfx; \bfalpha_t) (\bfx - \bfmu_t)^{\top} \Sigma^{-1}_t \tilde{A}_t (\bfx - \bfmu_t) \\
        & \quad - q(\bfx; \bfalpha_t) (\bfx - \bfmu_t)^{\top} \bbE_{q(\bfx; \bfalpha_t)}\left[\nabla_{\bfx}V(\bfx; \bfalpha_t)\right] .
    \end{align*}
    Substituting the value of $\tilde{A}_t$ in this expression, we see that it matches~\eqref{eq:to-match}.
\end{proof}

According to Theorem~\ref{theorem: fr_sol}, the particle dynamics function described in \eqref{gaussian_approx: approx_particle_flow}, which governs the Gaussian Fisher-Rao particle flow, must correspond to a time-scaled version of the particle dynamics function in \eqref{prelim: exact_flow} that governs the EDH flow. This equivalence holds under the linear Gaussian assumptions. This observation motivates our forthcoming discussion.

Under Assumption~\ref{assumption:linear_Gaussian} (linear Gaussian assumption), we have:
\begin{equation}
    \nabla_{\bfx}V(\bfx; \bfalpha_t) = \Sigma^{-1}_t(\bfmu_t - \bfx) + \Sigma^{-1}_p(\bfx - \bfmu_p), \quad \nabla^2_{\bfx}V(\bfx; \bfalpha_t) = -\Sigma^{-1}_t + \Sigma^{-1}_p, 
\end{equation}
where $\bfmu_p = \hat{\bfx} + PH^{\top} (R + HPH^{\top})^{-1} (\bfz - H \hat{\bfx})$ and $\Sigma^{-1}_p = P^{-1} + H^{\top}R^{-1}H$ denote the posterior mean and the inverse of the posterior covariance, respectively. As a result, the Gaussian Fisher-Rao parameter flow \eqref{gaussian_approx: approx_fr_para_flow} becomes:
\begin{equation}
\label{gaussian_approx: linear_fr_para_flow}
    \frac{\d \bfmu_{t}}{\d t} = -\Sigma_t \Sigma^{-1}_p (\bfmu_t - \bfmu_p), \qquad \frac{\d \Sigma^{-1}_{t}}{\d t} = \Sigma^{-1}_p - \Sigma^{-1}_t.
\end{equation} 
Also, the particle dynamics function from Lemma~\ref{lemma: natural_particle_flow}, which describes the Gaussian Fisher-Rao particle flow, is determined by,
\begin{equation}\label{gaussian_approx: linear_fr_particle_flow}
 \tilde{A}_t = - \frac{1}{2} \Sigma_t \left( \Sigma^{-1}_p - \Sigma^{-1}_t \right), \qquad \tilde{\bfb}_t = \Sigma_t \Sigma^{-1}_p \left( \bfmu_p - \bfmu_t \right) - \tilde{A}_t \bfmu_t. 
\end{equation}
Based on Theorem~\ref{theorem: fr_sol}, the transient parameters \eqref{prelim: transient_para} describing the transient density should be a time-scaled solution to the Gaussian Fisher-Rao parameter flow \eqref{gaussian_approx: linear_fr_para_flow} under linear Gaussian assumptions, indicating a connection between the particle dynamics functions. This intuition is formalized in the following result.

\begin{theorem}[EDH Flow as Fisher-Rao Particle Flow]
\label{theorem: linear_gaussian_connection}
    Under Assumption~\ref{assumption:linear_Gaussian}, the particle dynamics function $\tilde{\phi}(\bfx, t) = \tilde{A}_t \bfx + \tilde{\bfb}_t$ determined by \eqref{gaussian_approx: linear_fr_particle_flow}, which governs the Gaussian Fisher-Rao particle flow, is a time-scaled version of the EDH flow particle dynamics $\phi(\bfx, \lambda) = A_\lambda \bfx + \bfb_\lambda$ determined by \eqref{prelim: exact_flow}, with time scaling function given by $\lambda(t) = 1 - \exp(-t)$. 
\end{theorem}

\begin{proof}
    Under Assumption~\ref{assumption:linear_Gaussian}, a solution to the Gaussian Fisher-Rao parameter flow \eqref{gaussian_approx: linear_fr_para_flow} is given as follows:
    \begin{equation}
    \label{gaussian_approx: natural_para_sol}
        \bfmu_{t} = \Sigma_t (P^{-1} \hat{\bfx} + \lambda(t) H^{\top} R^{-1} \bfz), \qquad \Sigma_{t}^{-1} = P^{-1} + \lambda(t) H^{\top}R^{-1}H.
    \end{equation}
    This can be verified by observing that 
    \begin{equation}
        \begin{split}
            \frac{\d \Sigma^{-1}_{t}}{\d t} &= (1 - \lambda(t)) H^{\top}R^{-1}H = \Sigma^{-1}_p - \Sigma^{-1}_t \\
            \frac{\d \bfmu_t}{\d t} &= \frac{\d \Sigma_{t}}{\d t} \Sigma^{-1}_t \bfmu_t + (1 - \lambda(t)) \Sigma_{t} H^{\top} R^{-1} \bfz = -\Sigma_t \Sigma^{-1}_p \bfmu_t + \bfmu_t + (1 - \lambda(t)) \Sigma_{t} H^{\top} R^{-1} \bfz 
            \\
            &= -\Sigma_t \Sigma^{-1}_p \bfmu_t + \Sigma_t (P^{-1} \hat{\bfx} + H^{\top} R^{-1} \bfz) = -\Sigma_t \Sigma^{-1}_p \bfmu_t + \Sigma_t \Sigma^{-1}_p \bfmu_p, 
        \end{split}
    \end{equation}
    where the equality $P^{-1} \hat{\bfx} + H^{\top} R^{-1} \bfz = \Sigma^{-1}_p \bfmu_p$ can be obtained by applying the Woodbury matrix identity~\citep{petersen2008matrix}. Using the closed-form expressions for $\bfmu_t$ and $\Sigma_t$ in equation \eqref{gaussian_approx: natural_para_sol}, we can write $\tilde{A}_t$ and $\tilde{\bfb}_t$ in \eqref{gaussian_approx: linear_fr_particle_flow} as follows:
    \begin{equation} \label{eq: fr_closed_form}
    \begin{split}
        \tilde{A}_t &= \frac{1}{2} \Sigma_t \left( -\Sigma^{-1}_p + \Sigma_t^{-1} \right) = - \frac{1}{2} (1 - \lambda(t)) \Sigma_t H^{\top} R^{-1} H, \\
        \tilde{\bfb}_t &= \Sigma_t \Sigma^{-1}_p \left( \bfmu_p - \bfmu_t \right) - \tilde{A}_t \bfmu_t = - \Sigma_t \Sigma^{-1}_p \bfmu_t + \bfmu_t + (1 - \lambda(t)) \Sigma_t H^{\top} R^{-1} \bfz - \tilde{A}_t \bfmu_t  \\
        &= \tilde{A}_t \bfmu_t + (1 - \lambda(t)) \Sigma_t H^{\top} R^{-1} \bfz.
    \end{split}
    \end{equation}
    Next, we rewrite the term $\bfb_{\lambda}$ in \eqref{prelim: exact_flow} such that it shares a similar structure to the term $\tilde{\bfb}_t$ in \eqref{eq: fr_closed_form}. First, observe that
    \begin{equation}\label{eq: identity_1}
        \lambda (R + \lambda HPH^{\top})^{-1} HPH^{\top} = I - (R + \lambda HPH^{\top})^{-1} R.
    \end{equation}
    Using this equation, we deduce that
	\begin{equation}
    \label{eq: identity_2}
    \begin{split}
        (I + 2 \lambda A_{\lambda})P H^{\top} R^{-1} &= P H^{\top} R^{-1} - \lambda PH^{\top} (R + \lambda HPH^{\top})^{-1} H P H^{\top} R^{-1} \\
        &= P H^{\top} R^{-1} - PH^{\top} (I - (R + \lambda HPH^{\top})^{-1} R) R^{-1} \\
        &= PH^{\top} (R + \lambda HPH^{\top})^{-1} ,
    \end{split}
	\end{equation} 
    which in turn implies that
    \begin{equation} \label{eq: identity_3}
        \lambda A_{\lambda} P H^{\top} R^{-1} = \frac{1}{2} \left( PH^{\top} (R + \lambda HPH^{\top})^{-1} - P H^{\top} R^{-1} \right).
    \end{equation}
    Now, we employ \eqref{eq: identity_1} and the definitions of $\bfmu_{\lambda}$ and $\Sigma_{\lambda}$
    in \eqref{prelim: transient_para}, to write 
    \begin{equation}\label{eq: mu_lambda}
        \bfmu_{\lambda} = \hat{\bfx} + \lambda PH^{\top} (R + \lambda HPH^{\top})^{-1} (\bfz - H \hat{\bfx}). 
    \end{equation}
    Using this equation and the definition of $A_\lambda$, we can express
    \begin{equation} \label{eq: am_simplify_1}
    \begin{split}
        A_{\lambda} \bfmu_{\lambda} &= (I + 2 \lambda A_{\lambda}) A_{\lambda}\hat{\bfx} + \lambda A_{\lambda} P H^{\top} (R + \lambda HPH^{\top})^{-1} \bfz \\
        &= (I + 2 \lambda A_{\lambda}) A_{\lambda}\hat{\bfx} - \frac{1}{2} \underbrace{PH^{\top} (R + \lambda HPH^{\top})^{-1}}_{\eqref{eq: identity_2}} \underbrace{\lambda H P H^{\top} (R + \lambda HPH^{\top})^{-1}}_{\eqref{eq: identity_1}} \bfz \\
        &= (I + 2 \lambda A_{\lambda}) \Big( A_{\lambda}\hat{\bfx} - \frac{1}{2} (P H^{\top} R^{-1} - P H^{\top} (R + \lambda HPH^{\top})^{-1}) \bfz \Big).
    \end{split}
    \end{equation}
    Using this expression, the difference between $\bfb_{\lambda}$ and $A_{\lambda} \bfmu_{\lambda}$ can be expressed as follows:
	\begin{align*}
    \bfb_{\lambda} - A_{\lambda} \bfmu_{\lambda} & =
        (I + 2 \lambda A_{\lambda})(A_{\lambda} \hat{\bfx} + (I + \lambda A_{\lambda}) P H ^{\top} R^{-1} \bfz) - A_{\lambda} \bfmu_{\lambda} \\
        &= (I + 2 \lambda A_{\lambda}) \Big( \tfrac{3}{2} P H ^{\top} R^{-1} + \underbrace{\lambda A_{\lambda} P H ^{\top} R^{-1}}_{\eqref{eq: identity_3}} - \tfrac{1}{2} P H^{\top} (R + \lambda HPH^{\top})^{-1} \Big) \bfz \\
        &= (I + 2 \lambda A_{\lambda}) P H ^{\top} R^{-1} \bfz = PH^{\top} (R + \lambda HPH^{\top})^{-1} \bfz, 
	\end{align*} 
    where the last equality is obtained using \eqref{eq: identity_2}. As a result, we can rewrite the $\bfb_{\lambda}$ term in \eqref{prelim: exact_flow} as follows:
	\begin{equation}
        \bfb_{\lambda} = PH^{\top} (R + \lambda HPH^{\top})^{-1} \bfz + A_{\lambda} \bfmu_{\lambda}.
	\end{equation} 
    Replacing $\lambda$ with $\lambda(t)$ in the definition of $\bfmu_{\lambda}$ and $\Sigma_{\lambda}$ in \eqref{prelim: transient_para} and utilizing the Woodbury formula \citep{petersen2008matrix}, we have:
    \begin{equation}
    \begin{split}
        &\Sigma_{\lambda(t)} = P - \lambda(t) P H^{\top} (R + \lambda(t) HPH^{\top})^{-1} HP = \big( P^{-1} + \lambda(t) H^{\top}R^{-1}H \big)^{-1} = \Sigma_t
        \\
        &\bfmu_{\lambda(t)} = \Sigma_{\lambda(t)} (P^{-1} \hat{\bfx} + \lambda(t) H^{\top} R^{-1} \bfz) = \bfmu_{t}.
    \end{split}
    \end{equation}
    Using the expression~\eqref{gaussian_approx: natural_para_sol} of $\Sigma^{-1}_{t}$, we can obtain the following identity using \eqref{eq: identity_1}:
    \begin{equation}\label{eq: identity_4}
    \begin{split}
        &\Sigma_{t} H^{\top} R^{-1} = \big( P^{-1} + \lambda(t) H^{\top} R^{-1} H \big)^{-1} H^{\top} R^{-1} \\
        &= P H^{\top} R^{-1} - P H^{\top} \underbrace{\lambda(t) (R + \lambda(t) HPH^{\top})^{-1} H P H^{\top}}_{\eqref{eq: identity_1}} R^{-1} \\
        &= P H^{\top} (R + \lambda(t) HPH^{\top})^{-1}. 
    \end{split}
    \end{equation} 
    Finally, by applying \eqref{eq: identity_4}, we have:
    \begin{equation}
        A_{\lambda(t)} = - \frac{1}{2} \Sigma_t H^{\top} R^{-1} H = \frac{1}{1 - \lambda(t)} \tilde{A}_t, \quad \bfb_{\lambda(t)} = A_{\lambda(t)} \bfmu_t + \Sigma_t H^{\top} R^{-1} \bfz = \frac{1}{1 - \lambda(t)} \tilde{\bfb}_t.
    \end{equation}
\end{proof}

The proof of Theorem~\ref{theorem: linear_gaussian_connection} reveals the fact that by rewriting \eqref{gaussian_approx: linear_fr_particle_flow} using the closed-form expressions for $\bfmu_t$ and $\Sigma_t$ in \eqref{gaussian_approx: natural_para_sol}, the particle dynamics function determined by \eqref{gaussian_approx: linear_fr_particle_flow} shares the same expression as \eqref{prelim: exact_flow} up to an appropriate time scaling coefficient. Algorithm~\ref{alg: gaussian_fr} summarizes the key steps for the proposed Gaussian Fisher-Rao particle flow.

\begin{algorithm}[t]
\caption{Gaussian Fisher-Rao Particle Flow}
\label{alg: gaussian_fr}
\begin{algorithmic}[1]
\small
\Require Parameters $\bfalpha_0$ defined in \eqref{gaussian_approx: param_space} specifying the initial variational density $q(\bfx; \bfalpha_0)$, particles $\{[\bfx_j]_{t=0}\}_{j=1}^{M}$ sampled from the initial variational density, and joint density $p(\bfx, \bfz)$
\Output Particles $\{[\bfx_j]_{t=T}\}_{j=1}^{M}$ that approximate the posterior density $p(\bfx | \bfz)$

\Function{$\bff$}{$\{[\bfx_j]_t\}_{j=1}^{M}$, $\bfalpha_t$, $t$}
    \State $\delta \bfalpha_t \gets$ Evaluate \eqref{gaussian_approx: approx_fr_para_flow} with variational density parameters $\bfalpha_t$
    \State $\tilde{A}_t$, $\tilde{\bfb}_t \gets$ Evaluate \eqref{gaussian_approx: approx_particle_flow} with $\bfalpha_t$ and $t$
    \For{each particle \text{$[\bfx_j]_t$}}
    \State $\tilde{\bfphi}([\bfx_j]_t, t) \gets \tilde{A}_t [\bfx_j]_t + \tilde{\bfb}_t$
    \EndFor
\State \Return $\{\tilde{\bfphi}([\bfx_j]_t, t)\}_{j=1}^{M}$, $\delta \bfalpha_t$
\EndFunction

\While{ODE solver running}
    \State $\{[\bfx_j]_{t=T}\}_{j=1}^{M}, \bfalpha_T \gets$ SolveODE($\bff(\{[\bfx_j]_t\}_{j=1}^{M}, \bfalpha_t, t)$) with initial particles $\{[\bfx_j]_{t=0}\}_{j=1}^{M}$, initial variational density parameters $\bfalpha_0$, initial time $t=0$, and termination time $T$
\EndWhile

\State \Return $\{[\bfx_j]_{t=T}\}_{j=1}^{M}$
\end{algorithmic}
\end{algorithm}

Simulation results comparing the Gaussian Fisher-Rao particle flow and the EDH flow are presented in Figure~\ref{fig: linear_gaussian}. We use the following parameters in the simulation:
\begin{equation}
    \hat{\bfx} = \begin{bmatrix} 0 \\ 0 \end{bmatrix}\!, \quad P = \begin{bmatrix} 1.5 & 0.5 \\ 0.5 & 5.5 \end{bmatrix}\!, \quad H = \begin{bmatrix} 1 & 1.5 \\ 0.2 & 2 \end{bmatrix}\!, \quad R = \begin{bmatrix} 0.2 & 0.1 \\ 0.1 & 0.2 \end{bmatrix}\!, \quad \bfx^* = \begin{bmatrix} -1.18 \\ 4.12 \end{bmatrix}, 
\end{equation}
where $\bfx^*$ denotes the true state.

A single Gaussian approximation to the posterior is often insufficient due to its single-modal nature, especially when the observation model is nonlinear (which potentially results in a multi-modal posterior). This motivates our discussion in the next section, where the variational density follows a Gaussian mixture density.

\begin{figure}[t]
    \centering
    \includegraphics[width=0.45\linewidth,valign=t]{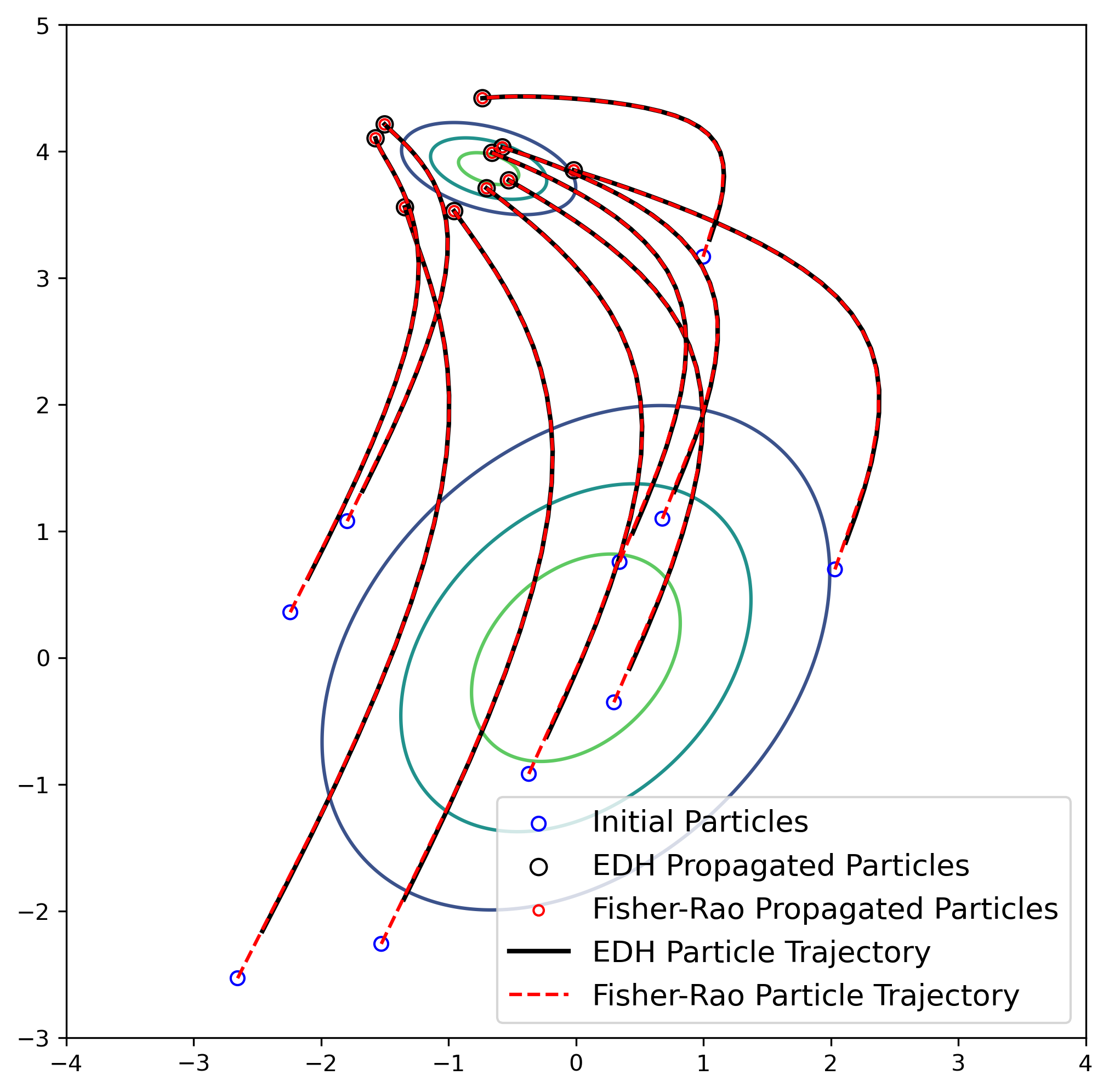}
    \caption{Comparison of the EDH flow and the Gaussian Fisher-Rao flow under linear Gaussian assumptions. We propagate 10 randomly selected particles through both flows. The trajectories of the particles are identical, verifying the results stated in Theorem~\ref{theorem: linear_gaussian_connection}.}
    \label{fig: linear_gaussian}
\end{figure}

\section{Approximated Gaussian Mixture Fisher-Rao Flows}
\label{sec: gaussian_mixture_approx}
This section focuses on the case where the variational density is selected as a Gaussian mixture density. Computing the Fisher information matrix associated with a Gaussian mixture density is costly, however. For better efficiency, we use the diagonal approximation of the FIM proposed in \cite{lin2019fast} and establish a connection between the approximated Gaussian mixture Fisher-Rao gradient flow and the particle dynamics in the particle flow. A Gaussian mixture density with $K$ components can be expressed as follows:
\begin{equation}
\label{gaussian_mixture_approx: gaussian_mixtuire_model}
    q(\bfx) = \sum_{k=1}^K q(\bfx | \omega = k) q(\omega = k), 
\end{equation}
where $q(\bfx | \omega = k) = p_{\calN}(\bfx; \bfmu^{(k)}, \Sigma^{(k)})$ and $q(\omega = k) = \pi^{(k)}$ is a multinomial PDF such that $\sum_{k=1}^{K} \pi^{(k)} = 1$. As noted by \cite{lin2019fast}, employing a natural parameterization of the Gaussian mixture density and an approximated FIM simplifies the derivation of an approximation of the Fisher-Rao parameter flow \eqref{fisher_rao: fr_para_flow}. Furthermore, the natural parameterization allows the component weight parameters to be expressed in real-valued log-odds, which eliminates the need for re-normalization to ensure $\sum_{k=1}^{K} \pi^{(k)} = 1$. Let $\bfeta$ denote the natural parameters of the Gaussian mixture density:
\begin{equation}
\label{gaussian_mixture_approx: natural_parameters}
    \eta^{(k)}_{\omega} = \log \bigg( \frac{\pi^{(k)}}{\pi^{(K)}} \bigg), \;
    \bfeta^{(k)}_x = \big( \bfgamma^{(k)}, \; \Gamma^{(k)} \big), \;  \bfgamma^{(k)} = [\Sigma^{(k)}]^{-1} \bfmu^{(k)}, \Gamma^{(k)} = -\frac{1}{2} [\Sigma^{(k)}]^{-1}, 
\end{equation}
where $\bfeta^{(k)}_x$ denotes the natural parameters of the $k$th Gaussian mixture component and $\eta^{(k)}_{\omega}$ denotes the natural parameter of the $k$th component weight. Notice that we have $\eta^{(K)}_{\omega} = 0$ and thus we set $\pi^{(K)} = 1 - \sum_{k=1}^{K-1} \pi^{(k)}$ in order to ensure $\sum_{k=1}^{K} \pi^{(k)} = 1$. The component weights from their natural parameterization are recovered via:
\begin{equation}
    \pi^{(k)} = \frac{\exp(\eta^{(k)})}{\sum_{j=1}^{K} \exp(\eta^{(j)})}.
\end{equation}
However, the FIM $\calI(\bfeta)$ defined in \eqref{fisher_rao: fisher_rao_tensor} of the Gaussian mixture variational density $q(\bfx)$ is difficult to compute. We use the block-diagonal approximation $\tilde{\calI}(\bfeta)$ of the FIM proposed in \cite{lin2019fast}:
\begin{equation}
\label{gaussian_mixture_approx: approx_FIM}
    \tilde{\calI}(\bfeta) = \diag \left( \calI(\bfeta^{(1)}_x, \eta^{(1)}_{\omega}), ...,\calI(\bfeta^{(K)}_x, \eta^{(K)}_{\omega}) \right), 
\end{equation}
where $\calI(\bfeta^{(k)}_x, \eta^{(k)}_{\omega})$ is the FIM of the $k$th joint Gaussian mixture component $q(\bfx | \omega = k) q(\omega = k)$. Using the approximated FIM, we can define the approximated Gaussian mixture Fisher-Rao parameter flow as follows:
\begin{equation}
\label{gaussian_mixture_approx: fr_para_flow}
    \frac{\d \bfeta_t}{\d t} = - \tilde{\calI}^{-1}(\bfeta_t) \nabla _{\bfeta_t} D_{KL}(q(\bfx; \bfeta_t) \| p(\bfx | \bfz) ).
\end{equation}
To ease notation, define:
\begin{equation}
\label{gaussian_mixture_approx: kl_kernel}
    V(\bfx; \bfeta) = \log(q(\bfx; \bfeta)) - \log(p(\bfx, \bfz)), 
\end{equation}
where $q(\bfx; \bfeta)$ is the variational density and $p(\bfx, \bfz) = p(\bfz | \bfx) p(\bfx)$ is the joint density. Due to the block-diagonal structure of the approximated FIM, the approximated Gaussian mixture Fisher-Rao parameter flow can be computed for the individual components, which is justified in the following proposition.

\begin{proposition}[Approximated Gaussian Mixture Fisher-Rao Parameter Flow]
Consider the approximated Gaussian mixture Fisher-Rao parameter flow \eqref{gaussian_mixture_approx: fr_para_flow}. The component-wise approximated Gaussian mixture Fisher-Rao parameter flow for each Gaussian mixture component $q(\bfx | \omega = k) q(\omega = k)$ takes the following form:
    \begin{equation}
    \begin{split}
    \label{gaussian_mixture_approx: component_natural_para_flow}
        &\frac{\d \bfgamma^{(k)}_t}{\d t} = - \bbE_{q(\bfx | \omega = k; \bfeta_t)} \left[\frac{1}{2} \nabla^2_{\bfx}V(\bfx; \bfeta_t)[\Gamma^{(k)}_t]^{-1} \bfgamma^{(k)}_t + \nabla_{\bfx}V(\bfx; \bfeta_t) \right], \\
        &\frac{\d \Gamma^{(k)}_t}{\d t} = - \frac{1}{2} \bbE_{q(\bfx | \omega = k; \bfeta_t)} \left[ \nabla^2_{\bfx}V(\bfx; \bfeta_t)\right], \\
        &\frac{\d [\eta^{(k)}_{\omega}]_t}{\d t} = \bbE_{q(\bfx | \omega = K; \bfeta_t)}\left[ V(\bfx; \bfeta_t) \right] - \bbE_{q(\bfx | \omega = k; \bfeta_t)}\left[ V(\bfx; \bfeta_t) \right].
    \end{split}
    \end{equation}
\end{proposition}
\begin{proof}
    According to the definition of the approximated FIM \eqref{gaussian_mixture_approx: approx_FIM}, the approximated Gaussian mixture Fisher-Rao parameter flow can be decomposed into $K$ components with each component taking the following form:
    \begin{equation}
        \frac{\d ([\bfeta^{(k)}_{x}]_t, [\eta^{(k)}_{\omega}]_t)}{\d t} = -\calI^{-1}([\bfeta^{(k)}_{x}]_t, [\eta^{(k)}_{\omega}]_t) \nabla _{([\bfeta^{(k)}_{x}]_t, [\eta^{(k)}_{\omega}]_t)} D_{KL}(q(\bfx; \bfeta_t) \| p(\bfx | \bfz) ). 
    \end{equation}
    According to \citet[Lemma~2]{lin2019fast}, the FIM $\calI(\bfeta^{(k)}_x, \eta^{(k)}_{\omega})$ of the $k$th joint Gaussian mixture component $q(\bfx | \omega = k) q(\omega = k)$ takes the following form:
    \begin{equation}
        \calI(\bfeta^{(k)}_x, \eta^{(k)}_{\omega}) = \diag(\pi^{(k)} \calI(\bfeta^{(k)}_{x}), \calI(\eta^{(k)}_{\omega})).
    \end{equation}
    As a result, we can further decompose \eqref{gaussian_mixture_approx: fr_para_flow} as follows: 
    \begin{equation}
    \label{gaussian_mixture_approx: component_flow_final}
    \begin{split}
        \frac{\d [\bfeta^{(k)}_{x}]_t}{\d t} = -\frac{1}{\pi^{(k)}} \calI^{-1}([\bfeta^{(k)}_{x}]_t) \nabla _{[\bfeta^{(k)}_{x}]_t} D_{KL}(q(\bfx; \bfeta_t) \| p(\bfx | \bfz) ), \\
        \frac{\d [\eta^{(k)}_{\omega}]_t}{\d t} = -\calI^{-1}([\eta^{(k)}_{\omega}]_t) \nabla _{[\eta^{(k)}_{\omega}]_t} D_{KL}(q(\bfx; \bfeta_t) \| p(\bfx | \bfz) ), 
    \end{split}
    \end{equation}
    where $\calI([\bfeta^{(k)}_{x}]_t)$ is the FIM of the $k$th Gaussian mixture component $q(\bfx | \omega = k)$ and $\calI([\eta^{(k)}_{\omega}]_t)$ is the FIM of the $k$th component weight $q(\omega = k)$. Also, according to \citet[Theorem~3]{lin2019fast}, the following equations hold:
    \begin{equation}
    \begin{split}
    \label{gaussian_mixture_approx: natural_gradient_relation}
        &\nabla_{[\bfm^{(k)}_x]_t} D_{KL}(q(\bfx; \bfeta_t) \| p(\bfx | \bfz) ) = \calI^{-1}([\bfeta^{(k)}_{x}]_t) \nabla _{[\bfeta^{(k)}_{x}]_t} D_{KL}(q(\bfx; \bfeta_t) \| p(\bfx | \bfz) ), \\
        &\nabla_{[m^{(k)}_{\omega}]_t} D_{KL}(q(\bfx; \bfeta_t) \| p(\bfx | \bfz) ) = \calI^{-1}([\eta^{(k)}_{\omega}]_t) \nabla _{[\eta^{(k)}_{\omega}]_t} D_{KL}(q(\bfx; \bfeta_t) \| p(\bfx | \bfz) ),
    \end{split}
    \end{equation}
    where $\bfm^{(k)}_x = \left( \bfmu^{(k)}, \bfmu^{(k)} [\bfmu^{(k)}]^{\top} + \Sigma^{(k)} \right)$ and $m^{(k)}_{\omega} = \pi^{(k)}$ denotes the expectation parameters of the $k$th Gaussian mixture component $q(\bfx | \omega = k)$ and the $k$th component weight $q(\omega = k)$, respectively. Using the chain rule, we can derive the following expression:
    \begin{equation}
    \label{gaussian_mixture_approx: component_derivative}
        \nabla_{\bfm^{(k)}_x} D_{KL} = \Big( \left( \nabla_{\bfmu^{(k)}} D_{KL} - 2 [\nabla_{\Sigma^{(k)}} D_{KL}] \bfmu^{(k)} \right), [\nabla_{\Sigma^{(k)}} D_{KL}] \Big).
    \end{equation}
    The gradient of the KL divergence with respect to $\bfmu^{(k)}$ and $\Sigma^{(k)}$ can be expressed in terms of the gradient and Hessian of $V(\bfx)$ defined in \eqref{gaussian_mixture_approx: kl_kernel} by using Bonnet's and Price's theorem, cf.~\cite{bonnet1964transformations, price1958useful, lin2019stein}:
    \begin{equation}
    \label{gaussian_mixture_approx: component_derivative_in_h}
        \nabla_{\bfmu^{(k)}} D_{KL} = \bbE_{q(\bfx | \omega = k)} \left[ \pi^{(k)}\nabla_{\bfx}V(\bfx)\right], \quad \nabla_{\Sigma^{(k)}} D_{KL} = \frac{1}{2} \bbE_{q(\bfx | \omega = k)} \left[ \pi^{(k)} \nabla^2_{\bfx}V(\bfx)\right]. 
    \end{equation}
    Substituting \eqref{gaussian_mixture_approx: component_derivative_in_h} into \eqref{gaussian_mixture_approx: component_derivative}, we have:
    \begin{equation}
    \label{gaussian_mixture_approx: expected_derivative_final}
        \nabla_{\bfm^{(k)}_x} D_{KL} = \Big(  \pi^{(k)} \bbE_{q(\bfx | \omega = k)} \left[ \nabla_{\bfx}V(\bfx) - \nabla^2_{\bfx}V(\bfx) \bfmu^{(k)} \right], \frac{\pi^{(k)}}{2} \bbE_{q(\bfx | \omega = k)} \left[\nabla^2_{\bfx}V(\bfx)\right] \Big).
    \end{equation}
    The gradient of the variational density with respect to the expectation parameter of the $k$th component weight takes the following form:
    \begin{equation}
        \nabla_{m^{(k)}_{\omega}} q(\bfx; \eta) = q(\bfx | \omega = k) - q(\bfx | \omega = K), 
    \end{equation}
    where the second term appears due to the fact $\pi^{(K)} = 1 - \sum_{k=1}^{K-1} \pi^{(k)}$. Utilizing the fact above, the gradient of the KL divergence with respect to the expectation parameter of the $k$th component weight takes the following form:
    \begin{equation}
    \begin{split}
    \label{gaussian_mixture_approx: weight_derivative_in_h}
        \nabla_{m^{(k)}_{\omega}} D_{KL} &= \int V(\bfx) \nabla_{m^{(k)}_{\omega}} q(\bfx; \eta) \d \bfx + \bbE_{q(\bfx; \eta)}\left[ \nabla_{m^{(k)}_{\omega}} \log(q(\bfx; \eta)) \right] \\
        &= \bbE_{q(\bfx | \omega = k)}\left[ V(\bfx) \right] - \bbE_{q(\bfx | \omega = K)}\left[ V(\bfx) \right] + \int q(\bfx | \omega = k) - q(\bfx | \omega = K)\d \bfx \\
        &= \bbE_{q(\bfx | \omega = k)}\left[ V(\bfx) \right] - \bbE_{q(\bfx | \omega = K)}\left[ V(\bfx) \right]. 
    \end{split}
    \end{equation}
    The desired results can be obtained by combining \eqref{gaussian_mixture_approx: expected_derivative_final}, \eqref{gaussian_mixture_approx: weight_derivative_in_h}, \eqref{gaussian_mixture_approx: natural_gradient_relation} and \eqref{gaussian_mixture_approx: component_flow_final}. 
\end{proof}

The approximated Gaussian mixture Fisher-Rao parameter flow \eqref{gaussian_mixture_approx: component_natural_para_flow} defines a time rate of change of the $k$th Gaussian mixture component. The corresponding particle flow for the $k$th Gaussian mixture component can be obtained by finding a particle dynamics function $\tilde{\phi}_k(\bfx, t)$ such that the following equation holds:
\begin{equation}
\label{gaussian_mixture_approx: component_fr_liouville}
    \frac{\partial q(\bfx | \omega = k; \bfeta_t)}{\partial \bfgamma^{(k)}_t} \frac{\d \bfgamma^{(k)}_t}{\d t} + \tr\bigg( \frac{\partial q(\bfx | \omega = k; \bfeta_t)}{\partial \Gamma^{(k)}_t} \frac{\d \Gamma^{(k)}_t}{\d t} \bigg) = - \nabla_{\bfx} \cdot \big( q(\bfx | \omega = k; \bfeta_t) \tilde{\phi}_k(\bfx, t) \big).
\end{equation}
The closed-form expression for this particle dynamics function $\tilde{\phi}_k(\bfx, t)$ is given in the following result.

\begin{proposition}[Approximated Gaussian Mixture Fisher-Rao Particle Flow]
    \label{prop:gm_fr_flow}
    The time rate of change of the $k$th Gaussian mixture component $q(\bfx | \omega=k)$ induced by the approximated Gaussian mixture Fisher-Rao parameter flow~\eqref{gaussian_mixture_approx: component_natural_para_flow} is captured by the  Liouville equation:
    \begin{equation}
    \label{gaussian_mixture_approx: natural_para_liouville}
    \begin{split}
        \frac{\d q(\bfx | \omega=k; \bfeta_t)}{\d t} = - \nabla_{\bfx} \cdot \big( q(\bfx | \omega=k; \bfeta_t) \tilde{\phi}_k(\bfx, t) \big), 
    \end{split}
    \end{equation} 
    with particle dynamics function $\tilde{\phi}_k(\bfx, t) = \tilde{A}^{(k)}_t \bfx + \tilde{\bfb}^{(k)}_t$,
    where 
    \begin{equation}\label{gaussian_mixture_approx: component_particle_flow}
    \begin{split}
        \tilde{A}^{(k)}_t &= \frac{1}{4} [\Gamma^{(k)}_t]^{-1} \bbE_{q(\bfx | \omega = k; t)} \left[\nabla^2_{\bfx}V(\bfx; \bfeta_t)\right], \\
        \tilde{\bfb}^{(k)}_t &= \frac{1}{2} [\Gamma^{(k)}_t]^{-1} \bbE_{q(\bfx | \omega = k; t)} \left[\nabla_{\bfx}V(\bfx; \bfeta_t)\right] + \frac{1}{2} \tilde{A}^{(k)}_t [\Gamma^{(k)}_t]^{-1} \bfgamma^{(k)}_t,
    \end{split}
    \end{equation}
    with $V(\bfx; \bfeta_t)$ given in \eqref{gaussian_mixture_approx: kl_kernel}, $\Gamma^{(k)}$ and $\bfgamma^{(k)}$ are natural parameters of the $k$th Gaussian mixture components given in \eqref{gaussian_mixture_approx: natural_parameters}.
\end{proposition}
\begin{proof}
    The gradient of a Gaussian PDF $p_{\calN}(\bfx; \bfmu, \Sigma)$ with respect to its natural parameters \eqref{gaussian_mixture_approx: natural_parameters} is given by:
    \begin{equation}
    \begin{split}
        \frac{\partial p_{\calN}(\bfx; \bfmu, \Sigma)}{\partial \bfgamma} &= p_{\calN}(\bfx; \bfmu, \Sigma) \big( \bfx^{\top} + \frac{1}{2} \bfgamma^{\top} \Gamma^{-1} \big), \\
        \frac{\partial p_{\calN}(\bfx; \bfmu, \Sigma)}{\partial \Gamma} &= p_{\calN}(\bfx; \bfmu, \Sigma) \big( \bfx \bfx^{\top} - \frac{1}{4} \Gamma^{-1} \bfgamma \bfgamma^{\top} \Gamma^{-1} + \frac{1}{2} \Gamma^{-1} \big). 
    \end{split}
    \end{equation}
    The time rate of change of the $k$th Gaussian mixture component $q(\bfx | \omega=k)$ induced by the approximated Gaussian mixture Fisher-Rao parameter flow~\eqref{gaussian_mixture_approx: component_natural_para_flow} is given by:
    \begin{equation}
        \frac{\partial q(\bfx | \omega = k; \bfeta_t)}{\partial \bfgamma^{(k)}_t} \frac{\d \bfgamma^{(k)}_t}{\d t} + \tr\bigg( \frac{\partial q(\bfx | \omega = k; \bfeta_t)}{\partial \Gamma^{(k)}_t} \frac{\d \Gamma^{(k)}_t}{\d t} \bigg), 
    \end{equation}
    where $\frac{\d \bfgamma^{(k)}_t}{\d t}$ and $\frac{\d \Gamma^{(k)}_t}{\d t}$ are given in \eqref{gaussian_mixture_approx: component_natural_para_flow}. With these expressions, the proof now proceeds analogously to the proof of Lemma~\ref{lemma: natural_particle_flow}.
\end{proof}
\begin{remark}
    Using the definition~\eqref{gaussian_mixture_approx: natural_parameters} of the natural parameters and the chain rule, one can express the approximated Gaussian mixture Fisher-Rao parameter flow in conventional parameters as follows:
    \begin{equation}
    \label{gaussian_mixture_approx: unified_mixture_gaussian_fr_para_flow}
        \frac{\d \bfmu^{(k)}_{t}}{\d t} = -\Sigma^{(k)}_t \bbE_{q(\bfx|\omega = k; \bfeta_t)}\left[\nabla_{\bfx}V(\bfx)\right], \quad \frac{\d (\Sigma^{(k)}_{t})^{-1}}{\d t} = \bbE_{q(\bfx|\omega = k; \bfeta_t)}\left[\nabla^2_{\bfx}V(\bfx)\right], 
    \end{equation}
    as well as the particle dynamics function $\tilde{\phi}_k(\bfx, t)$ governing the approximated Gaussian mixture Fisher-Rao particle flow
    \begin{equation}\label{gaussian_mixture_approx: unified_mixture_gaussian_fr_particle_flow}
    \begin{split}
        \tilde{A}^{(k)}_t &= -\frac{1}{2} \Sigma^{(k)}_{t} \bbE_{q(\bfx|\omega = k; \bfeta_t)} \left[\nabla^2_{\bfx}V(\bfx)\right],
        \\
        \tilde{\bfb}^{(k)}_t &= \Sigma^{(k)}_{t} \bbE_{q(\bfx|\omega = k; \bfeta_t)}\left[\nabla_{\bfx}V(\bfx)\right] - \tilde{A}^{(k)}_t \bfmu^{(k)}_{t}.
    \end{split}
    \end{equation}
    
\end{remark}

\begin{algorithm}[t]
\caption{Approximated Gaussian Mixture Fisher-Rao Particle Flow}
\label{alg: gaussian_mixture_fr}
\begin{algorithmic}[1]
\small
\Require Parameters $\bfeta_0$ defined in \eqref{gaussian_mixture_approx: natural_parameters} specifying the initial variational density $q(\bfx; \bfeta_0)$, particles $\{\{[\bfx^{(k)}_j]_{t=0}\}_{j=1}^{M}\}_{k=1}^{K}$ sampled from the initial variational density, and joint density $p(\bfx, \bfz)$
\Output Particles $\{\{[\bfx^{(k)}_j]_{t=T}\}_{j=1}^{M}\}_{k=1}^{K}$ that approximate the posterior density $p(\bfx | \bfz)$

\Function{$\bff$}{$\{\{[\bfx^{(k)}_j]_{t=T}\}_{j=1}^{M}\}_{k=1}^{K}$, $\bfeta_t$, $t$}
    \For{each $k \in [1, K]$}
    \State $\delta \bfeta^{(k)}_t \gets$ Evaluate \eqref{gaussian_mixture_approx: component_natural_para_flow} with variational density parameters $\bfeta_t$
    \State $\tilde{A}^{(k)}_t$, $\tilde{\bfb}^{(k)}_t \gets$ Evaluate \eqref{gaussian_mixture_approx: component_particle_flow} with $\bfeta_t$ and $t$
    \For{each particle \text{$[\bfx^{(k)}_j]_t$}}
    \State $\tilde{\bfphi}([\bfx^{(k)}_j]_t, t) \gets \tilde{A}^{(k)}_t [\bfx^{(k)}_j]_t + \tilde{\bfb}^{(k)}_t$
    \EndFor
    \EndFor
\State \Return $\{\{\tilde{\bfphi}([\bfx^{(k)}_j]_t, t)\}_{j=1}^{M}\}_{k=1}^{K}$, $\delta \bfeta_t$
\EndFunction

\While{ODE solver running}
    \State $\{\{[\bfx^{(k)}_j]_{t=T}\}_{j=1}^{M}\}_{k=1}^{K}, \bfeta_T \gets$ SolveODE($\bff(\{\{[\bfx^{(k)}_j]_t\}_{j=1}^{M}\}_{k=1}^{K}, \bfeta_t, t)$) with initial particles $\{\{[\bfx^{(k)}_j]_{t=0}\}_{j=1}^{M}\}_{k=1}^{K}$, initial variational density parameters $\bfeta_0$, initial time $t=0$, and termination time $T$
\EndWhile

\State \Return $\{\{[\bfx^{(k)}_j]_{t=T}\}_{j=1}^{M}\}_{k=1}^{K}$
\end{algorithmic}
\end{algorithm}

Algorithm~\ref{alg: gaussian_mixture_fr} summarizes the key steps for the proposed approximated Gaussian mixture Fisher-Rao particle flow. The approximated Gaussian mixture Fisher-Rao particle flow derived in this section can capture multi-modal behavior in the posterior density. We find that the approximated Gaussian mixture Fisher-Rao flows in \eqref{gaussian_mixture_approx: unified_mixture_gaussian_fr_para_flow} and \eqref{gaussian_mixture_approx: unified_mixture_gaussian_fr_particle_flow} share a similar form as the Gaussian Fisher-Rao flows, cf. \eqref{gaussian_approx: approx_fr_para_flow} and \eqref{gaussian_approx: approx_particle_flow}, which indicates that the approximated Gaussian mixture Fisher-Rao particle flow can be interpreted as a weighted mixture of Gaussian Fisher-Rao particle flows. Note that all these flows require the evaluation of the expectation of the gradient and Hessian of the function $V(\bfx; \bfeta)$ defined in \eqref{gaussian_mixture_approx: kl_kernel}, which in the case of Gaussian mixtures is not readily available. In the following section, we show that one can utilize the particle flow method to obtain these expectations in an efficient way. 
\section{Derivative- and Inverse-Free Formulation of the Fisher-Rao Flows}
\label{sec: derivative_inverse_free}

In this section, we derive several identities associated with the particle flow method and show how they can make the computation of the flow parameters more efficient. To start, we notice the component-wise approximated Gaussian mixture Fisher-Rao parameter flow \eqref{gaussian_mixture_approx: unified_mixture_gaussian_fr_para_flow} takes the same form as the Gaussian Fisher-Rao parameter flow \eqref{gaussian_approx: approx_fr_para_flow}. To make the discussion clear, we consider the parameter flow of each Gaussian mixture component:
\begin{equation}
\label{dif: unified_parameter_flow}
    \frac{\d{} \bar{\bfmu}_{t}}{\d t} = -\bar{\Sigma}_t \bbE_{\bfx \sim \calN(\bar{\bfmu}_t, \bar{\Sigma}_t)}\left[\nabla_{\bfx}V(\bfx)\right], \quad \frac{\d{} \bar{\Sigma}^{-1}_{t}}{\d t} = \bbE_{\bfx \sim \calN(\bar{\bfmu}_t, \bar{\Sigma}_t)}\left[\nabla^2_{\bfx}V(\bfx)\right], 
\end{equation}
where $V(\bfx)$ is defined in \eqref{fisher_rao: kl_kernel}. Similarly, we consider the particle dynamics function governing the particle flow:
\begin{equation}
\label{dif: unified_particle_flow}
    \bar{\phi}(\bfx, t) = \bar{A}_t \bfx + \bar{\bfb}_t,
\end{equation}
where $\bar{A}_t = -\frac{1}{2} \bar{\Sigma}_t \bbE_{\bfx \sim \calN(\bar{\bfmu}_t, \bar{\Sigma}_t)}\left[\nabla^2_{\bfx}V(\bfx)\right]$, $\bar{\bfb}_t = -\bar{\Sigma}_t \bbE_{\bfx \sim \calN(\bar{\bfmu}_t, \bar{\Sigma}_t)} \left[\nabla_{\bfx}V(\bfx)\right] - \bar{A}_t \bar{\bfmu}_t$. We first find a derivative-free expression for the expectation terms in~\eqref{dif: unified_parameter_flow} and~\eqref{dif: unified_particle_flow}, and then compare several particle-based approximation methods to compute the resulting derivative-free expectation terms. 

To avoid computing derivatives of $V(\bfx)$, we apply Stein's lemma \citep{stein1981estimation}, yielding the following expressions:
\begin{equation}
\label{dif: derivative_free}
\begin{split}
    \bbE_{\bfx \sim \calN(\bar{\bfmu}_t, \bar{\Sigma}_t)}\left[\nabla_{\bfx}V(\bfx)\right] &= \bar{\Sigma}^{-1}_t \bbE_{\bfx \sim \calN(\bar{\bfmu}_t, \bar{\Sigma}_t)}\left[ (\bfx - \bar{\bfmu}_t) V(\bfx) \right], \\
    \bbE_{\bfx \sim \calN(\bar{\bfmu}_t, \bar{\Sigma}_t)}\left[\nabla^2_{\bfx}V(\bfx)\right] &= \bar{\Sigma}^{-1}_t \bbE_{\bfx \sim \calN(\bar{\bfmu}_t, \bar{\Sigma}_t)}\left[ (\bfx - \bar{\bfmu}_t)(\bfx - \bar{\bfmu}_t)^{\top} V(\bfx) \right] \bar{\Sigma}^{-1}_t \\
    & \quad - \bar{\Sigma}^{-1}_t \bbE_{\bfx \sim \calN(\bar{\bfmu}_t, \bar{\Sigma}_t)}\left[ V(\bfx) \right].
\end{split}
\end{equation}
Using the identities above, we obtain a derivative-free parameter flow:
\begin{equation}
\label{dif: unified_derivative_free_parameter_flow}
\begin{split}
    \frac{\d{} \bar{\bfmu}_{t}}{\d t} &= - \bbE_{\bfx \sim \calN(\bar{\bfmu}_t, \bar{\Sigma}_t)}\left[ (\bfx - \bar{\bfmu}_t) V(\bfx) \right],
    \\
    \frac{\d{} \bar{\Sigma}^{-1}_{t}}{\d t} &= \bar{\Sigma}^{-1}_t \bbE_{\bfx \sim \calN(\bar{\bfmu}_t, \bar{\Sigma}_t)}\left[ (\bfx - \bar{\bfmu}_t)(\bfx - \bar{\bfmu}_t)^{\top} V(\bfx) \right] \bar{\Sigma}^{-1}_t - \bar{\Sigma}^{-1}_t \bbE_{\bfx \sim \calN(\bar{\bfmu}_t, \bar{\Sigma}_t)}\left[ V(\bfx) \right].
\end{split}
\end{equation}
Similarly, we can obtain derivative-free expressions for $\bar{A}_t$ and $\bar{\bfb}_t$ describing the particle flow:
\begin{equation}
\label{dif: unified_derivative_free_particle_flow}
\begin{split}
    \bar{A}_t &= -\frac{1}{2} \bbE_{\bfx \sim\calN(\bar{\bfmu}_t, \bar{\Sigma}_t)}\left[ (\bfx - \bar{\bfmu}_t)(\bfx - \bar{\bfmu}_t)^{\top} V(\bfx) \right] \bar{\Sigma}^{-1}_t + \frac{1}{2} \bbE_{\bfx \sim \calN(\bar{\bfmu}_t, \bar{\Sigma}_t)}\left[ V(\bfx) \right]
    \\
    \bar{\bfb}_t &= -\bbE_{\bfx \sim \calN(\bar{\bfmu}_t, \bar{\Sigma}_t)}\left[ (\bfx - \bar{\bfmu}_t) V(\bfx) \right] - \bar{A}_t \bar{\bfmu}_t.
\end{split}
\end{equation}
Note that the derivative-free flow expressions in \eqref{dif: unified_derivative_free_parameter_flow} and \eqref{dif: unified_derivative_free_particle_flow} only depend on $\bar{\Sigma}^{-1}_t$. By propagating $\bar{\Sigma}^{-1}_t$ instead of $\bar{\Sigma}_t$, we can also avoid inefficient matrix inverse calculations in the flow expressions.

The computation of the expectation terms in \eqref{dif: unified_derivative_free_parameter_flow} and \eqref{dif: unified_derivative_free_particle_flow} analytically is generally not possible and often times numerical approximation is sought. Since the expectation is with respect to a Gaussian density, there are various accurate and efficient approximation techniques, such as linearization of the terms \citep{anderson2005optimal} and Monte-Carlo integration using Unscented or Gauss-Hermite particles \citep{julier1995new, sarkka2023bayesian}. We use a particle-based approximation for the expectation terms:
\begin{equation}
\label{dif: exp_approx}
    \bbE_{\bfx \sim\calN(\bar{\bfmu}_t, \bar{\Sigma}_t)}\left[ f(\bfx) \right] \approx \sum_{i=1}^{N} w_i f(\bfx^{(i)}_t), 
\end{equation}
where $w_i$ are weights, $\bfx^{(i)}_t \in \bbR^{n}$ are particles and $f(\bfx)$ is an arbitrary function. The particles and associated weights can be generated using the Gauss-Hermite cubature method \citep{sarkka2023bayesian}. We first generate $p$ Gauss-Hermite particles $\{ \xi_i \}_{i=1}^p$ of dimension one corresponding to $\calN(0, 1)$ by computing the roots of the Gauss-Hermite polynomial of order~$p$:
\begin{equation}
    h_p(\xi) = (-1)^{p} \exp(\xi^2/2) \frac{\d{}^p \exp(-\xi^2/2)}{\d \xi^p}, 
\end{equation}
with the corresponding weights $\{^{(1)}w_i \}_{i=1}^p$ obtained as follows:
\begin{equation}
    ^{(1)}w_i = \frac{p!}{(p h_{p-1}(\xi_i))^2}, 
\end{equation}
where we use the left superscript in $^{(1)}w_i$ to indicate that these weights correspond to the one-dimensional Gauss-Hermite quadrature rule. 
The Gauss-Hermite particles $\{ \bfxi_i \}_{i=1}^{p^n}$ of dimension $n$ correspond to $\calN(\mathbf{0}, I)$ are generated as the Cartesian product of the one-dimensional Gauss-Hermite particles:
\begin{equation}
\label{dif: gauss_hermite_particle_unit}
    \{ \bfxi_i \}_{i=1}^{p^n} = \{ (\xi_{i_1}, \xi_{i_2}, ..., \xi_{i_n}), \quad i_1, i_2, ..., i_n \in \{1, 2, ..., p \} \}, 
\end{equation}
with the corresponding weights $\{ w_i \}_{i=1}^{p^n}$ obtained as follows: 
\begin{equation}
\label{dif: gauss_hermite_weight_unit}
    \{ w_i \}_{i=1}^{p^n} = \{ ^{(1)}w_{i_1} ^{(1)}w_{i_2} ... ^{(1)}w_{i_n}, \quad i_1, i_2, ..., i_n \in \{1, 2, ..., p \} \}. 
\end{equation}
Finally, Gauss-Hermite particles $\{ \bfx_i \}_{i=1}^{p^n}$ of dimension $n$ corresponding to $\calN(\bar{\bfmu}_t, \bar{\Sigma}_t)$ are obtained as follows:
\begin{equation}
\label{dif: reparameterization_trick}
    \bfx^{(i)}_t = \bar{L}_t \bfxi_i + \bar{\bfmu}_t, 
\end{equation}
where $\bar{L}_t$ is a matrix square root that satisfies $\bar{\Sigma}_t = \bar{L}_t \bar{L}^{\top}_t$ and the superscript $i$ in $\bfx^{(i)}_t$ denotes the $i$th particle. The weights $\{ w_i \}_{i=1}^{p^n}$ remain unchanged. 

Although we only need to generate Gauss-Hermite particles $\{ \bfxi_i \}_{i=1}^{p^n}$ and corresponding weights $\{ w_i \}_{i=1}^{p^n}$ for a standard normal distribution once, we need to scale and translate the particles using $\bar{L}_t$ and $\bar{\bfmu}_t$ over time. However, we can avoid the scaling and translation to obtain the Gauss-Hermite particle corresponding to $\calN(\bar{\bfmu}_t, \bar{\Sigma}_t)$ by propagating the Gauss-Hermite particles using the particle flow described in \eqref{prelim: particle_dynamics} with pseudo dynamics given by \eqref{dif: unified_particle_flow}. In order to establish this, we need to introduce the following auxiliary result.
\begin{theorem}[Mahalanobis Distance is Invariant]
\label{theorem: maha_inv}
    Define the Mahalanobis distance as:
    \begin{equation}
        D_{\calM}(\bfx, \bfmu, \Sigma) \coloneqq (\bfx - \bfmu)^{\top} \Sigma^{-1} (\bfx - \bfmu).
    \end{equation}
    Consider a Gaussian distribution $\calN(\bar{\bfmu}_t, \bar{\Sigma}_t)$ with its parameters evolving according to \eqref{dif: unified_parameter_flow}. Let $\bfx_0$ be a particle that evolves according to \eqref{prelim: particle_dynamics} with the particle dynamics function given by \eqref{dif: unified_particle_flow}. Then, 
    \begin{equation}
        D_{\calM}(\bfx_t, \bar{\bfmu}_t, \bar{\Sigma}_t) = D_{\calM}(\bfx_0, \bar{\bfmu}_{t=0}, \bar{\Sigma}_{t=0}), \quad \forall t > 0.
    \end{equation}
\end{theorem}
\begin{proof}
    According to the chain rule, we have:
    \begin{equation}
    \begin{split}
        &\frac{\d D_{\calM}(\bfx_t, \bar{\bfmu}_t, \bar{\Sigma}_t)}{\d t} = 2 (\frac{\d \bfx_t}{\d t} - \frac{\d{} \bar{\bfmu}_t}{\d t})^{\top} \bar{\Sigma}^{-1}_t (\bfx_t - \bar{\bfmu}_t) + (\bfx_t - \bar{\bfmu}_t)^{\top} \frac{\d{} \bar{\Sigma}^{-1}_t}{\d t} (\bfx_t - \bar{\bfmu}_t) \\
        &= 2 (\bfx_t - \bar{\bfmu}_t)^{\top} \bar{A}^{\top}_t \bar{\Sigma}^{-1}_t (\bfx_t - \bar{\bfmu}_t) + (\bfx_t - \bar{\bfmu}_t)^{\top} \bbE_{\bfx \sim \calN(\bar{\bfmu}_t, \bar{\Sigma}_t)}\left[\nabla^2_{\bfx}V(\bfx)\right] (\bfx_t - \bar{\bfmu}_t)
        = 0, 
    \end{split}
    \end{equation}
    indicating that the time rate of change of the Mahalanobis distance is zero.
\end{proof}

We are ready to prove that Gauss-Hermite particles remain Gauss-Hermite when propagated along the  
Fisher-Rao particle flow.

\begin{theorem}[Fisher-Rao Particle Flow as Gauss-Hermite Transform]
\label{theorem: gh_preserve}
    Consider a Gaussian distribution $\calN(\bar{\bfmu}_t, \bar{\Sigma}_t)$ with parameters evolving according to \eqref{dif: unified_parameter_flow}. Let $\bfx_0$ be a Gauss-Hermite particle obtained from $\calN(\bar{\bfmu}_{t=0}, \bar{\Sigma}_{t=0})$ which evolves according to \eqref{prelim: particle_dynamics} with particle dynamics function given by \eqref{dif: unified_particle_flow}. Then, $\bfx_t$ is a Gauss-Hermite particle corresponding to $\calN(\bar{\bfmu}_t, \bar{\Sigma}_t)$. 
\end{theorem}
\begin{proof}
    Let $\bfxi$ be arbitrary, consider the particle $\bfxi_t = \bar{L}_{t} \bfxi + \bar{\bfmu}_{t}$, where $\bar{\bfmu}_{t}$ evolves according to \eqref{dif: unified_parameter_flow} and $\bar{L}_{t}$ evolves according to the following equation:
    \begin{equation}
    \label{dif: covariance_square_root}
        \frac{\d{} \bar{L}_t}{\d t} = \bar{A}_t \bar{L}_t, 
    \end{equation}
    with $\bar{A}_t$ given in \eqref{dif: unified_particle_flow}. We first show $\bfxi_t$ is the solution to the particle flow \eqref{prelim: particle_dynamics}, where the particle dynamics function is given by \eqref{dif: unified_particle_flow}. Next, we show that if we have $\bar{L}_{t=0}$ satisfy $\bar{\Sigma}_{t=0} = \bar{L}_{t=0} \bar{L}^{\top}_{t=0}$ and the evolution of $\bar{L}_{t}$ satisfies \eqref{dif: covariance_square_root}, then $\bar{\Sigma}_t = \bar{L}_t \bar{L}^{\top}_t$. 
    
    According to the chain rule, the evolution of the particle $\bfxi_t$ satisfies:
    \begin{equation}
       \frac{\d \bfxi_{t}}{\d t} = \frac{\d{} \bar{L}_t}{\d t} \bfxi + \frac{\d{} \bar{\bfmu}_t}{\d t} = \bar{A}_t \bar{L}_t \bfxi + \bar{A}_t \bar{\bfmu}_t + \bar{\bfb}_t = \bar{A}_t \bfxi_{t} + \bar{\bfb}_t, 
    \end{equation}
    which is exactly the same as the particle flow \eqref{prelim: particle_dynamics} with the particle dynamics given in \eqref{dif: unified_particle_flow}. This justifies our first claim. According to Theorem~\ref{theorem: maha_inv}, we have:
    \begin{equation}\label{eq:auxx1}
        D_{\calM}(\bfxi_t, \bar{\bfmu}_t, \bar{\Sigma}_t) = \bfxi^{\top} \bar{L}^{\top}_{t} \Sigma_t^{-1} \bar{L}_{t} \bfxi = \bfxi^{\top} \bfxi.
    \end{equation}
    Since $\bar{L}^{\top}_{t} \Sigma_t^{-1} \bar{L}_{t}$ is symmetric, the following equation holds:
    \begin{equation}\label{eq:auxx2}
         \bar{L}^{\top}_{t} \Sigma_t^{-1} \bar{L}_{t} = I,
    \end{equation}
    indicating that $\bar{\Sigma}^{-1}_t = [\bar{L}^{-1}_t]^{\top} \bar{L}^{-1}_t$ and $\bar{\Sigma}_t = \bar{L}_t \bar{L}^{\top}_t$. This justifies our second claim. Hence, a Gauss-Hermite particle corresponding to $\calN(\bar{\bfmu}_t, \bar{\Sigma}_t)$ can be obtained by solving the particle flow \eqref{prelim: particle_dynamics} with the particle dynamics function given by \eqref{dif: unified_particle_flow}.
\end{proof}

This result shows that, instead of obtaining Gauss-Hermite particles corresponding to $\calN(\bar{\bfmu}_t, \bar{\Sigma}_t)$, we can propagate Gauss-Hermite particles according to the particle flow in \eqref{prelim: particle_dynamics} with the particle dynamics given in \eqref{dif: unified_particle_flow}, starting from the Gauss-Hermite particles corresponding to $\calN(\bar{\bfmu}_{t=0}, \bar{\Sigma}_{t=0})$. 

\begin{remark}[Propagating Covariance in Square Root Form]
    According to the proof of Theorem~\ref{theorem: gh_preserve}, we have that if $\bar{L}_{t=0}$ satisfies $\bar{\Sigma}_{t=0} = \bar{L}_{t=0} \bar{L}^{\top}_{t=0}$ and the evolution of $\bar{L}_{t}$ satisfies \eqref{dif: covariance_square_root}, then $\bar{\Sigma}_t = \bar{L}_t \bar{L}^{\top}_t$. In this regard, we can propagate the inverse covariance matrix in square-root form as:
    \begin{equation}
        \frac{\d{} \bar{L}^{-1}_t}{\d t} = - \bar{L}^{-1}_t \bar{A}_t, 
    \end{equation}
    where $\bar{L}_{t=0}$ satisfies $\bar{\Sigma}_{t=0} = \bar{L}_{t=0} \bar{L}^{\top}_{t=0}$. This flow 
    enforces that the inverse covariance matrix stays symmetric and positive definite during propagation. Note that the matrix square root decomposition is not unique and, hence, depending on the initialization of this flow, the evolution of the matrix square root will differ. Our proposed method is a special case of the one in \cite{morf1977square}, while other solutions exist, such as the one proposed in \cite{sarkka2009unscent}, which considers the lower triangular structure of the covariance square root matrix.
\end{remark}

In this section, we demonstrated several important identities satisfied by the particle flow and its underlying variational density. Corollary~\ref{theorem: gh_preserve} demonstrates that Gauss-Hermite particles, generated by a specific covariance square root, evolve according to the particle flow with the particle dynamics function defined by \eqref{dif: unified_particle_flow}. Consequently, evaluating the Gaussian probability density function at these particle locations amounts to rescaling the density evaluated at the particles' initial positions. However, this result should be used with caution when dealing with the Gaussian mixture variational distribution since it only applies to particles associated with their respective Gaussian mixture components. In the next section, we present numerical results to demonstrate the accuracy of the Fisher-Rao particle flow method. We also evaluate the expectation terms in \eqref{dif: unified_parameter_flow} and \eqref{dif: unified_particle_flow} using the methods discussed in this section and compare their accuracy.

\section{Evaluation}
\label{sec: Evaluation}

In this section, we first evaluate our Fisher-Rao flows in two low-dimensional scenarios. In the first scenario, the prior density is a Gaussian mixture, and the likelihood function is obtained from a linear Gaussian observation model. This leads to a multi-modal posterior density. In the second scenario, the prior density is a single Gaussian, and the likelihood function is obtained from a nonlinear Gaussian observation model. This leads to a posterior density that is not Gaussian. We compare the approximated Gaussian mixture Fisher-Rao particle flow with the Gaussian sum particle flow \citep{zhang2024multisensor}, the particle flow Gaussian mixture model (PF-GMM) \citep{pal2017gaussian}, the particle flow particle filter Gaussian mixture model (PFPF-GMM) \citep{pal2018particle}, and the Wasserstein gradient flow \citep{lambert2022wass}. We also demonstrate numerically that using Stein's gradient and Hessian \eqref{dif: derivative_free} with Gauss-Hermite particles \eqref{dif: reparameterization_trick} leads to particle-efficient and stable approximation of the expectation terms \eqref{dif: unified_parameter_flow}. Finally, we evaluate our Fisher-Rao flows in a high-dimensional scenario, using the posterior generated in the context of Bayesian logistic regression with dimensions $50$ and $100$ \citep{lambert2022wass}.

\subsection{Gaussian Mixture Prior}
\label{sec: eval_gmm}
We consider an equally weighted Gaussian mixture with four components as a prior density:
\begin{equation}
    p(\bfx) = \frac{1}{4} \sum_{i=1}^{4} p_{\calN}(\bfx; \hat{\bfx}^{(i)}, P),
\end{equation}
where
\begin{equation}
    \hat{\bfx}^{(i)} = \begin{bmatrix} \pm 5 \\ \pm 5 \end{bmatrix}, \quad P = \begin{bmatrix} 5 & 0 \\ 0 & 5 \end{bmatrix} .
\end{equation}
The likelihood function is also a Gaussian density $p_{\calN}(\bfz; H\bfx, R)$, with 
\begin{equation}
   H = \begin{bmatrix} 2 & -0.2 \\ 0.3 & 2.5 \end{bmatrix}, \quad R = \begin{bmatrix} 170 & 64 \\ 64 & 230 \end{bmatrix}, \quad \bfx^* = \begin{bmatrix} 2.67 \\ 1.67 \end{bmatrix}, 
\end{equation}
where $\bfx^*$ denotes the true value of $\bfx$ used to generate an observation $\bfz$. We set a large observation covariance to clearly separate the four modes of the posterior distribution. 

For the approximated Gaussian mixture Fisher-Rao particle flow \eqref{gaussian_mixture_approx: natural_para_liouville}, the initial mean parameter for the $k$th Gaussian mixture component $\bfmu^{(k)}_{t=0}$ is sampled from one of the prior Gaussian mixture components $\bfmu^{(k)}_{t=0} \sim \calN(\hat{\bfx}^{(i)}, P)$. The initial variance parameters for the $k$th Gaussian mixture component is set to $\Sigma^{(k)}_{t=0} = 3P$. The initial weight for the $k$th Gaussian mixture component is set to $\omega^{(k)}_{t=0} = 1 / K$, where $K$ denotes the total number of Gaussian mixture components employed in our approach. 
Each Gaussian mixture component is associated with $16$ Gauss-Hermite particles generated according to \eqref{dif: reparameterization_trick}. The Wasserstein gradient flow method \citep{lambert2022wass} uses the same initialization approach as the approximated Gaussian mixture Fisher-Rao particle flows. For the Gaussian sum particle flow method \citep{zhang2024multisensor}, each Gaussian particle flow component is initialized with $1000$ randomly sampled particles from $\calN(\bfmu^{(k)}, P)$, where $\bfmu^{(k)}$ is sampled from one of the prior Gaussian mixture components $\bfmu^{(k)} \sim \calN(\hat{\bfx}^{(i)}, P)$. This initialization method is employed for the Gaussian sum particle flow to ensure consistency with the other two methods, as it utilizes particles to compute the empirical mean during propagation. For the PF-GMM method \citep{pal2017gaussian} and the PFPF-GMM method \citep{pal2018particle}, we use an initial density identical to the prior, which is a four-component Gaussian mixture. Both methods are initialized with the same set of particles, generated by drawing $1000$ samples from each mixture component. For all methods, we use the following discrete KL divergence approximation:
\begin{equation}
    D_{KL}(q(\bfx) \| p(\bfx | \bfz)) \approx \frac{1}{N}\sum_{i=1}^{N} \log\left( \frac{q(\bfx_i)}{p(\bfx_i, \bfz)} \right) + \frac{1}{N}\log\left( \sum_{j=1}^{N} p(\bfx_j, \bfz) \right), 
\end{equation}
where $p(\bfx, \bfz) = p(\bfz | \bfx) p(\bfx)$ denotes the joint density, $\bfx_i$ are the particles, and $N$ is the total number of particles. The additional summation term over the joint density evaluations provides an approximation of the posterior normalizing constant. To ensure an accurate approximation of the KL divergence, we use $10^6$ particles uniformly distributed on a grid over the region $[-15, 15] \times [-15, 15]$. However, we need a parametric representation of the posterior approximation of PF-GMM and PFPF-GMM to apply the discrete KL divergence approximation described above. For the PF-GMM method, the component weights and covariance parameters are obtained via the parallel EKF update, and the mean parameter of each Gaussian mixture component is computed as the empirical mean of the particles associated with that component. For the PFPF-GMM method, the mean and covariance parameters for each Gaussian mixture component are estimated from the weighted particles using their empirical mean and empirical covariance, and the component weights are obtained using the same approach employed for PF-GMM.

Figure~\ref{fig: single_gaussian_mp} shows the results for the approximated Gaussian mixture Fisher-Rao particle flow \eqref{dif: reparameterization_trick}, the Gaussian particle flow \citep{zhang2024multisensor}, and the Wasserstein gradient flow \citep{lambert2022wass} when a single Gaussian is used to approximate the posterior density. This test demonstrates that the Gaussian particle flow \citep{zhang2024multisensor} exhibits pronounced sensitivity to its initialization, converging to the posterior component from which the initial particle ensemble is drawn. On the other hand, the Wasserstein gradient flow \citep{lambert2022wass} consistently converges to the posterior component with the highest weight. Our Gaussian Fisher-Rao particle flow \eqref{gaussian_approx: natural_para_liouville} achieves a balance between fitting the dominant posterior component and representing the remaining components, thereby achieving the lowest KL divergence. Figure~\ref{fig: mp} shows the results for each method when using a Gaussian mixture to approximate the posterior density. For the approximated Gaussian mixture Fisher-Rao particle flow \eqref{dif: reparameterization_trick}, the Gaussian sum particle flow \citep{zhang2024multisensor}, and the Wasserstein gradient flow \citep{lambert2022wass}, we use a Gaussian mixture with $20$ components. For the PF-GMM \citep{pal2017gaussian} and the PFPF-GMM \citep{pal2018particle}, we use a Gaussian mixture with $4$ components to match their formulation. In this case, the Gaussian sum particle flow fails to capture the four posterior components, resulting in the highest KL divergence. The Wasserstein gradient flow successfully captures the locations of the four different components of the posterior. However, it fails to capture the component weights. The PF-GMM, PFPF-GMM, and our approximated Gaussian mixture Fisher-Rao particle flow capture the locations and the weights of the four posterior components. Among these methods, the PF-GMM yields the lowest KL divergence approximation. The effectiveness of PF-GMM and PFPF-GMM can be explained by the fact that they require specific problem information, as both methods construct an EDH flow for each prior component. For the linear observation model considered in this case, each EDH flow satisfies the linear Gaussian assumptions. Hence, according to Lemma~\ref{lemma: exact_flow_exact}, each EDH flow yields an exact migration of the particles. However, our approximated Gaussian mixture Fisher-Rao particle flow achieves comparable performance to PF-GMM and PFPF-GMM without exploiting this problem-specific structure. Figure~\ref{fig: gmm_gradient} shows a comparison of the approximation accuracy of the expectation terms. In this test case, we observe that when using Stein's gradient and Hessian \eqref{dif: derivative_free}, Gauss-Hermite particles \eqref{dif: reparameterization_trick} of degree $4$ are accurate enough to ensure stable propagation, since they achieve comparable accuracy to the results obtained using Gauss-Hermite particles of degree $32$. However, Gauss-Hermite particles will lead to degraded accuracy when using the analytical gradient and Hessian.

\begin{figure}[t]
    \centering
    \includegraphics[width=0.48\linewidth,valign=t]{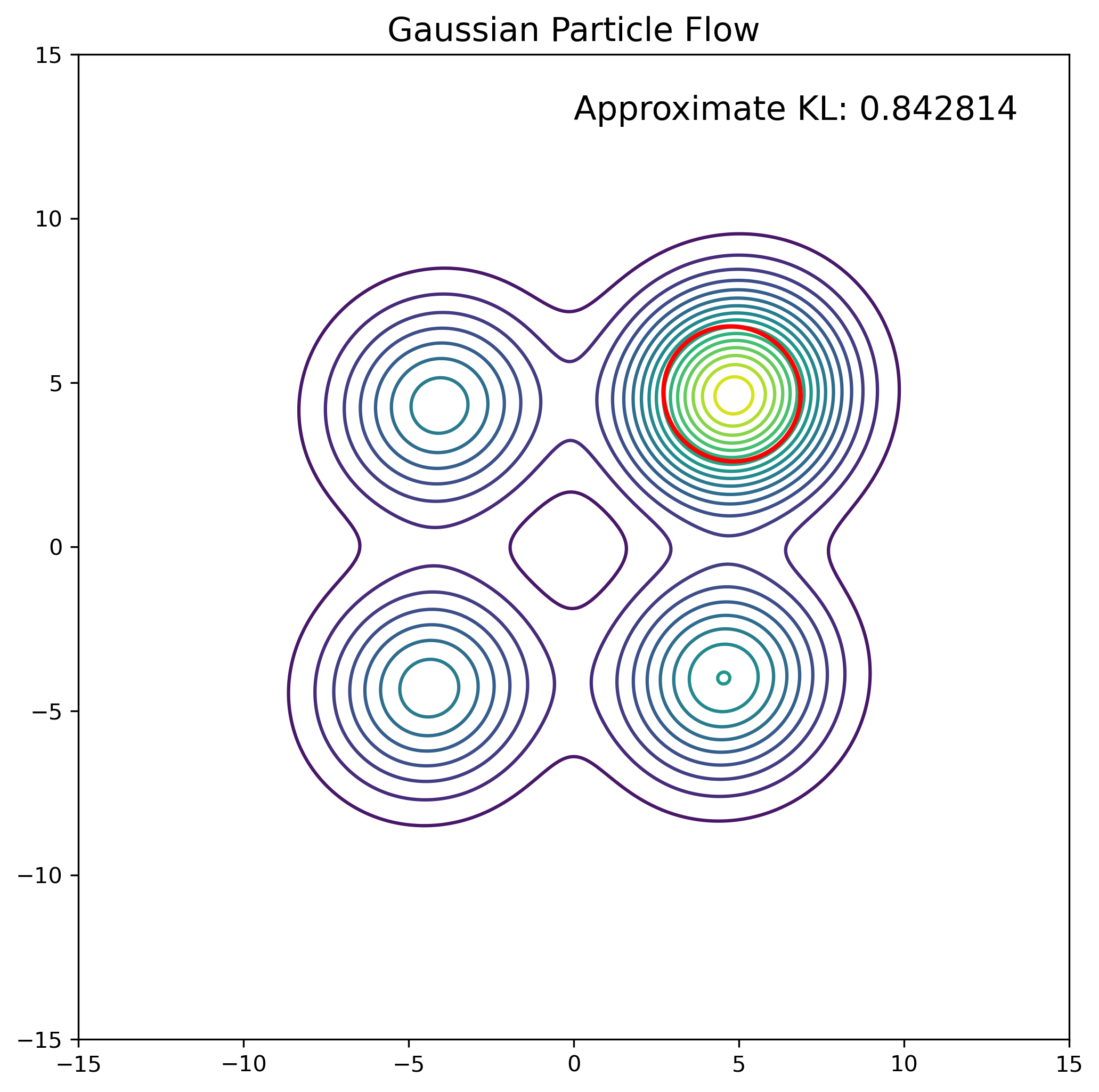}%
    \hfill%
    \includegraphics[width=0.48\linewidth,valign=t]{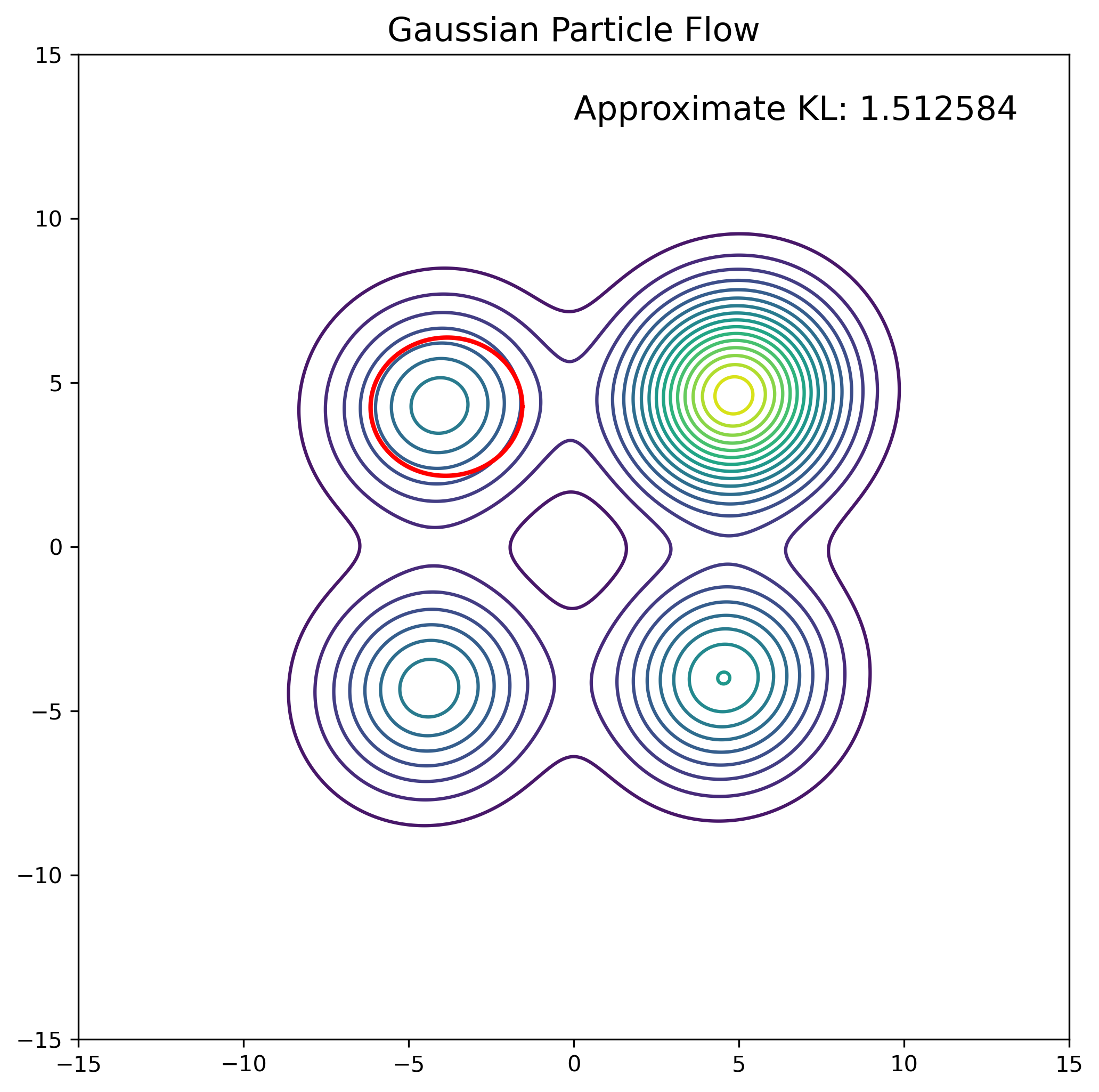}
    \includegraphics[width=0.48\linewidth,valign=t]{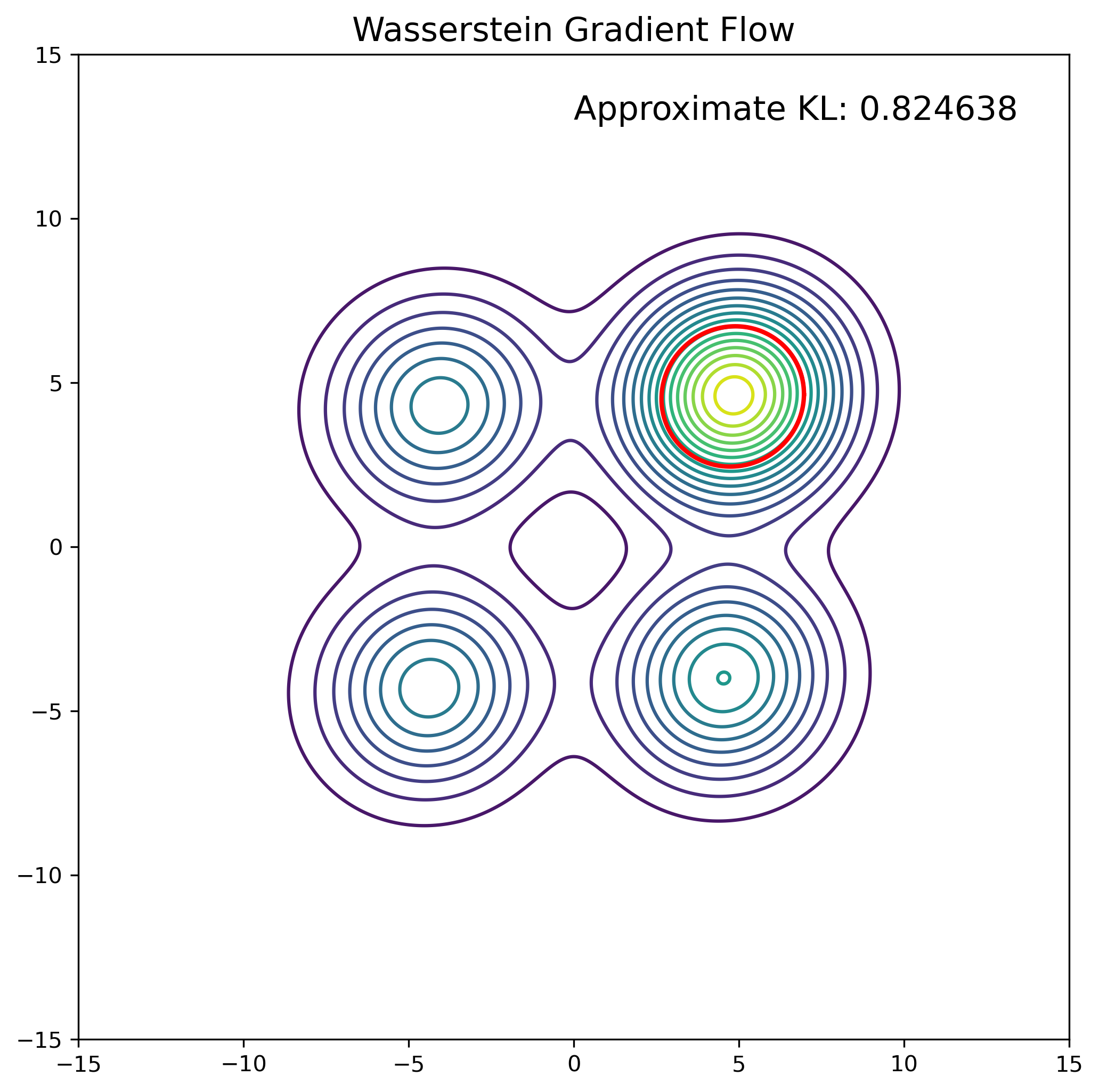}%
    \hfill%
    \includegraphics[width=0.48\linewidth,valign=t]{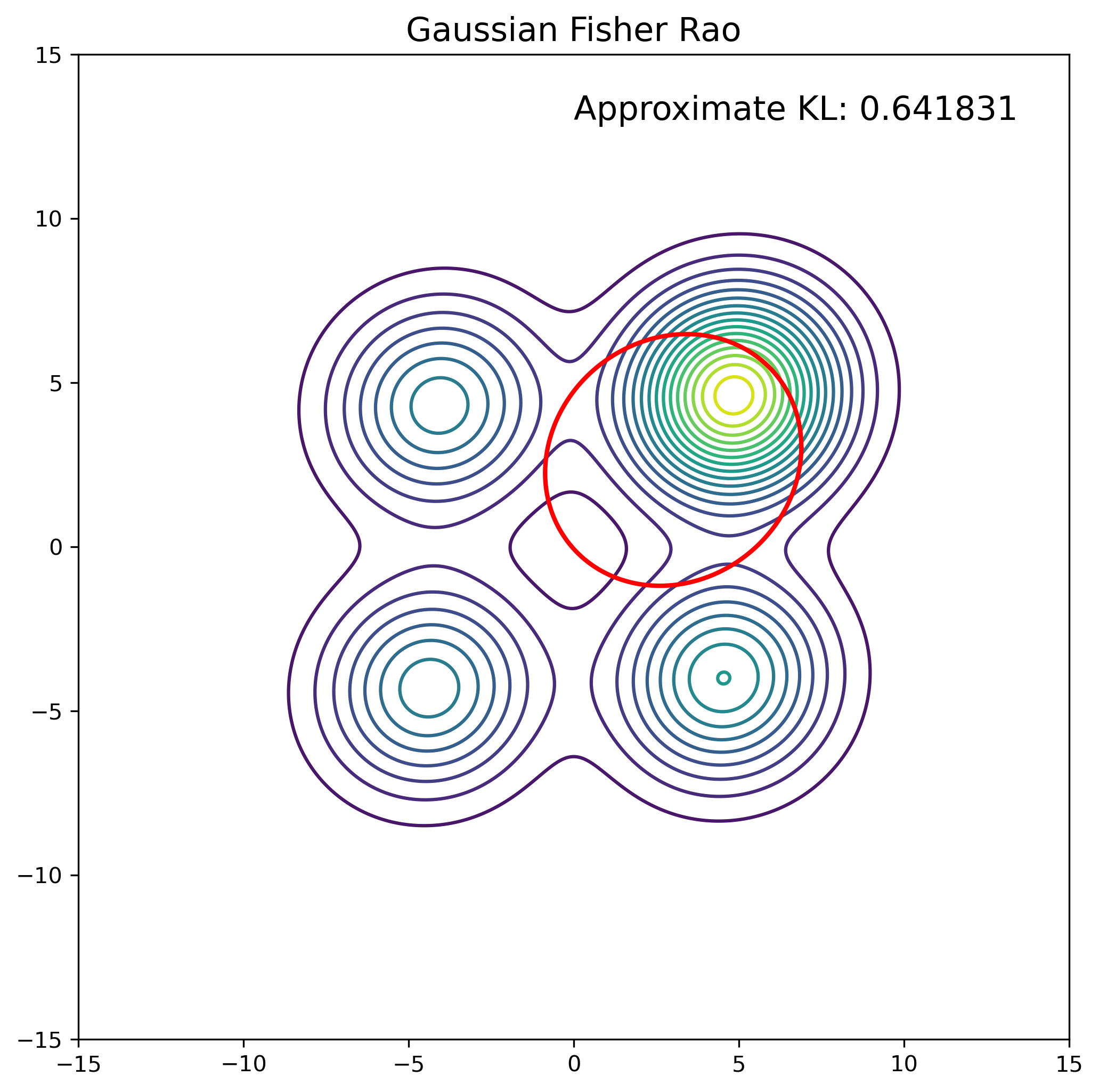}
    \caption{Comparison of the Gaussian Fisher-Rao particle flow \eqref{gaussian_approx: natural_para_liouville} with the Gaussian particle flow \citep{zhang2024multisensor} and the Wasserstein gradient flow \citep{lambert2022wass} for the Gaussian mixture prior case. For each method, we use a single Gaussian to approximate the posterior density. The covariance contour corresponding to one Mahalanobis distance is overlaid on the reference contour. The top two figures display results generated by Gaussian particle flow with different initializations. The top left figure shows the result with particles generated from $\calN(\hat{\bfx}^{(1)}, P)$, while the top right figure shows the result with particles generated from $\calN(\hat{\bfx}^{(2)}, P)$. This method is sensitive to initialization, and as a result, its accuracy is highly dependent on the initial condition. The bottom left figure shows the result generated by Wasserstein gradient flow, which converges to the mode with the largest component weight. The bottom right figure shows the result generated by our Gaussian Fisher-Rao particle flow, which achieves the lowest approximated KL divergence.}
    \label{fig: single_gaussian_mp}
\end{figure}

\begin{figure}[t]
    \centering
    \includegraphics[width=0.32\linewidth,valign=t]{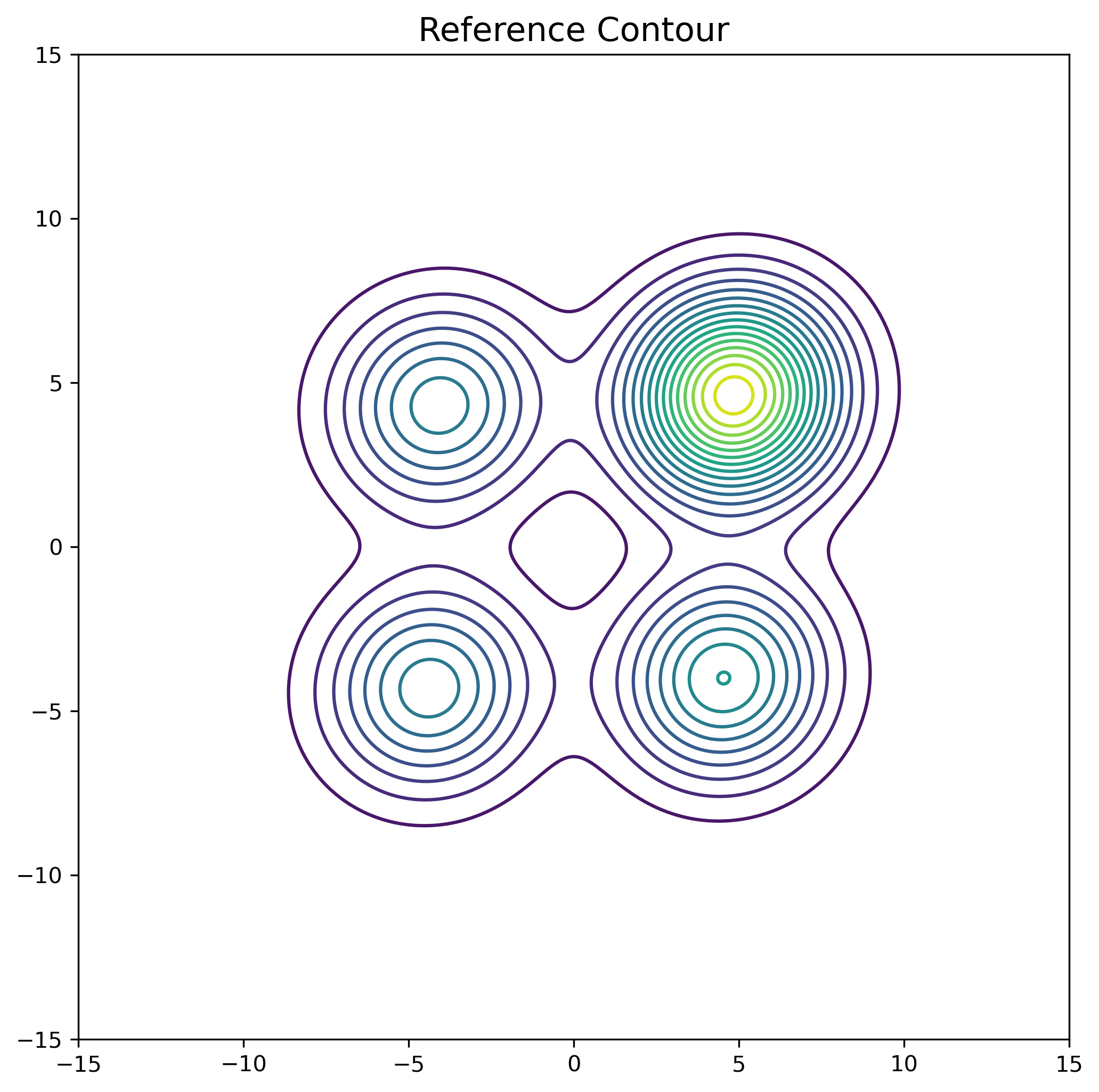}%
    \includegraphics[width=0.32\linewidth,valign=t]{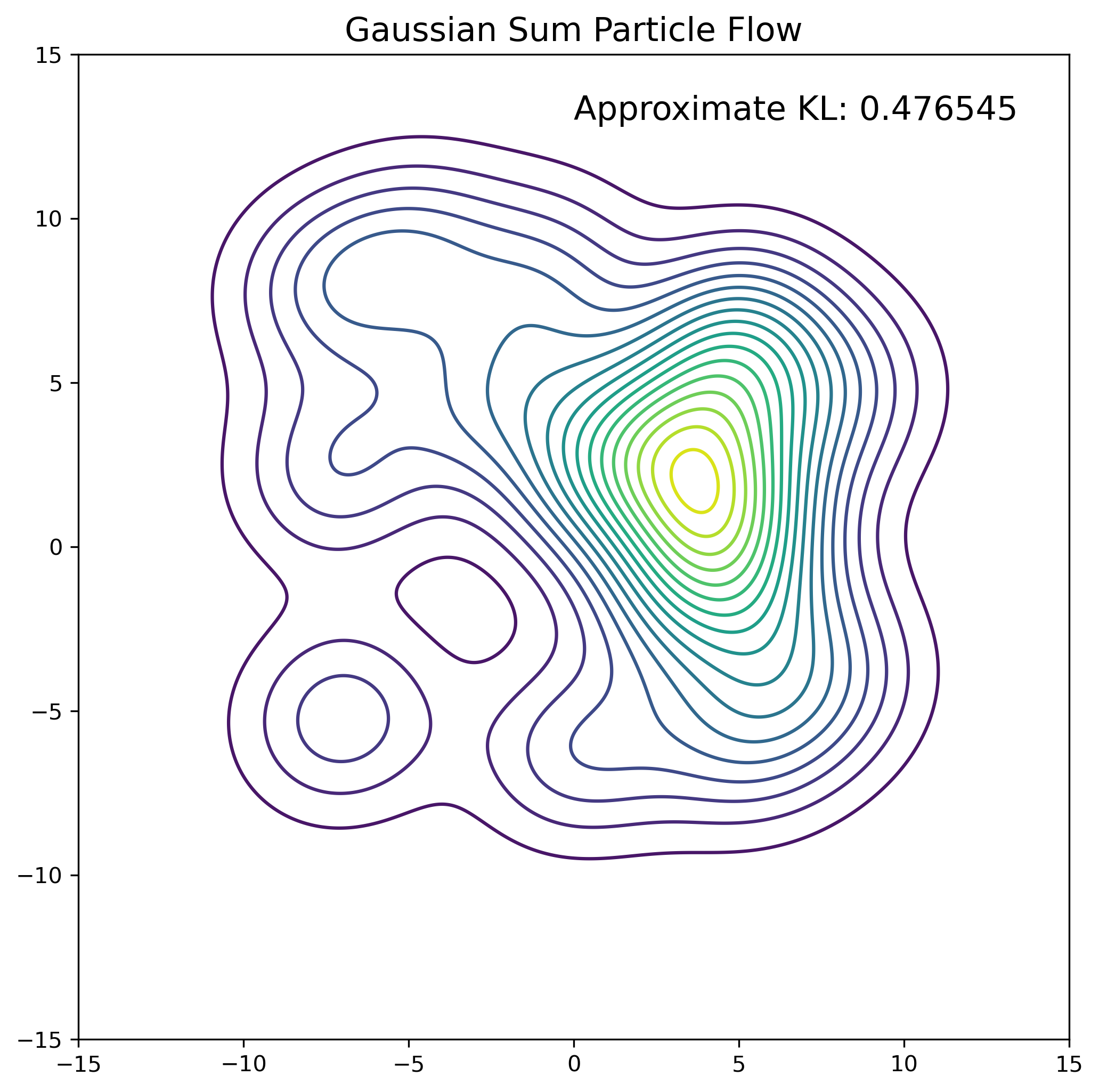}%
    \includegraphics[width=0.32\linewidth,valign=t]{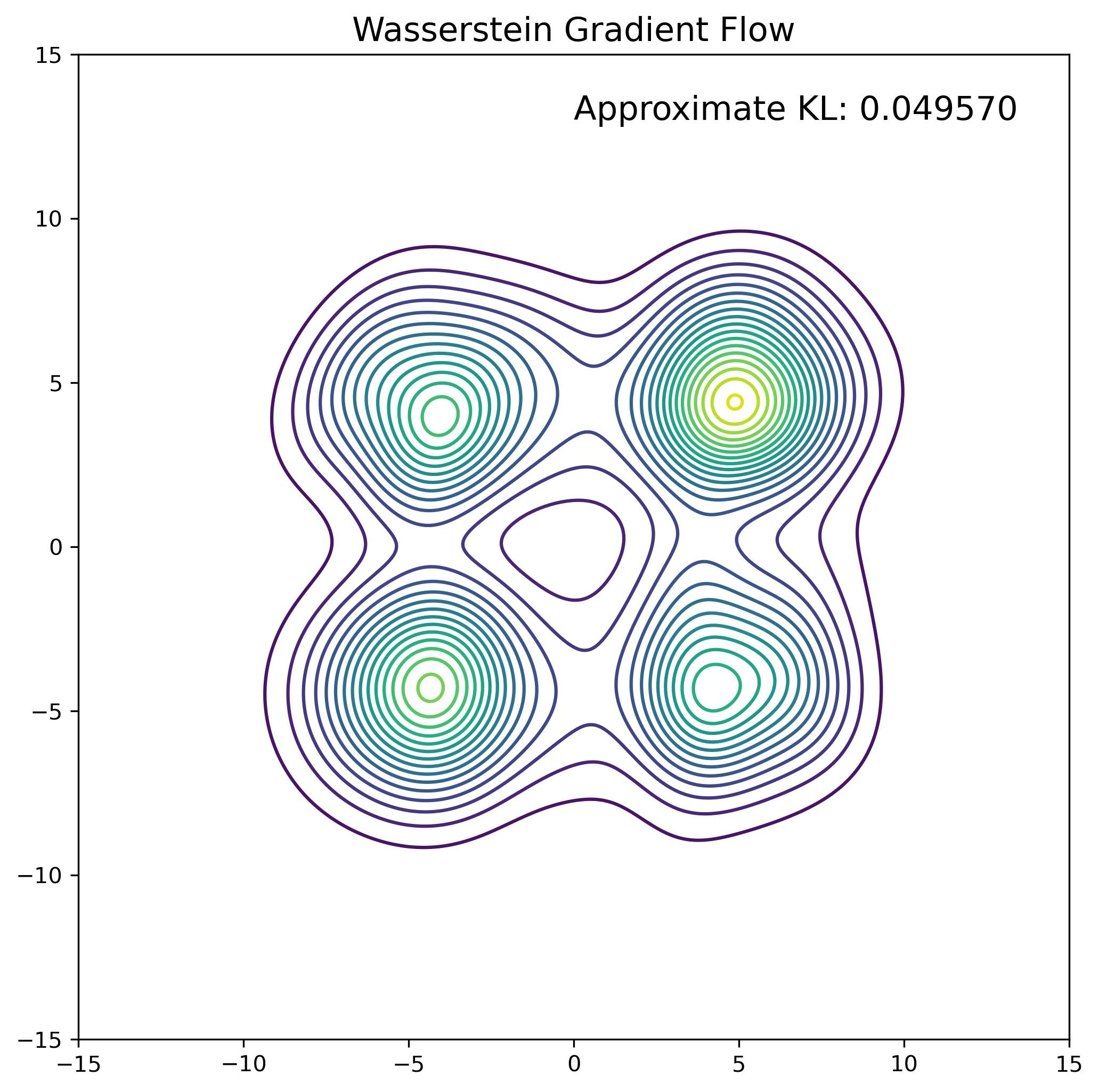}%
    \hfill%
    \includegraphics[width=0.32\linewidth,valign=t]{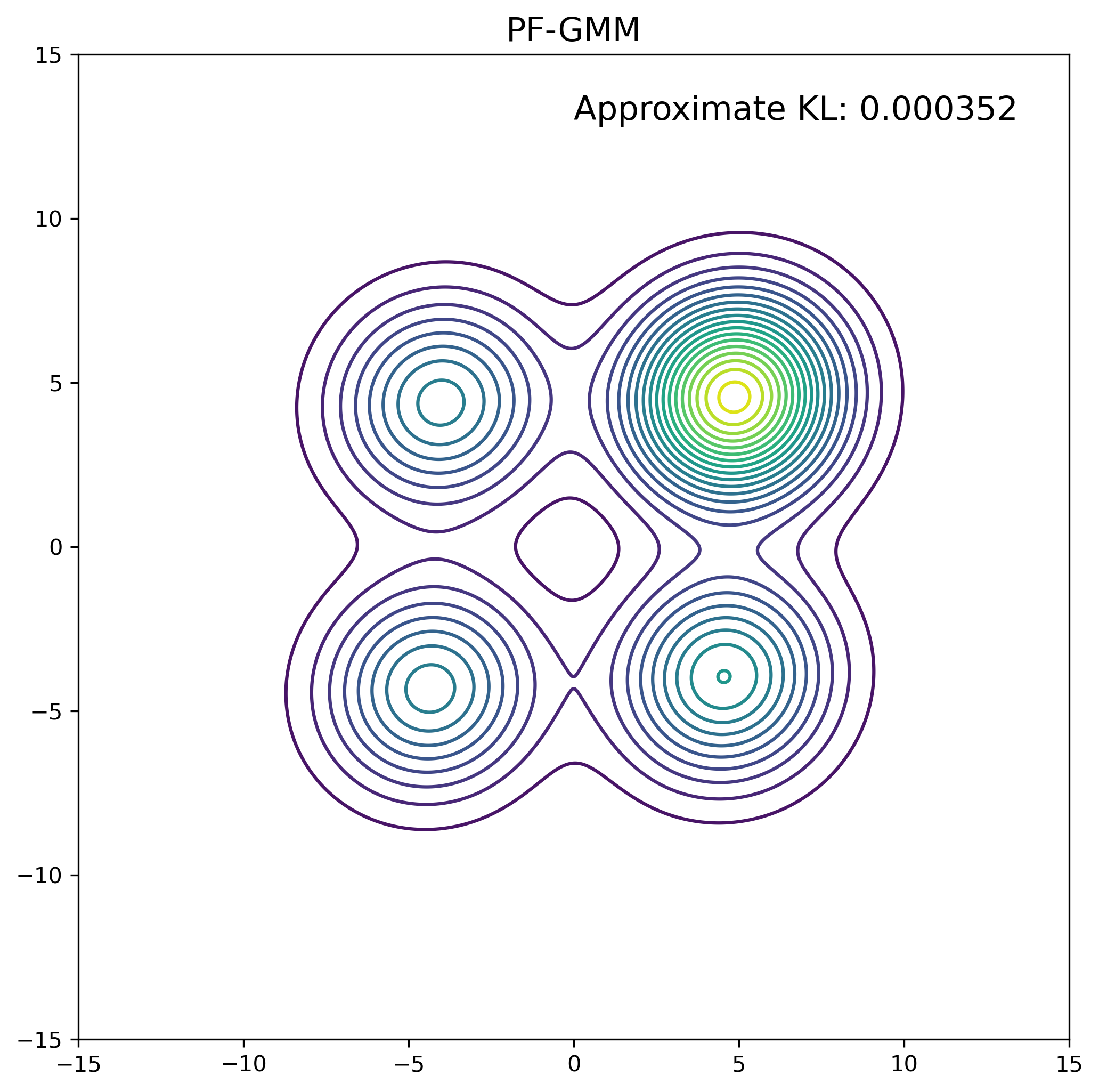}%
    \includegraphics[width=0.32\linewidth,valign=t]{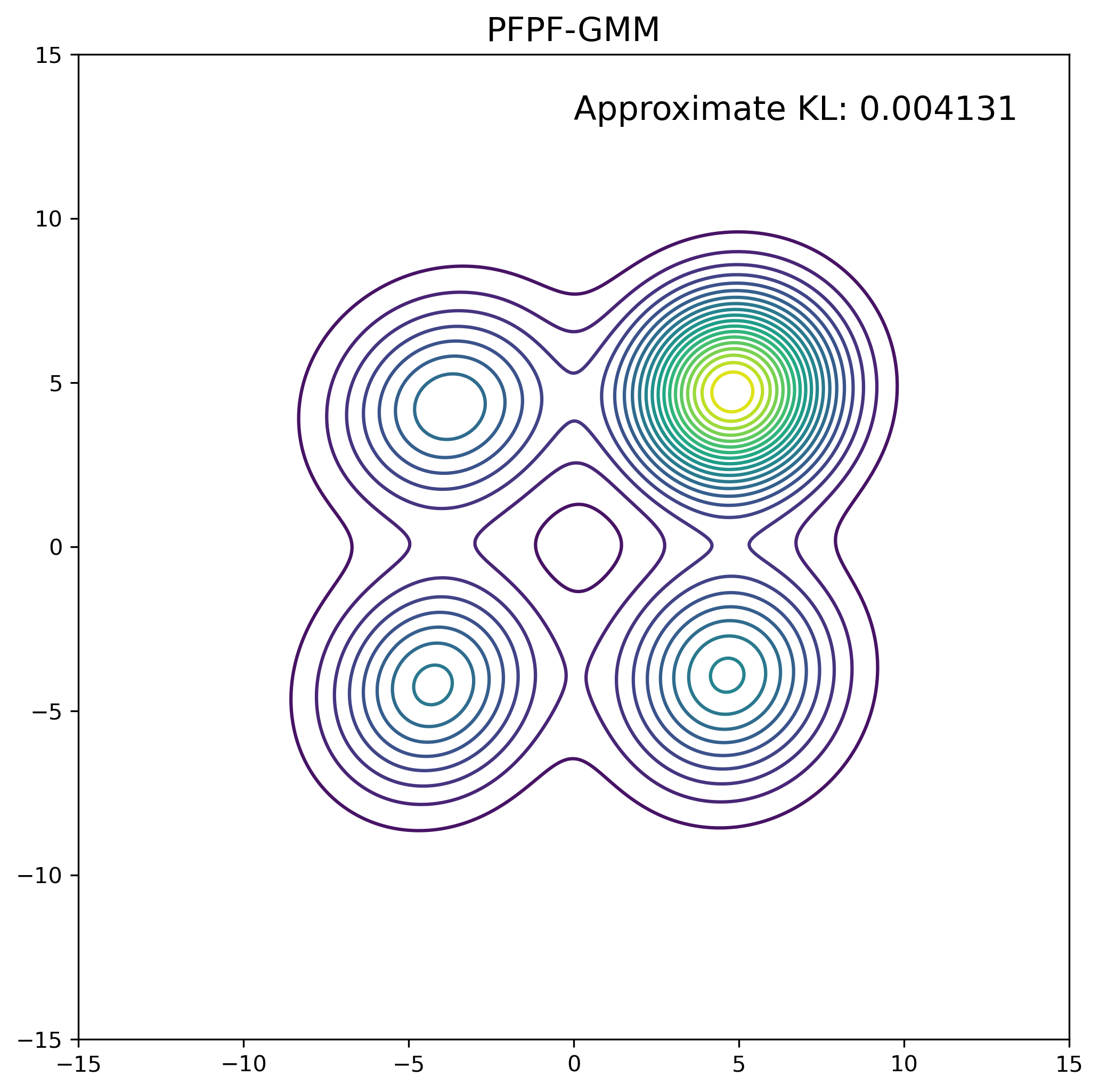}%
    \includegraphics[width=0.32\linewidth,valign=t]{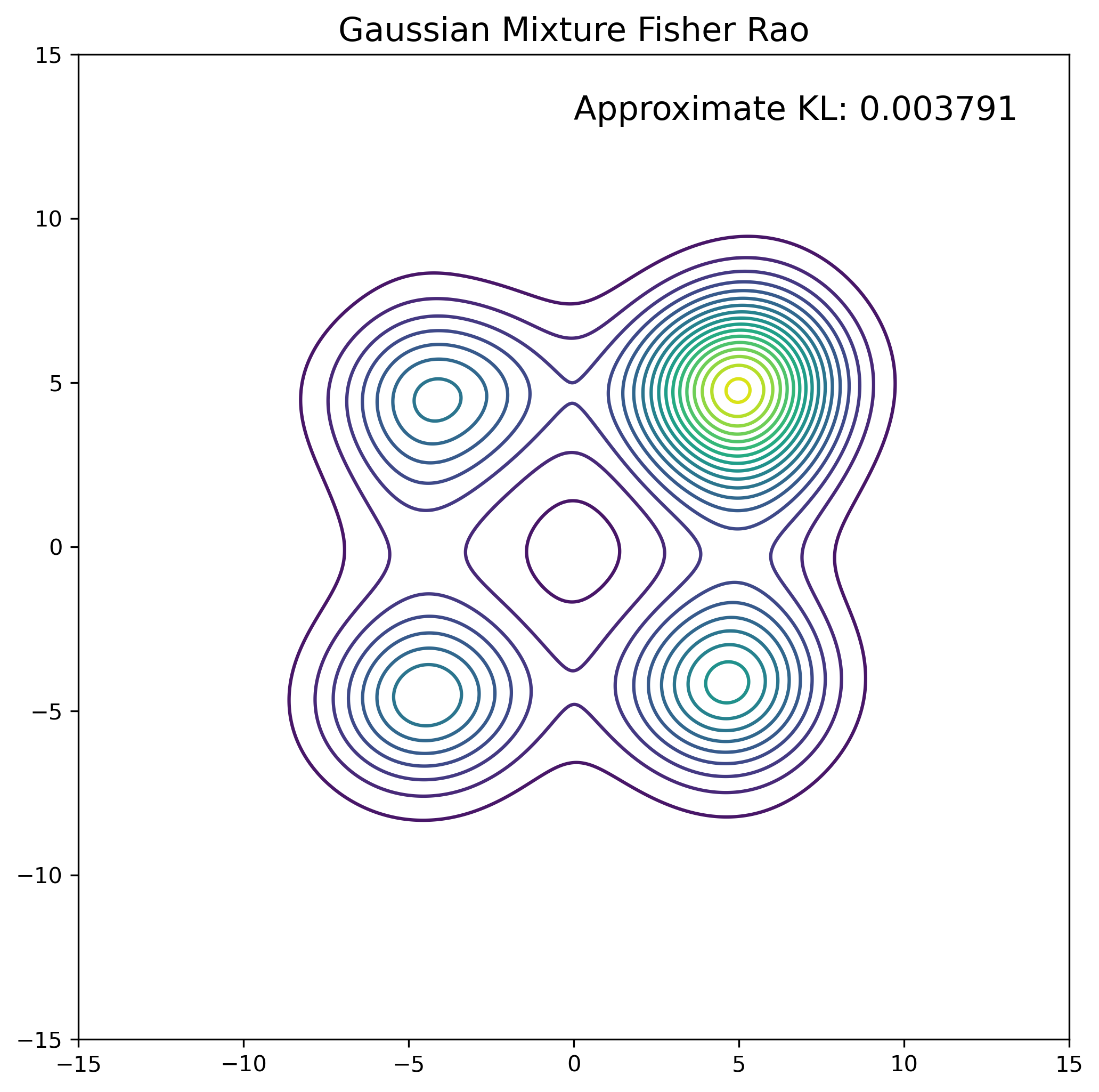}%
    \caption{Comparison of the approximated Gaussian mixture Fisher-Rao particle flow \eqref{gaussian_mixture_approx: natural_para_liouville} with the Gaussian sum particle flow \citep{zhang2024multisensor}, the Wasserstein gradient flow \citep{lambert2022wass}, the PF-GMM \citep{pal2017gaussian}, and the PFPF-GMM \citep{pal2018particle} for the Gaussian mixture prior case. The approximated posterior contour is shown for each method. The top left figure shows the reference contour. The top middle figure shows the contour generated by Gaussian sum particle flow, which does not capture the four components of the reference contour and achieves the highest KL divergence. The top right figure shows the contour generated by Wasserstein gradient flow, which captures the locations of the four components of the reference contour but fails to capture the component weights. The bottom left figure shows the contour generated by PF-GMM, and the bottom middle figure shows the contour generated by PFPF-GMM. Both methods capture the locations and the weights of the four components of the reference contour. The PF-GMM method achieves the lowest KL divergence approximation. The bottom right figure shows the contour generated by the approximated Gaussian mixture Fisher-Rao particle flow, which also captures both the locations and the weights of the four components of the reference contour, and achieves a comparable KL divergence approximation compared to the PF-GMM method.}
    \label{fig: mp}
\end{figure}

\begin{figure}[t]
    \centering
    \includegraphics[width=0.98\linewidth,valign=t]{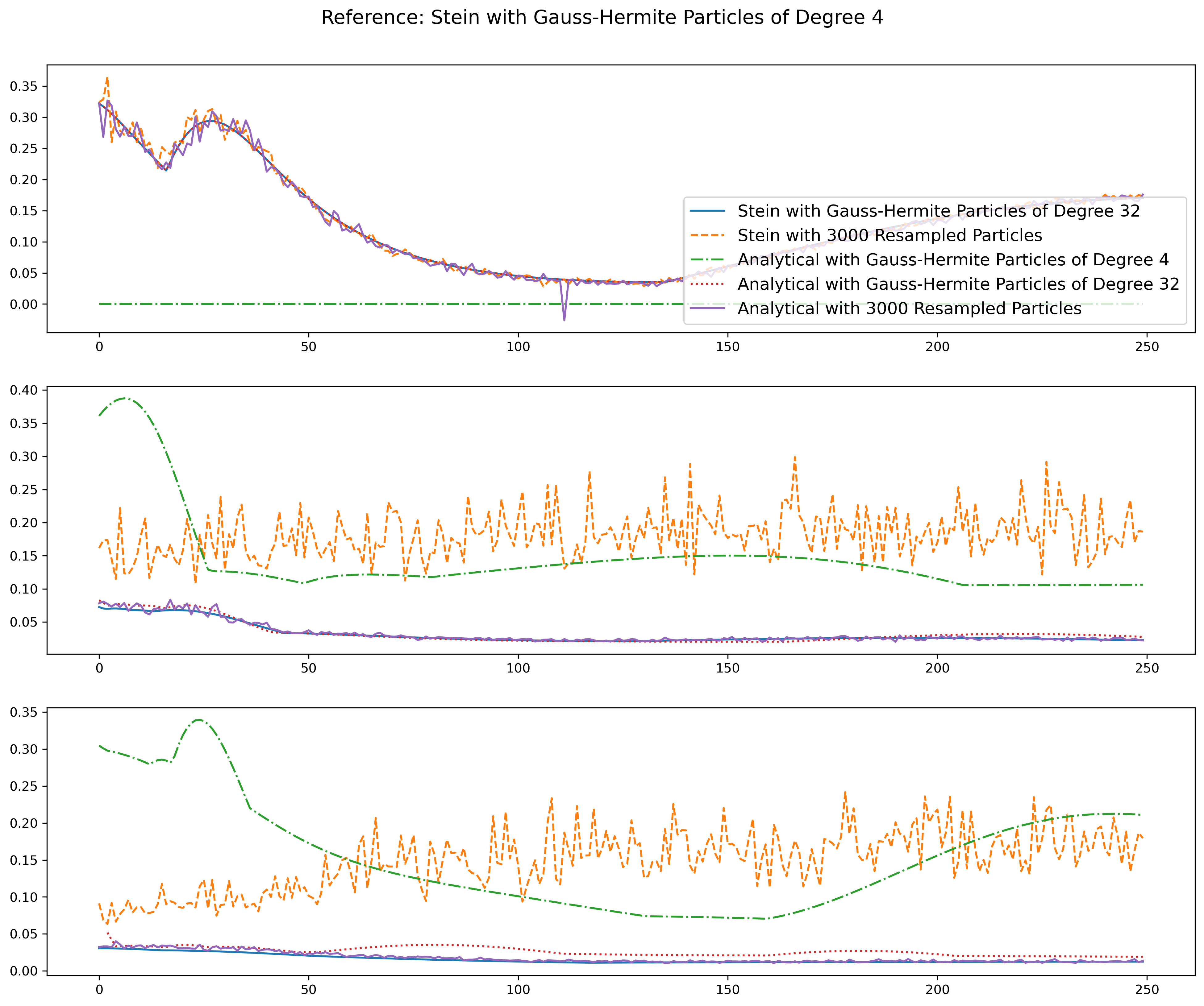}
    \caption{Comparison of expectation evaluations of the $V(\bfx)$ function \eqref{fisher_rao: kl_kernel} for the Gaussian mixture prior case. In each figure, the results obtained by Stein's method \eqref{dif: derivative_free} with Gauss-Hermite particles of degree $32$ and resampled particles are represented by solid blue and dashed orange lines. The results obtained by analytical calculations with Gauss-Hermite particles of degree $4$, Gauss-Hermite particles of degree $32$, and resampled particles are represented by dashed green, dotted red, and solid purple lines, respectively. The top figure shows the difference in the expected $V(\bfx)$ function, the middle plot shows the difference in the expected gradient of $V(\bfx)$ function $\bbE\left[ \nabla_{\bfx} V(\bfx) \right]$, and the bottom plot depicts the difference in the expected Hessian of $V(\bfx)$ function $\bbE\left[ \nabla_{\bfx}^2 V(\bfx) \right]$. The maximum difference between the compared and reference methods across all Gaussian mixture components is reported. The reference method uses Stein's gradient and Hessian with Gauss-Hermite particles of degree $4$.}
    \label{fig: gmm_gradient}
\end{figure}

\subsection{Nonlinear Observation Likelihood}
\label{sec: eval_nonlinear}

In this case, we consider a single Gaussian as the prior density $p(\bfx) = p_{\calN}(\bfx; \hat{\bfx}, P)$, with 
\begin{equation}
    \hat{\bfx} = \begin{bmatrix} 1 \\ 1 \end{bmatrix}, \quad P = \begin{bmatrix} 5.5 & -1.5 \\ -1.5 & 5.5 \end{bmatrix}. 
\end{equation}
The likelihood function is also a Gaussian density $\calN(\bfz; H(\bfx), R)$, with
\begin{equation}
   H(\bfx) = \| \bfx \|, \quad R = 2, \quad \bfx^* = \begin{bmatrix} 4.7 \\ -3.1 \end{bmatrix}, 
\end{equation}
where $\bfx^*$ denotes the true value of $\bfx$ used to generate the observation. 

For our approximated Gaussian mixture Fisher-Rao flow \eqref{gaussian_mixture_approx: natural_para_liouville}, the initial mean parameter for the $k$th Gaussian mixture component $\bfmu^{(k)}_{t=0}$ is sampled from the prior $\bfmu^{(k)}_{t=0} \sim \calN(\hat{\bfx}, P)$. The initial variance parameter of the $k$th Gaussian mixture component is set to $\Sigma^{(k)}_{t=0} = 3P$. The initial weight for the $k$th Gaussian mixture component is set to $\omega^{(k)}_{t=0} = 1 / K$, where $K$ denotes the total number of Gaussian mixture components employed in our approach. The Wasserstein gradient flow method \citep{lambert2022wass} uses the same initialization approach as the approximated Gaussian mixture Fisher-Rao particle flow. For the Gaussian sum particle flow method \citep{zhang2024multisensor}, each Gaussian particle flow component is initialized with $1000$ randomly sampled particles from $\calN(\bfmu^{(k)}, P)$, where $\bfmu^{(k)}$ is sampled from the prior density $\bfmu^{(k)} \sim \calN(\hat{\bfx}, P)$. This initialization method is employed to ensure consistency with the other two methods, as the Gaussian sum particle flow utilizes particles to compute the empirical mean during propagation. The PF-GMM method \citep{pal2017gaussian} and the PFPF-GMM method \citep{pal2018particle} use the same initialization approach as the Gaussian sum particle flow method. This initialization is necessitated by the fact that both PF-GMM and PFPF-GMM are formulated to handle Gaussian mixture prior densities and/or Gaussian mixture likelihood densities. In the nonlinear observation model case, both the prior and the likelihood are single Gaussian densities. As a result, following the original formulations will restrict the posterior approximation to a single Gaussian density, which is inadequate for representing the nonlinear posterior and consequently lead to poor performance. We use the same method described in Section~\ref{sec: eval_gmm} to obtain a parametric representation of the posterior approximation generated by PF-GMM and PFPF-GMM.

Figure~\ref{fig: single_gaussian_nonlinear} shows the results for each method when a single Gaussian is used to approximate the posterior density. In this case, the Gaussian particle flow, the PFPF-GMM, and the Gaussian Fisher-Rao particle flow converge to the area with high posterior probability. However, the Wasserstein gradient flow and the PF-GMM fail to converge to the region with high posterior probability, resulting in a high KL divergence approximation. Figure~\ref{fig: nonlinear} shows the results for each method when a Gaussian mixture with $20$ components is used to approximate the posterior density. In this case, the Gaussian sum particle flow only captures regions with high posterior probability. The Wasserstein gradient flow only captures the shape of the posterior, but it fails to recover the region with high posterior probability, resulting in the highest KL divergence approximation. On the other hand, the PF-GMM, the PFPF-GMM, and the approximated Gaussian mixture particle flow all capture the shape of the posterior. Among these methods, our approximated Gaussian mixture Fisher-Rao particle flow generates the most accurate approximation of the posterior, yielding the lowest KL divergence approximation. The results obtained by PF-GMM and PFPF-GMM are more spread out than the true posterior, resulting in higher KL divergence approximations relative to the approximated Gaussian mixture particle flow.  Figure~\ref{fig: nonlinear_gradient} shows a comparison of the approximation accuracy of the expectation terms. In this test case, we observe that when using Stein's gradient and Hessian \eqref{dif: derivative_free}, Gauss-Hermite particles \eqref{dif: reparameterization_trick} of degree $4$ are sufficient to ensure stable propagation. However, when using the analytical gradient and Hessian, a large particle ensemble is required to achieve accuracy similar to the results obtained using Stein's gradient and Hessian.

%
%
\begin{figure}[t]
    \centering
    \includegraphics[width=0.33\linewidth]{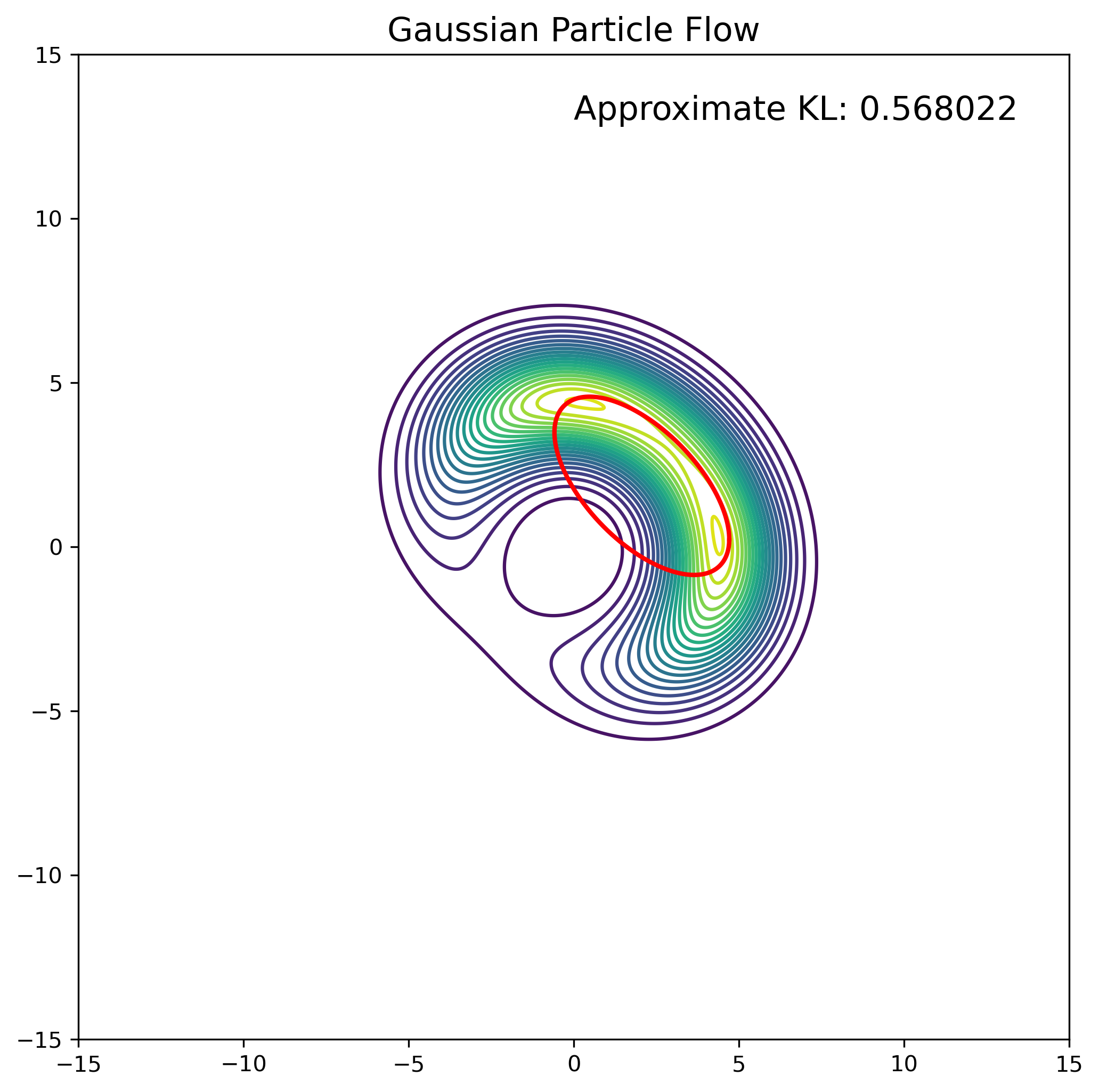}%
    \includegraphics[width=0.33\linewidth]{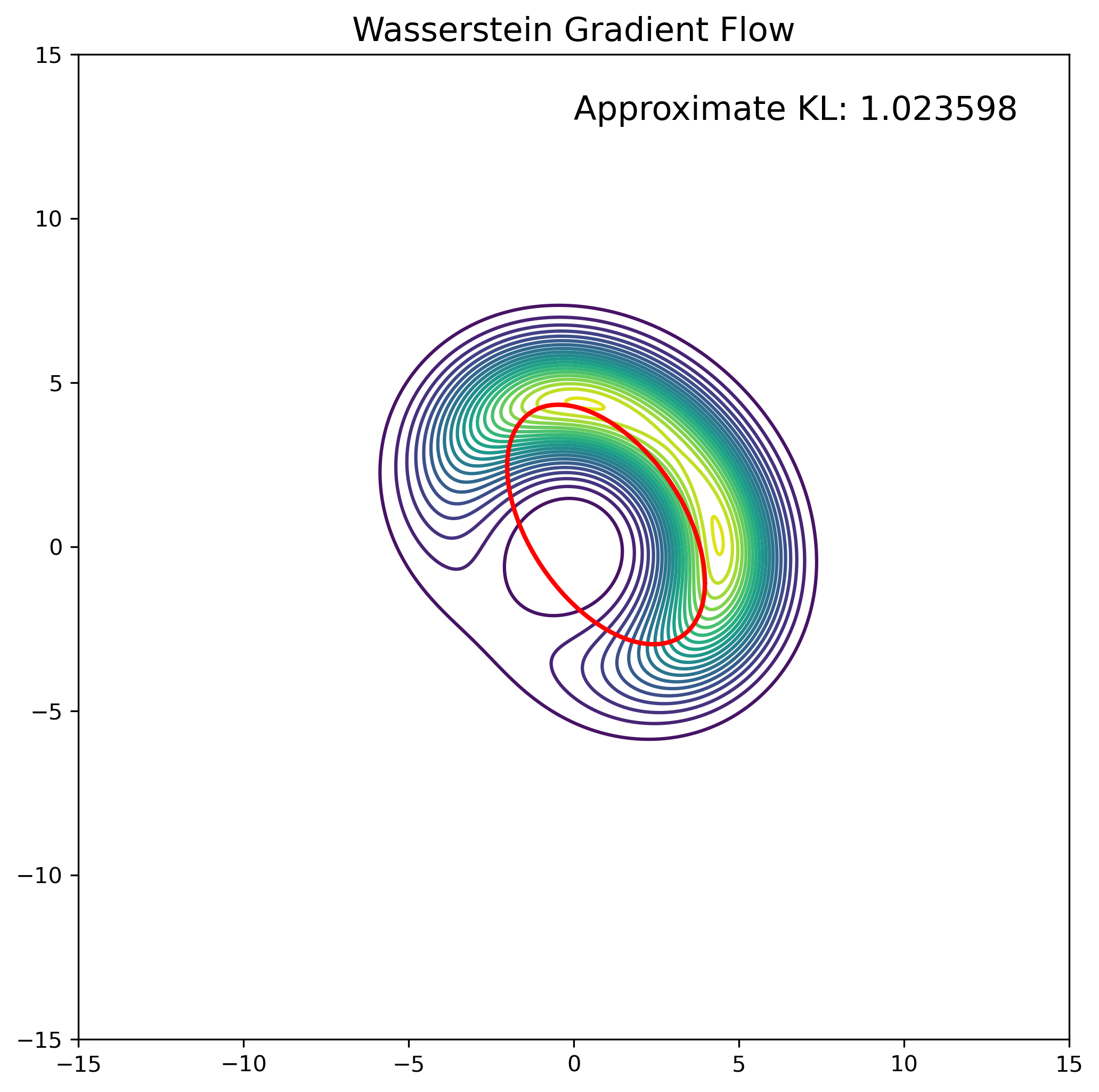}%
    \includegraphics[width=0.33\linewidth]{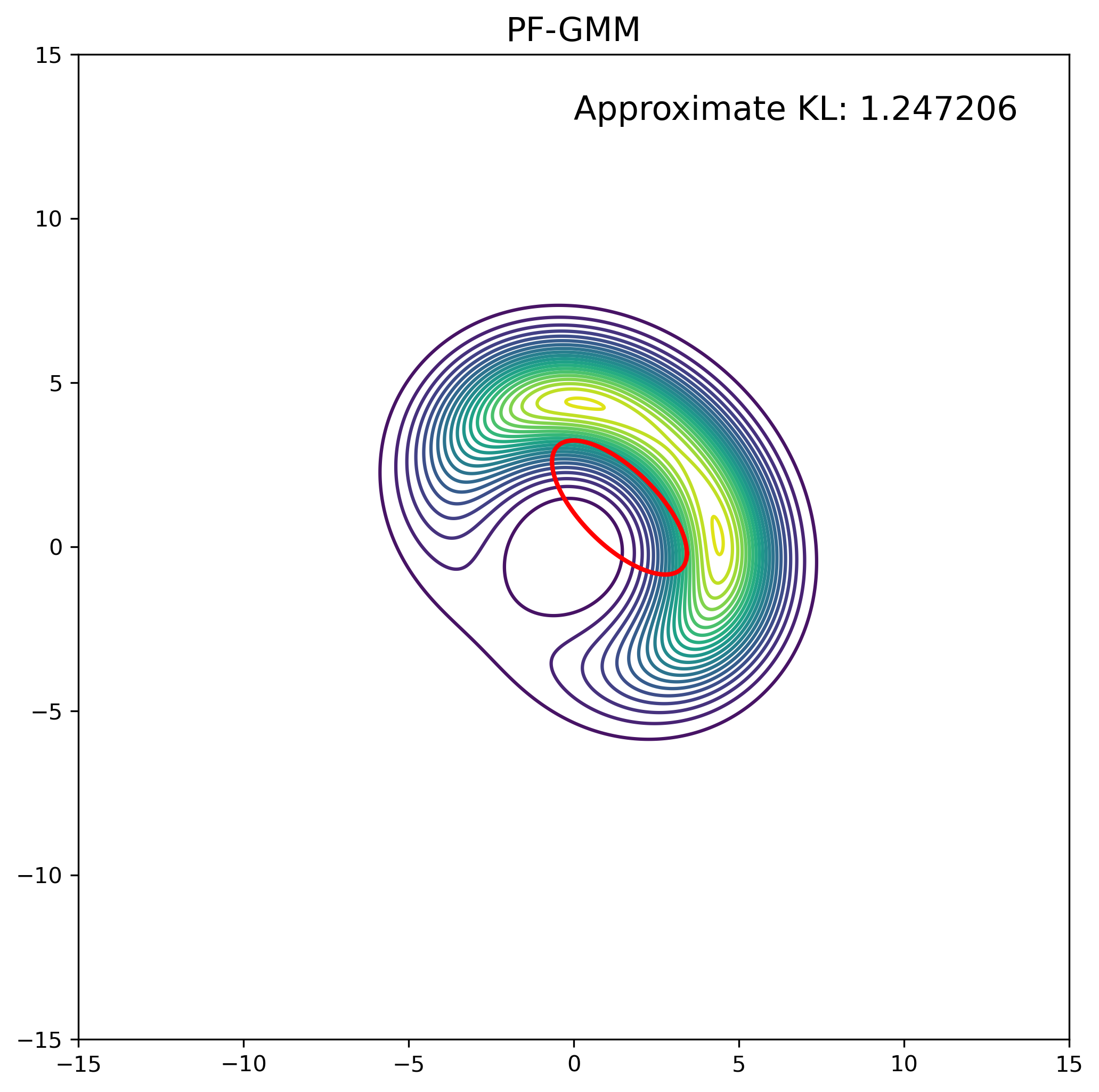}%
    \hfill%
    \includegraphics[width=0.33\linewidth]{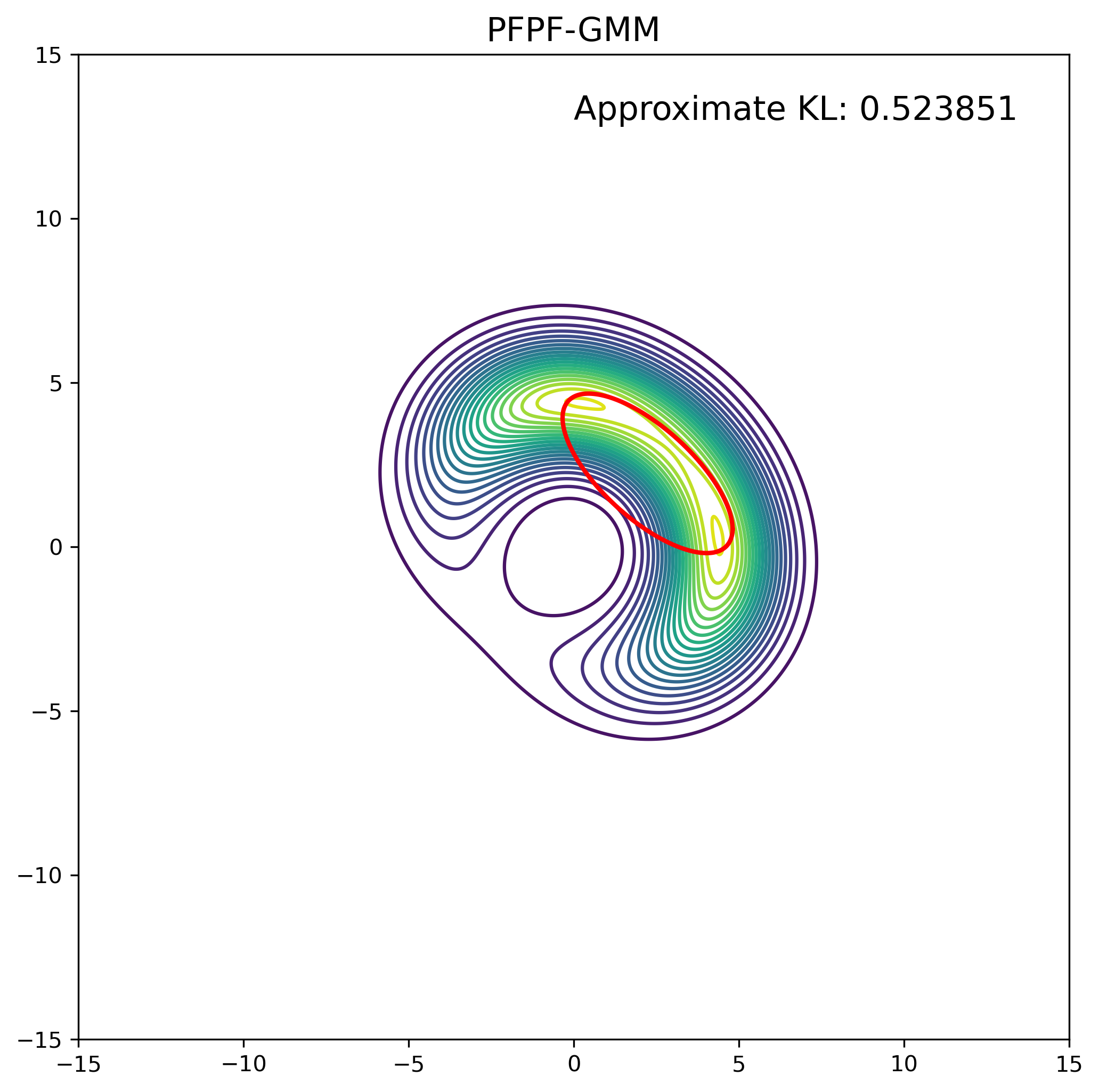}
    \includegraphics[width=0.33\linewidth]{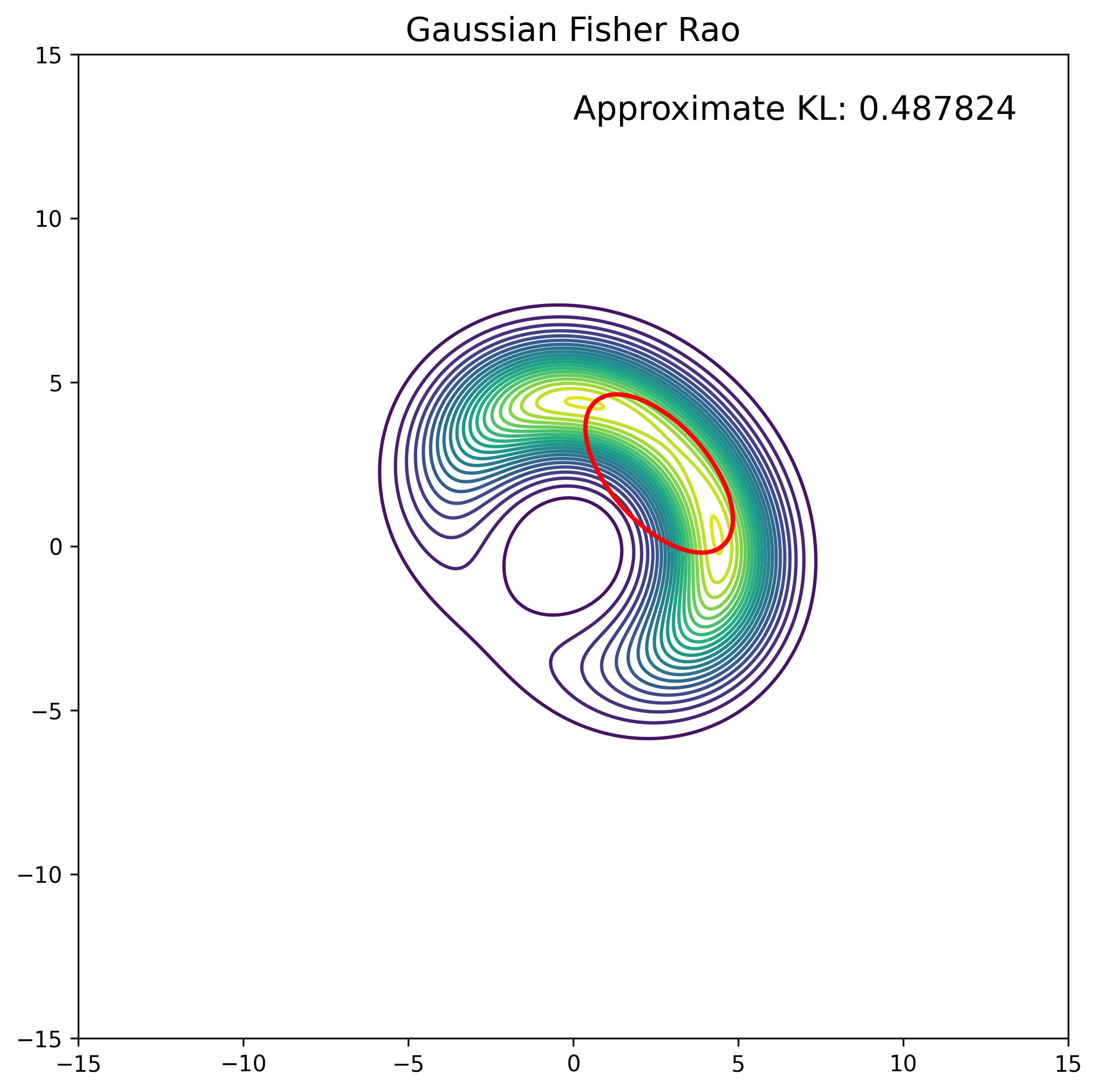}
    \caption{Comparison of the Gaussian Fisher-Rao particle flow \eqref{gaussian_approx: natural_para_liouville} with the Gaussian particle flow \citep{zhang2024multisensor}, the Wasserstein gradient flow \citep{lambert2022wass}, the PF-GMM \citep{pal2017gaussian}, and the PFPF-GMM \citep{pal2018particle} for the nonlinear observation model case. For each method, we use a single Gaussian to approximate the posterior density. The covariance contour corresponding to one Mahalanobis distance is overlaid on the reference contour. The top left figure shows the result obtained by the Gaussian particle flow method, which produces an approximation that is slightly misaligned with the region of highest posterior probability, leading to a higher approximated KL divergence. The top middle figure shows the result obtained by the Wasserstein gradient flow method, which fails to accurately capture the region of the highest posterior probability, resulting in the highest approximated KL divergence. The top right figure shows the result obtained by the PF-GMM method, which fails to capture any region with high posterior probability, leading to the highest KL divergence approximation. The results obtained by the PFPF-GMM and the Gaussian Fisher-Rao particle flow are shown in the bottom right and bottom left figures, respectively. Both methods provide a relatively accurate Gaussian approximation of the posterior density, resulting in comparably low KL divergence approximation.}
    \label{fig: single_gaussian_nonlinear}
\end{figure}

\begin{figure}[tbh]
\centering
\includegraphics[width=0.32\linewidth,valign=t]{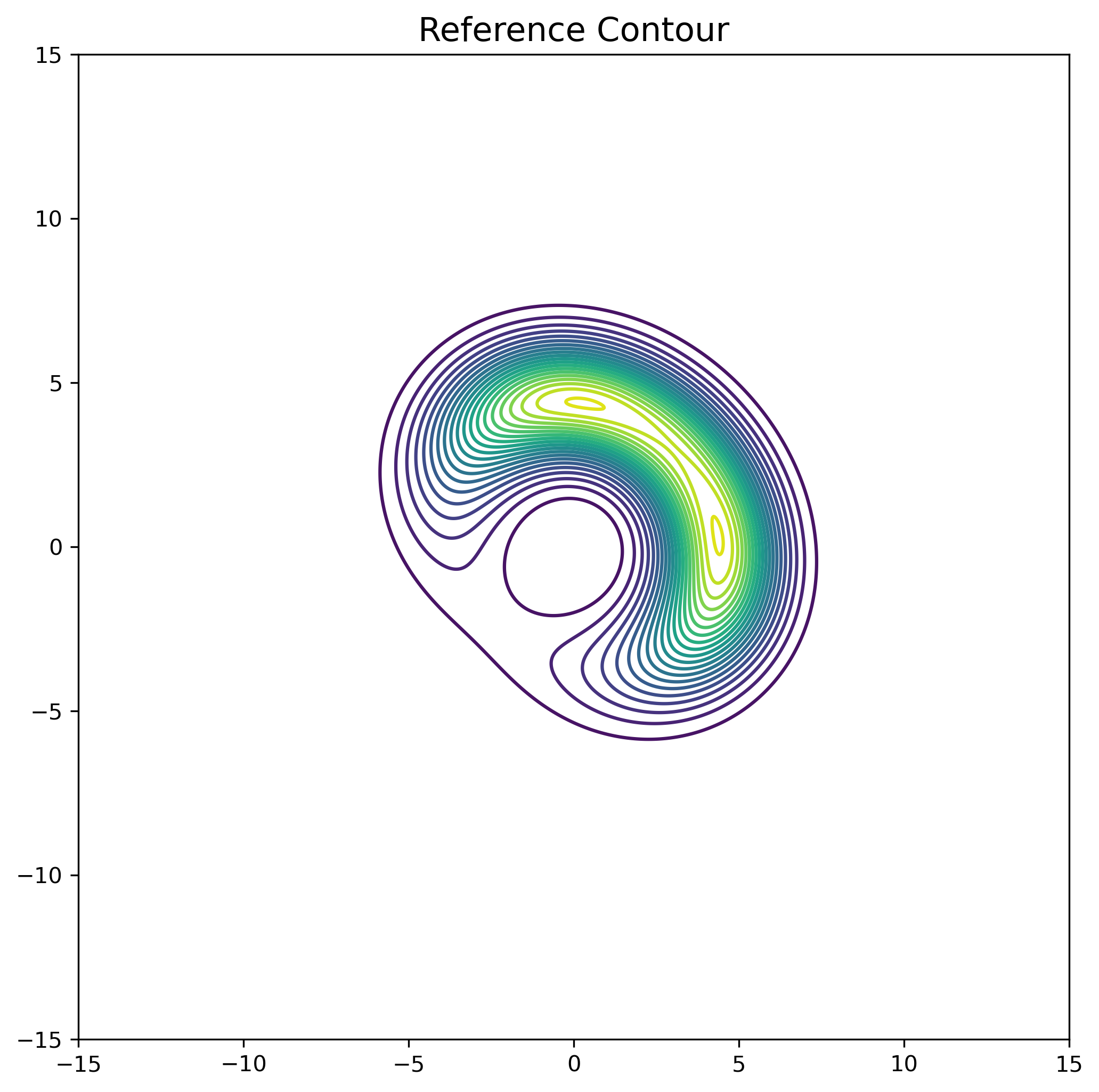}%
\includegraphics[width=0.32\linewidth,valign=t]{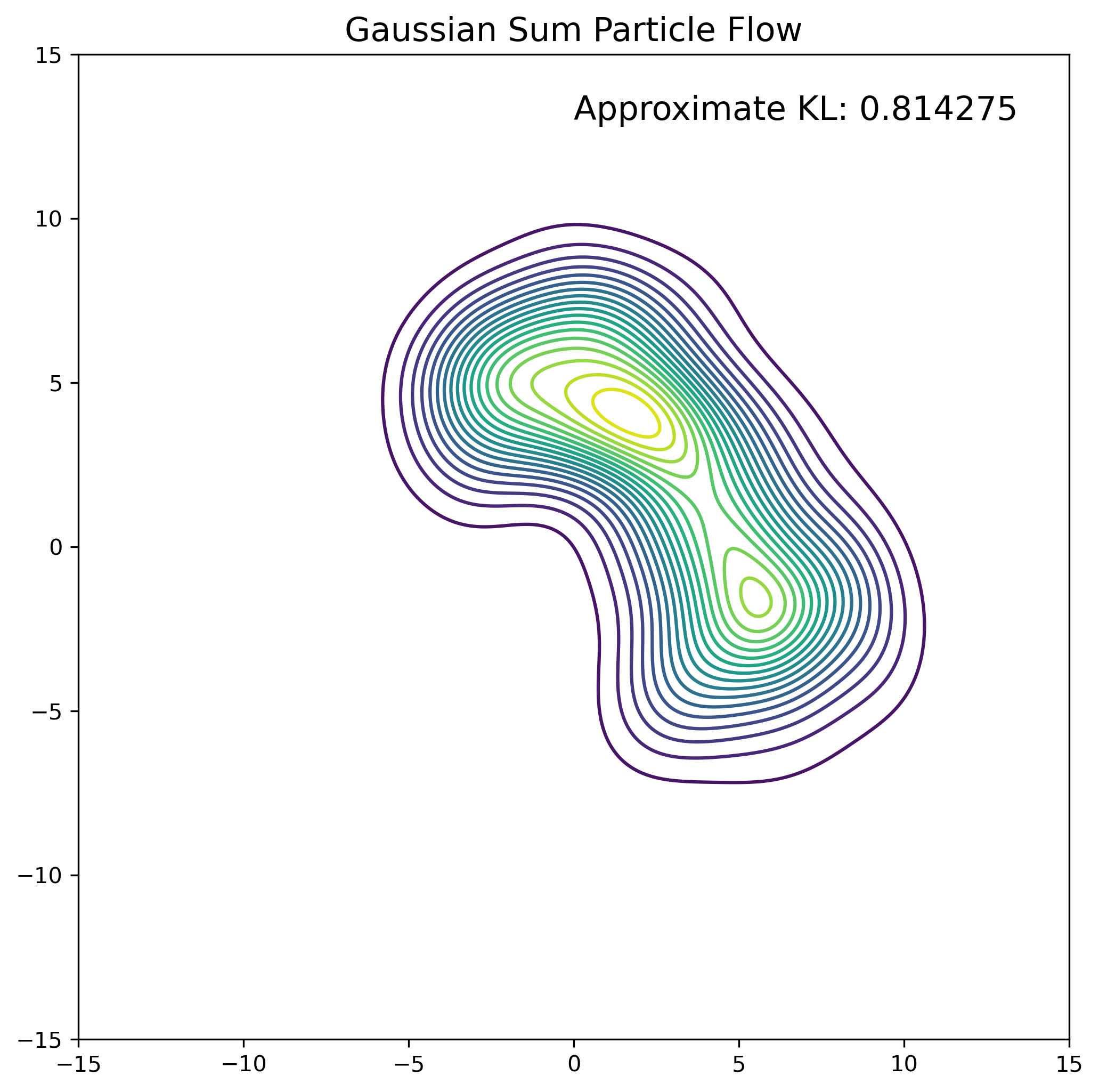}%
\includegraphics[width=0.32\linewidth,valign=t]{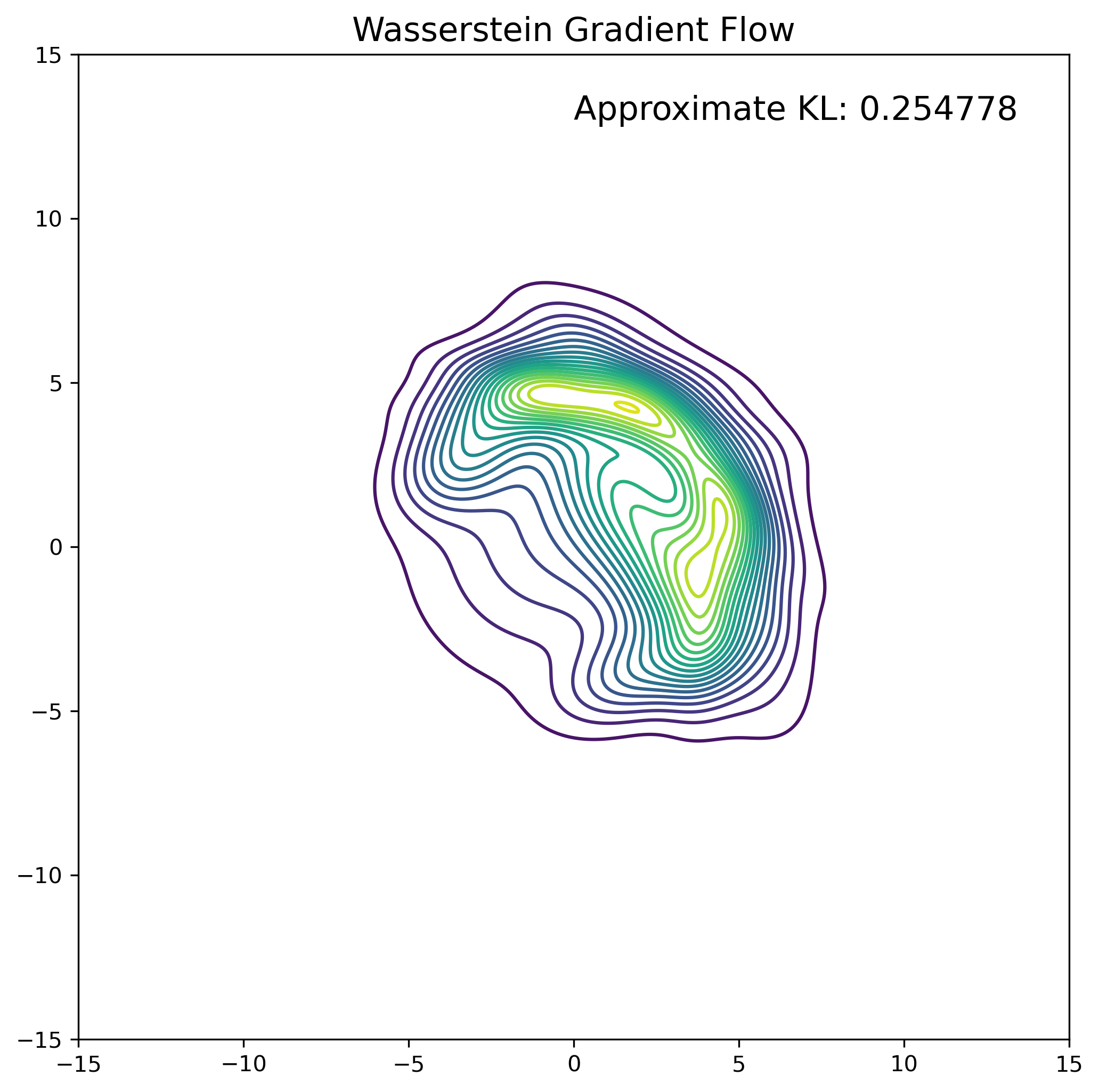}%
\hfill%
\includegraphics[width=0.32\linewidth,valign=t]{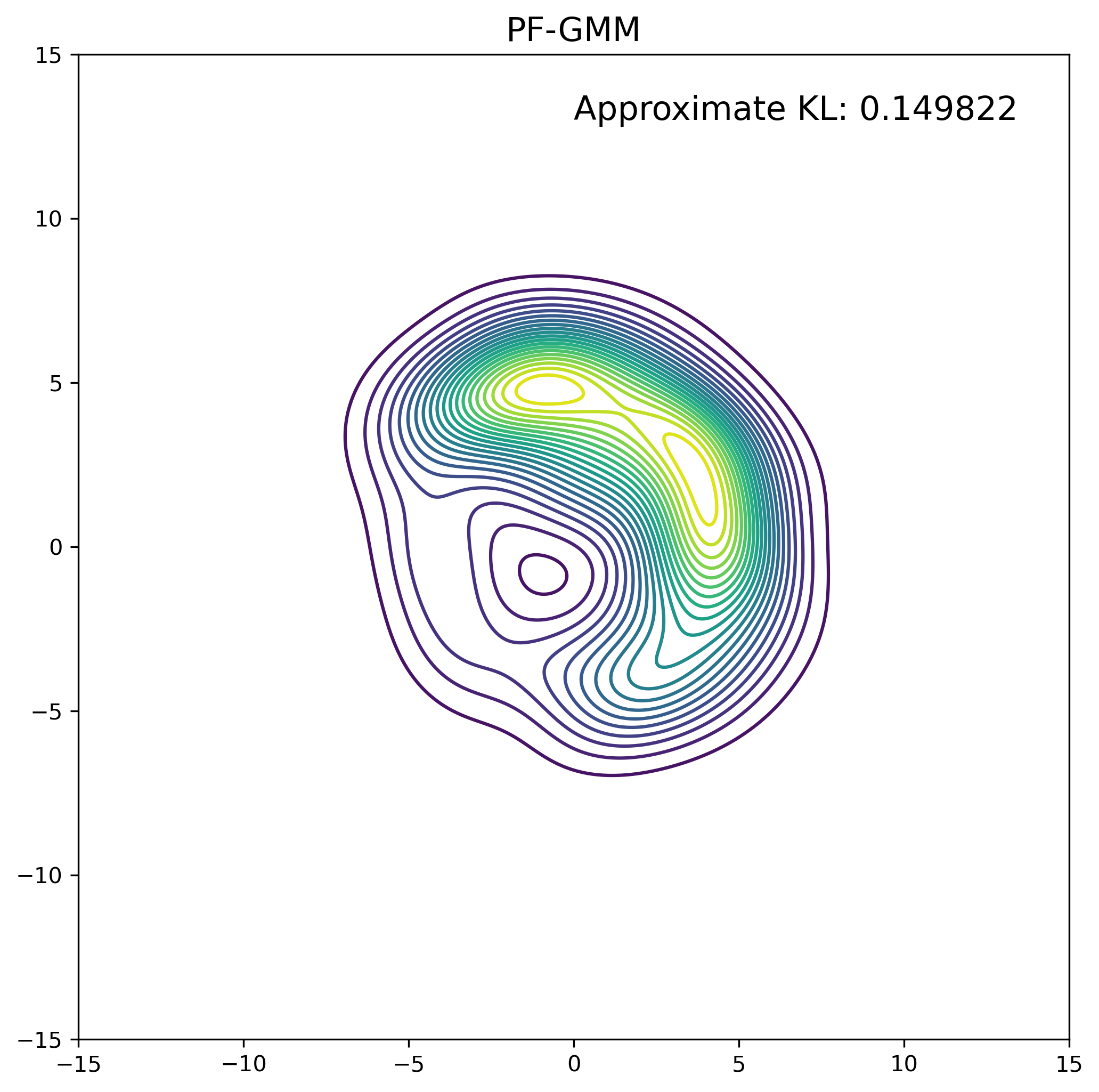}%
\includegraphics[width=0.32\linewidth,valign=t]{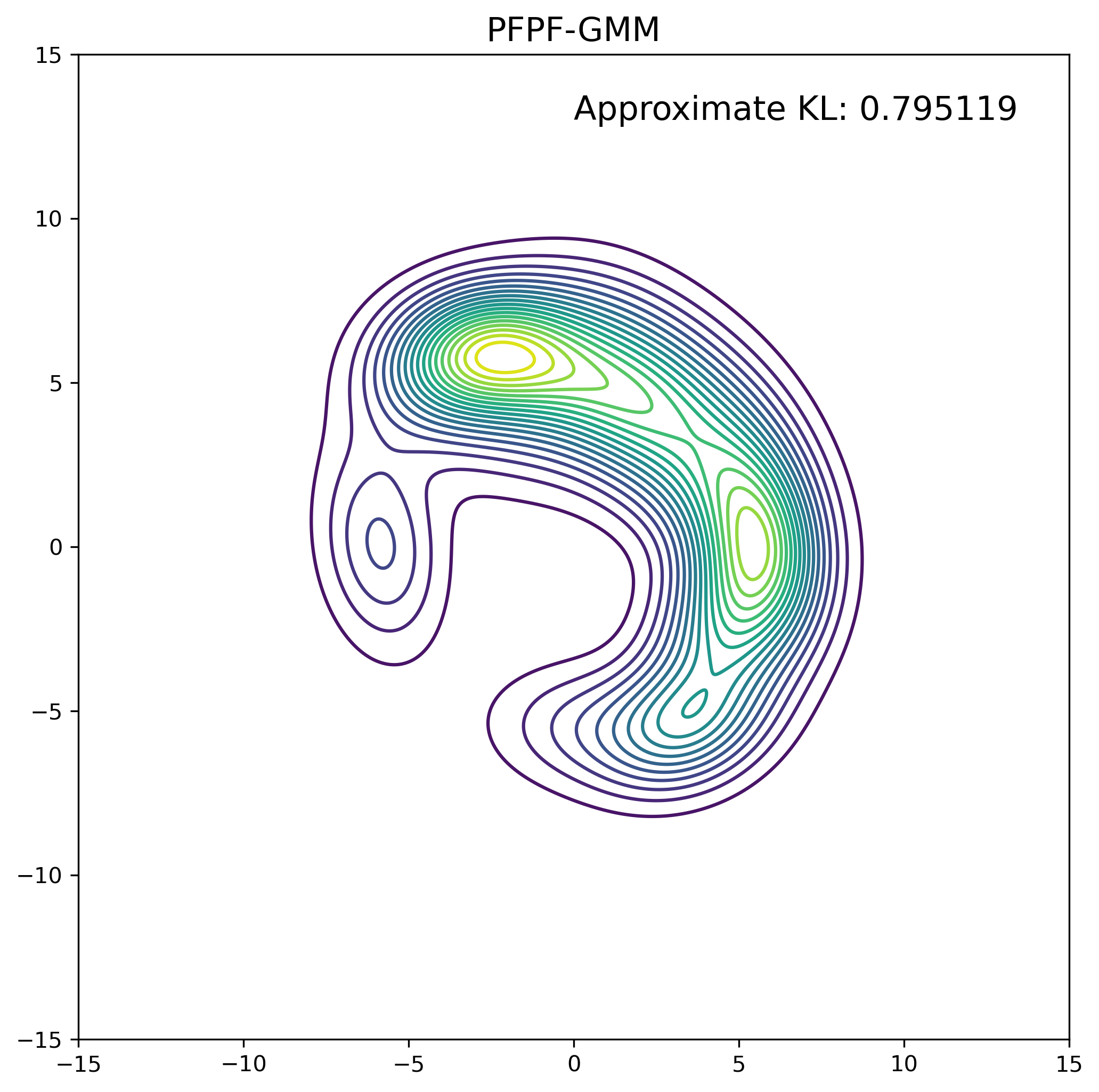}%
\includegraphics[width=0.32\linewidth,valign=t]{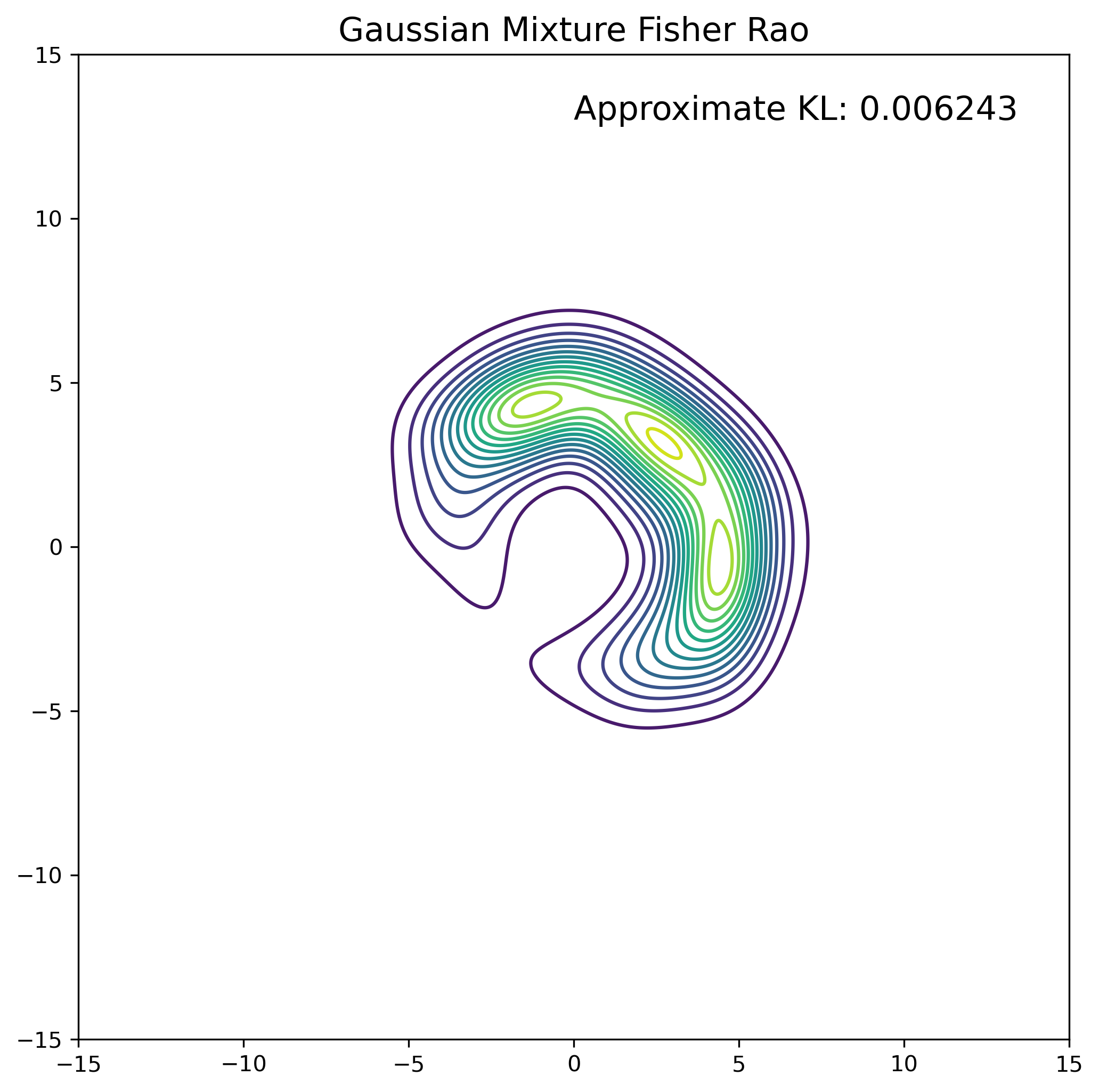}%
\caption{Comparison of the approximated Gaussian mixture Fisher-Rao particle flow \eqref{gaussian_mixture_approx: natural_para_liouville} with the Gaussian sum particle flow \citep{zhang2024multisensor}, the Wasserstein gradient flow \citep{lambert2022wass}, the PF-GMM \citep{pal2017gaussian}, and the PFPF-GMM \citep{pal2018particle} for the nonlinear observation model case. The approximated posterior contour is shown for each method. The top left figure shows the reference contour. The top middle figure shows the contour generated by Gaussian sum particle flow, which fails to capture the shape of the posterior, yielding the highest KL divergence. The top right figure shows the contour generated by Wasserstein gradient flow, which fails to give an accurate approximation of the posterior contour. The bottom left figure shows the contour generated by PF-GMM, which captures the shape of the poster but fails to align the region of high posterior probability. The bottom middle figure shows the contour generated by PFPF-GMM, which captures the overall shape of the posterior but is more spread out. The bottom right figure shows the contour generated by the approximated Gaussian mixture Fisher-Rao particle flow, which captures the reference banana-shaped posterior and yields the lowest KL divergence.}
\label{fig: nonlinear}
\end{figure}

\begin{figure}[tbh]
    \centering
    \includegraphics[width=0.98\linewidth,valign=t]{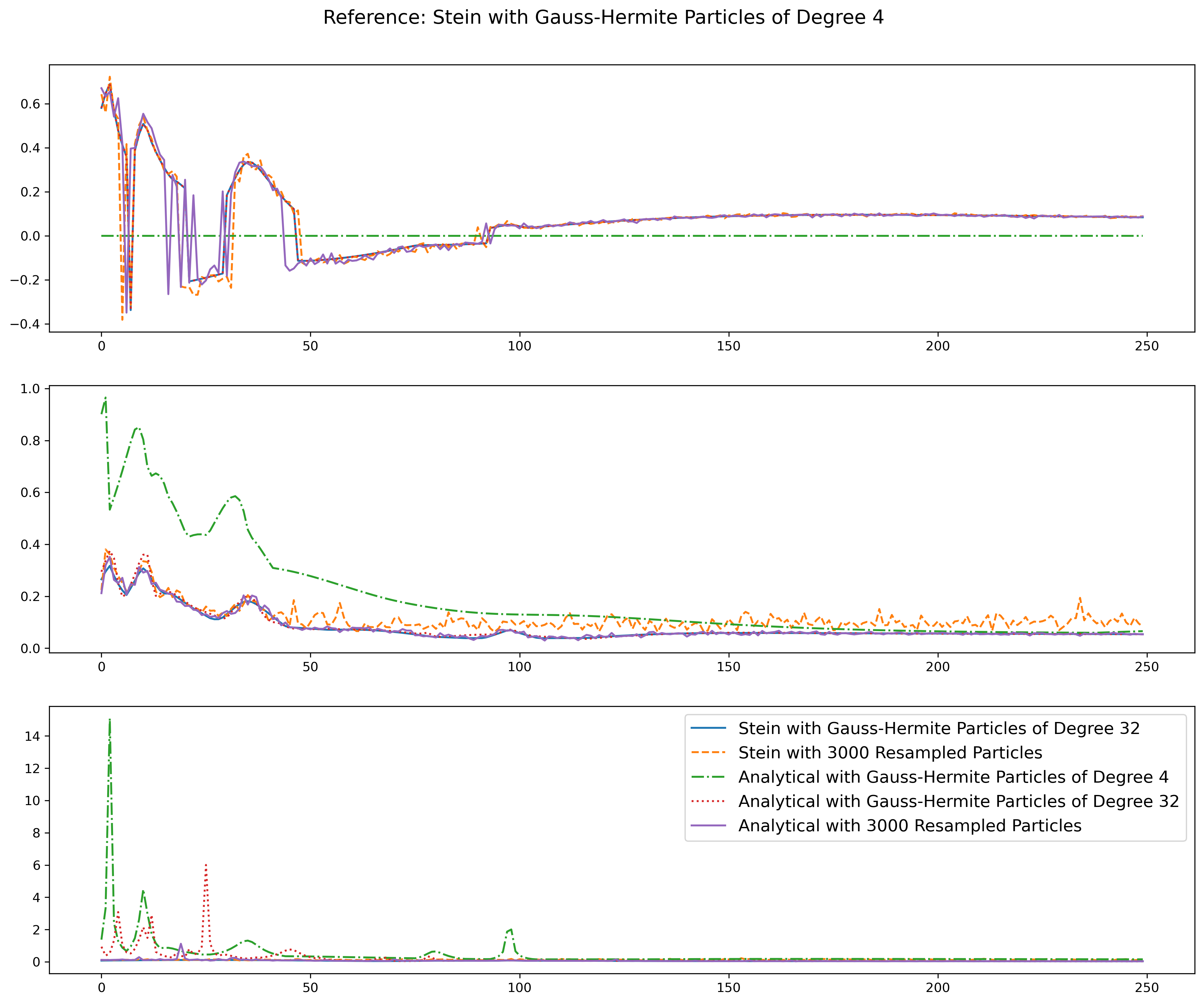}
    \caption{Comparison of expectation evaluations of the $V(\bfx)$ function \eqref{fisher_rao: kl_kernel} for the nonlinear observation model case. In each figure, the results obtained by Stein's method \eqref{dif: derivative_free} with Gauss-Hermite particles of degree $32$ and resampled particles are represented by solid blue and dashed orange lines. The results obtained by analytical calculations with Gauss-Hermite particles of degree $4$, Gauss-Hermite particles of degree $32$, and resampled particles are represented by dashed green, dotted red, and solid purple lines, respectively. The top plot displays the difference in the expected $V(\bfx)$ function, the middle plot shows the difference in the expected gradient of $V(\bfx)$ function $\bbE\left[ \nabla_{\bfx} V(\bfx) \right]$, and the bottom plot shows the difference in the expected Hessian of $V(\bfx)$ function $\bbE\left[ \nabla_{\bfx}^2 V(\bfx) \right]$. The maximum difference between the compared methods and the reference method across all Gaussian mixture components is reported. The reference method uses Stein's gradient and Hessian with Gauss-Hermite particles of degree $4$.}
    \label{fig: nonlinear_gradient}
\end{figure}

%
\subsection{Bayesian Logistic Regression}
In this case, we use an unnormalized posterior generated in the context of Bayesian logistic regression on a two-class dataset $\calZ = \{ (\bfy_i, z_i) \}_{i=1}^N$ provided by \cite{lambert2022wass}. The probability of the binary label $z_i \in \{0, 1\}$ given $\bfy_i \in \bbR^n$ and $\bfx \in \bbR^n$ is defined by the following Bernoulli density:
\begin{equation}
    p(z_i | \bfx; \bfy_i) = \sigma(\bfy_i^{\top} \bfx)^{z_i} (1 - \sigma(\bfy_i^{\top} \bfx))^{1 - z_i},
\end{equation}
where $\sigma(x) = 1 / (1 + \exp(-x))$ is the logistic sigmoid function. The posterior associated with data $\calZ$ starting from an uninformative prior on $\bfx$ is:
\begin{equation}
\label{eval_hd: un_posterior}
    p(\bfx | \calZ) = \frac{1}{c}\prod_{i = 1}^N p(z_i | \bfx; \bfy_i) = \frac{1}{c}p(\bfx, \calZ), 
\end{equation}
where $c$ is the normalization constant. 
In this test, we assume access only to the unnormalized posterior by fixing the normalizing constant $c=1$.

For our Gaussian Fisher-Rao particle flow \eqref{gaussian_approx: natural_para_liouville}, the initial mean parameter $\bfmu_{t=0}$ is sampled from a Gaussian distribution $\bfmu_{t=0} \sim \calN(\mathbf{0}, 5I)$, and the initial variance parameter $\Sigma_{t=0}$ is set to $\Sigma_{t=0} = 5I$. The Wasserstein gradient flow method \citep{lambert2022wass} uses the same initialization approach as the Gaussian Fisher-Rao particle flow. For our approximated Gaussian mixture Fisher-Rao particle flow \eqref{gaussian_mixture_approx: natural_para_liouville}, the initial mean parameter for the $k$th Gaussian mixture component $\bfmu^{(k)}_{t=0}$ is sampled from a Gaussian distribution $\bfmu^{(k)}_{t=0} \sim \calN(\mathbf{0}, 5I)$, the initial variance parameter for the $k$th Gaussian mixture component is set to $\Sigma^{(k)}_{t=0} = 5I$ and the initial weight parameter for the $k$th Gaussian mixture component is set to $\omega^{(k)}_{t=0} = 1 / K$, where $K$ denotes the total number of Gaussian mixture components employed in our approach. Since the KL divergence approximation is inaccurate as $n$ increases significantly, we report the approximated evidence lower bound (ELBO) for each method:
\begin{equation}
    \text{ELBO}(p(\bfx, \calZ) \| q(\bfx)) \approx \frac{1}{N}\sum_{i=1}^{N} \log\left( \frac{p(\bfx_i, \calZ)}{q(\bfx_i)} \right), 
\end{equation}
with $p(\bfx, \calZ)$ given in \eqref{eval_hd: un_posterior}, $\bfx_i$ are particles and $N$ denotes the total number of particles.

When simulating the Fisher-Rao particle flow, it is more efficient to only propagate the mean and the covariance parameters and recover the particles using Theorem~\ref{theorem: gh_preserve} when necessary. Figure~\ref{fig: logistic_reg} shows that using the recovered particles achieves identical trajectories compared to the propagated particles. Moreover, for all test cases, both the Gaussian Fisher-Rao particle flow and the approximated Gaussian mixture particle flow achieve similar ELBO after $100$ iterations, indicating that it is sufficient to use a single Gaussian to approximate the posterior density. For all test cases, both the Gaussian Fisher-Rao particle flow and the approximate Gaussian mixture Fisher-Rao particle flow outperform the Wasserstein gradient flow and achieve a better convergence rate.

%
\begin{figure}[tbh]
\centering
\includegraphics[width=0.98\linewidth,valign=t]{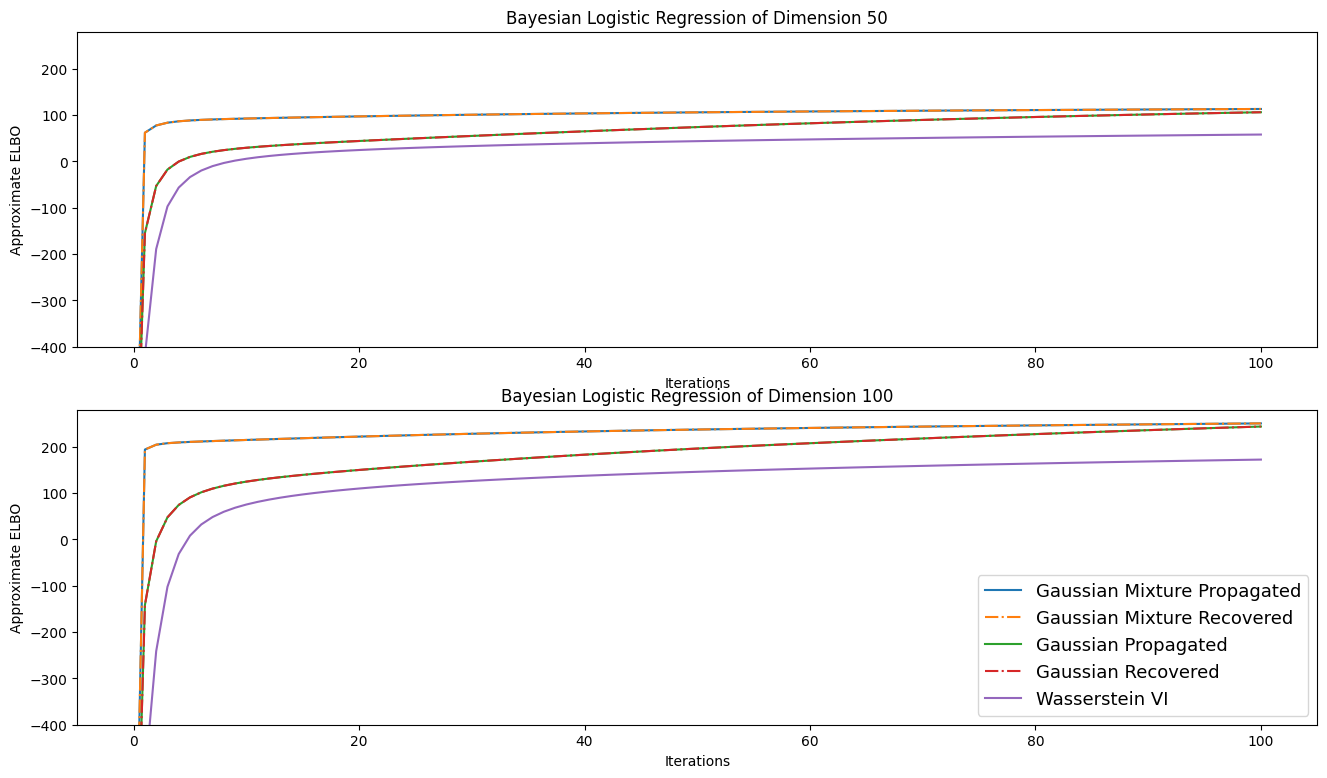}
\caption{Comparison of the Fisher-Rao particle flows \eqref{gaussian_approx: natural_para_liouville}, \eqref{gaussian_mixture_approx: natural_para_liouville} with the Wasserstein gradient flow \citep{lambert2022wass} for the Bayesian logistic regression task with different dimensions. The results obtained by Gaussian Fisher-Rao particle flow \eqref{gaussian_approx: natural_para_liouville} with propagated and recovered particles are represented by solid blue and dashed orange lines. The results obtained by approximated Gaussian mixture Fisher-Rao particle flow \eqref{gaussian_mixture_approx: natural_para_liouville} with propagated and recovered particles are represented by solid green and dashed red lines. Solid purple lines represent the result obtained by the Wasserstein gradient flow method. The approximated ELBO is reported for each method. It can be shown that using the recovered particles from Theorem~\ref{theorem: gh_preserve} yields performance identical to that of the propagated particles. Additionally, we observe that approximating the posterior with a single Gaussian is sufficient for this task, as it achieves nearly identical results to a Gaussian mixture with $5$ components. Our Fisher-Rao particle flows outperform the Wasserstein gradient flow.}
\label{fig: logistic_reg}
\end{figure}

\section{Extension to Non-Gaussian Densities}
\label{sec: Application}

In this section, we show how the Fisher-Rao particle flows described in Algorithms~\ref{alg: gaussian_fr} and~\ref{alg: gaussian_mixture_fr} can be combined with the normalizing flow method \citep{rezende2015variational} to deal with non-Gaussian posterior densities. We begin with a brief overview of normalizing flow, followed by a description of the proposed approach, in which the Fisher-Rao particle flow is used to optimize the base density of the normalizing flow. Once the particle dynamics function associated with the base density has been determined, we then show how to obtain the corresponding particle dynamics function for the variational density parameterized by the normalizing flow. 
Finally, we present numerical experiments that illustrate the effectiveness of the proposed method.

\subsection{Review of Normalizing Flow}
Let $\bfx \in \bbR^n$ be a random variable obtained by applying a transformation $F : \bbR^n \to \bbR^n$ to another random variable $\bfu \in \bbR^n$:
\begin{equation}
    \bfx = F(\bfu), \quad \bfu \sim p_{u}(\bfu), 
\end{equation}
where the density $p_{u}(\bfu)$ is referred to as \emph{base density}. We make the following assumption on the transformation.
\begin{assumption}[Regularity Assumptions for Flow Transformations]
\label{assumption: transformation_regularity}
    The transformation $F: \bbR^n \to \bbR^n$ is invertible, continuously differentiable almost everywhere, and satisfies $\det( \nabla_{\bfu}F(\bfu) ) \neq 0$ for almost every $\bfu \in \bbR^n$.
\end{assumption}

This ensures that the change of variables can be applied to obtain the density of $\bfx$ as:
\begin{equation}\label{eq: nf_density}
    p_x(\bfx) = p_{u}(\bfu) | \det(\nabla_{\bfu}F(\bfu)) |^{-1}.
\end{equation}
In practice, the transformation $F$ is typically specified in parametric form. Common parameterizations include the planar transformation \citep{blessing2024beyond}, given by
\begin{equation}\label{eq: planar}
    F(\bfu) = \bfu + \bfy h(\bfw^{\top} \bfu + b), 
\end{equation}
where $\bfy,\bfw \in \bbR^n$ are parameter vectors, $b \in \bbR$ is a scalar bias term, and $h(\cdot)$ is the hyperbolic tangent function. Another widely used parameterization is the radial transformation \citep{blessing2024beyond}, defined as:
\begin{equation}\label{eq: radial}
    F(\bfu) = \bfu + \frac{\beta}{\alpha + \| \bfu - \bfu_0 \|} (\bfu - \bfu_0), 
\end{equation}
where $\bfu_0 \in \bbR^n$ is a parameter vector, $\alpha \in \bbR^+$ is a positive scalar bias parameter, and $\beta \in \bbR$ is a scalar scaling parameter. 

A single transformation may not be sufficiently expressive to approximate the posterior density in the context of variational inference. Instead, one can specify the transformation as a composition of multiple intermediate transformations as follows:
\begin{equation}
    F(\bfu) = F_{J} \circ F_{J-1} \circ \cdots \circ F_{1}(\bfu), 
\end{equation}
where $F_{j}$ denotes the $j$-th intermediate transformation. Each intermediate transformation may take one of the parametric forms discussed above. In the context of variational inference, normalizing flow provides an alternative means of specifying the variational density. Next, we introduce a particle flow-based normalizing flow formulation, in which the variational density is obtained by applying the transformation in \eqref{eq: nf_density} to a base density $p_{u}(\bfu)$ represented by a set of particles.

\subsection{Particle Flow-Based Normalizing Flow}

To improve the expressiveness of our particle flow formulation beyond Gaussian and Gaussian mixture parameterizations, we represent the variational density $q(\bfx)$ as a composition of a base density $b(\bfu)$ and an invertible transformation $F(\bfu)$ satisfying Assumption~\ref{assumption: transformation_regularity}. The resulting variational density $q(\bfx)$ is given by the push-forward of $b(\bfu)$ through $F$:
\begin{equation}
    q(\bfx) = b(\bfu) | \det(\nabla_{\bfu}F(\bfu)) |^{-1}.
\end{equation}
We parameterize the base density and the transformation as $b(\bfu; \bftheta_u)$ and $F(\bfu; \bftheta_F)$, respectively. Letting $\bftheta = (\bftheta_u, \bftheta_F)$, the resulting normalizing flow-based variational density  admits also a parametric form, $q(\bfx; \bftheta)$. As discussed in Section~\ref{subsec: fisher_rao_flow}, to minimize the KL divergence between the variational density and the posterior density, the evolution of $\bftheta$ is governed by the Fisher-Rao parameter flow given in \eqref{fisher_rao: fr_para_flow}. Given this parameter evolution, the corresponding time derivative of the variational density induces a particle dynamics function $\bfphi(\bfx, t)$ through the Liouville equation:
\begin{equation}
    \frac{\partial q(\bfx; \bftheta_t)}{\partial \bftheta_t} \frac{\d \bftheta_t}{\d t} = - \nabla_{\bfx} \cdot \left( q(\bfx; \bftheta_t) \bfphi(\bfx, t) \right).
\end{equation}
Despite its generality, this approach faces two limitations. First, computing the Fisher–Rao parameter flow requires evaluating the Fisher information matrix \eqref{fisher_rao: fisher_rao_tensor}, which is challenging for normalizing flows due to the complexity of the transformation $F$. Second, as discussed earlier, solving the Liouville equation for the particle dynamics function $\bfphi(\bfx, t)$ is challenging for non-Gaussian variational families, which prevents the derivation of a tractable particle dynamics function for evolving the particles. Moreover, deriving the associated particle flow under the full parameterization typically requires transformation-specific analysis, which further limits the generality of the resulting particle dynamics across different normalizing flow parameterizations.

To address these challenges and gain a better intuition of our proposed method, we temporarily decoupled the optimization of the base density and the transformation. In particular, for a given transformation $F(\bfu; \bar{\bftheta}_F)$, with fixed parameters $\bar{\bftheta}_F$, we show that the Fisher–Rao particle flow can be employed to optimize the base density $b(\bfu;\bftheta_u)$. We begin by expanding the KL divergence between the variational density and the posterior density as follows \citep{nf2021jmlr}:
\begin{align} \label{eq: nf_kl}
    D_{KL}&(q(\bfx; \bftheta) \| p(\bfx | \bfz)) = \int q(\bfx; \bftheta) \log \left( \frac{q(\bfx; \bftheta)}{p(\bfx | \bfz)} \right) \d \bfx \\
    &= \int b(\bfu; \bftheta_u) | \det(\nabla_{\bfu}F(\bfu; \bftheta_F)) |^{-1} \log \left( \frac{b(\bfu; \bftheta_u) | \det(\nabla_{\bfu} F(\bfu; \bftheta_F)) |^{-1}}{p(F(\bfu; \bftheta_F) | \bfz)} \right) \d F(\bfu; \bftheta_F) \\
    &= \int b(\bfu; \bftheta_u) \bigg[ \log(b(\bfu; \bftheta_u)) - \log( | \det(\nabla_{\bfu} F(\bfu; \bftheta_F)) |) - \log(p(F(\bfu; \bftheta_F) | \bfz)) \bigg] \d \bfu \\
    & = \int b(\bfu; \bftheta_u) \bigg[ \log(b(\bfu; \bftheta_u)) - \log(p(\bfu | \bfz; \bftheta_F)) \bigg] \d \bfu \\
    &= D_{KL}(b(\bfu; \bftheta_u) \| p(\bfu | \bfz; \bftheta_F)) ,
\end{align}
where the transformed posterior density $p(\bfu | \bfz; \bftheta_F)$ is obtained by applying a change of variables to transform the posterior on $\bfx$ as:
\begin{equation}\label{eq: induced_base}
    p(\bfu | \bfz; \bftheta_F) = p(F(\bfu; \bftheta_F) | \bfz) | \det(\nabla_{\bfu} F(\bfu; \bftheta_F)) |.
\end{equation}
For fixed $F(\bfu; \bar{\bftheta}_F)$, we can cast the base density optimization as a variational inference problem using $b(\bfu; \bftheta_u)$ to approximate the transformed posterior $p(\bfu | \bfz; \bar{\bftheta}_F)$:
\begin{equation}\label{application: base_optimization}
    \min_{\bftheta_u \in \Theta_u} D_{KL}(b(\bfu; \bftheta_u) \| p(\bfu | \bfz; \bar{\bftheta}_F) ),
\end{equation}
where $\Theta_u$ is the admissible parameter set for the base variational density. Selecting the base density as Gaussian or Gaussian mixture allows the methods developed in the previous sections to be applied. The Fisher–Rao parameter flow \eqref{fisher_rao: fr_para_flow} can be constructed using the methods discussed in Sections~\ref{sec: gaussian_flow} and~\ref{sec: gaussian_mixture_approx}. The resulting parameter evolution induces a time-varying base density $b(\bfu; [\bftheta_u]_t)$ that solves the optimization problem in \eqref{application: base_optimization}. 

Although this method generalizes our particle flow approach to variational densities beyond Gaussian families, we need to address the selection of an appropriate fixed transformation. Next, we deal with the joint optimization of the base density and the transformation. As can be seen from \eqref{eq: nf_kl}, the minimum KL divergence $D_{KL}(b(\bfu; \bftheta_u) \| p(\bfu | \bfz; \bftheta_F))$ can be achieved when the transformed posterior density belongs to the same distribution family as the base density. This observation motivates the consideration of a joint optimization problem over both base density parameters and the transformation parameters:
\begin{equation}\label{application: joint_optimization}
    \min_{\bftheta \in \Theta_u \times \Theta_F} D_{KL}(b(\bfu; \bftheta_u) \| p(\bfu | \bfz; \bftheta_F) ),
\end{equation}
where $\Theta_F$ is the admissible parameter set for the transformation. To solve the optimization problem \eqref{application: joint_optimization}, one can introduce another gradient flow for the transformation parameters $\bftheta_F$ and combine it with the Fisher-Rao parameter flow \eqref{fisher_rao: fr_para_flow} for the base variational density parameters, which leads to the following parameter flow in the joint parameter space:
\begin{equation}\label{application: joint_parameter_flow}
    \frac{\partial \bftheta_t}{\partial t} = \begin{pmatrix} \partial [\bftheta_u]_t / \partial t \\ \partial [\bftheta_F]_t / \partial t \end{pmatrix} = \begin{pmatrix} -\calI^{-1}([\bftheta_u]_t) \nabla_{[\bftheta_u]_t} D_{KL}(b(\bfu; [\bftheta_u]_t) \| p(\bfu | \bfz; [\bftheta_F]_t) ) \\ - \gamma \nabla_{[\bftheta_F]_t} D_{KL}(b(\bfu; [\bftheta_u]_t) \| p(\bfu | \bfz; [\bftheta_F]_t) ) \end{pmatrix}, 
\end{equation}
where $\calI([\bftheta_u]_t)$ denotes the FIM defined in \eqref{fisher_rao: fisher_rao_tensor} associated with the base variational density, and $\gamma > 0$ is a scaling constant introduced to control the magnitude of the gradient associated with the transformation parameters. The parameter flow above induces a time-varying transformation governed by:
\begin{equation} \label{application: tv_transformation}
    \frac{\partial F(\bfu; [\bftheta_F]_t)}{\partial t} = \frac{\partial F(\bfu; [\bftheta_F]_t)}{\partial [\bftheta_F]_t} \frac{\partial [\bftheta_F]_t}{\partial t}. 
\end{equation}
The parameter flow \eqref{application: joint_parameter_flow} combines a natural gradient flow in the base density parameters with a standard gradient flow in the transformation parameters. This construction enables the use of the Fisher-Rao particle flows developed in Sections~\ref{sec: gaussian_flow} and~\ref{sec: gaussian_mixture_approx} to obtain accurate particles $\{[\bfu_j]_t\}_{j=1}^{M}$ from the base variational density. As discussed in Section~\ref{sec: derivative_inverse_free}, in practice, these particles can be used to approximate the gradient of the KL divergence with respect to the joint parameters $\bftheta_t$ as follows:
\begin{equation} \label{application: kl_approximation}
    \nabla_{\bftheta_t} D_{KL}(b(\bfu; [\bftheta_u]_t) \| p(\bfu | \bfz; [\bftheta_F]_t) ) \approx \frac{1}{M} \sum_{i=1}^M \nabla_{\bftheta_t} \frac{b([\bfu_i]_t; [\bftheta_u]_t)}{p([\bfu_i]_t, \bfz; [\bftheta_F]_t)},
\end{equation}
where $p(\bfu, \bfz; \bftheta_F) = p(F(\bfu; \bftheta_F), \bfz) | \det(\nabla_{\bfu} F(\bfu; \bftheta_F)) |$ denotes the transformed joint density. A more detailed discussion of using this particle-based approximation in high-dimensional settings is provided in Appendix~\ref{app:C}. Moreover, the flow \eqref{application: joint_parameter_flow} defines a valid descent direction on the joint parameter space $\Theta_u \times \Theta_F$ for the KL divergence objective. This result is stated below, with proof deferred to Appendix~\ref{app:A}.

\begin{lemma}[Valid Descent Direction]
    Consider the parameter flow \eqref{application: joint_parameter_flow}. If $\frac{\partial \bftheta_t}{\partial t} \neq \mathbf{0}$, then there exists $\epsilon > 0$ such that the following condition holds:
    \begin{equation}
        D_{KL}\left( q\left( \bfx; \bftheta_t + \alpha \frac{\partial \bftheta_t}{\partial t} \right) \middle\|\; p(\bfx | \bfz) \right) < D_{KL}\left( q\left( \bfx; \bftheta_t\right) \| \; p(\bfx | \bfz) \right), \quad \forall \alpha \in (0, \epsilon).
    \end{equation}
\end{lemma}

With the parameter flow \eqref{application: joint_parameter_flow} defining a valid descent direction for the KL divergence, we next show that, given a particle dynamics function associated with the base variational density $b(\bfu)$ that satisfies the Liouville equation, one can construct a corresponding particle dynamics function for the variational density.

\paragraph{Particle Dynamics Function Construction}
Now, consider a particle dynamics function $\bfphi_u(\bfu, t)$ that satisfies the Liouville equation associated with the Fisher-Rao parameter flow. Our goal is to derive the corresponding particle dynamics function for the transformed variational density $q(\bfx; \bftheta_t)$. To this end, we first introduce a weak form of the Liouville equation. 

\begin{definition}[Weak Form of the Liouville Equation \citep{ambrosio2005gradient}]
    A particle dynamics function $\bfphi(\bfx, t)$ satisfies the Liouville equation \eqref{prelim: liouville} associated with density $p(\bfx; t)$ in the sense of distributions if the following condition holds:
    \begin{equation}
        \int_0^T \int_{\bbR^n} \left(\nabla_t \varphi(\bfx, t) + \bfphi^{\top}(\bfx, t) \nabla_{\bfx} \varphi(\bfx, t) \right) p(\bfx; t) \d \bfx \d t = 0, \quad \forall \varphi \in C^{1}_{c}(\bbR^n \times (0, T)), 
    \end{equation}
    where $C^{1}_{c}(\bbR^n \times (0, T))$ denotes the space of continuously differentiable functions with compact support in $\bbR^n \times (0, T)$.
\end{definition}
With the definition above, we can construct a transformed particle dynamics that satisfies the Liouville equation of the transformed variational density $q(\bfx; \bftheta_t)$ in the sense of distributions. 
\begin{proposition}[Transformed Particle Dynamics Function]\label{prop: transformed_pf}
    Let $\bfphi_u(\bfu, t)$ be a particle dynamics function for the base variational density satisfying:
    \begin{equation} \label{application: liouville_base}
        \nabla_t b(\bfu; [\bftheta_u]_t) = \frac{\partial b(\bfu; [\bftheta_u]_t)}{\partial [\bftheta_u]_t} \frac{\partial [\bftheta_u]_t}{\partial t} = - \nabla_{\bfu} \cdot \left(b(\bfu; [\bftheta_u]_t) \bfphi_u(\bfu, t) \right),
    \end{equation}
    for all $\bfu \in \bbR^n$ and $t \in [0, T]$, where the time derivative of the base variation density parameters $[\bftheta_u]_t$ is specified by the parameter flow \eqref{application: joint_parameter_flow}. Suppose the transformation $F(\bfu; [\bftheta_F]_t)$ is proper, i.e., the preimage $U = F^{-1}(X; [\bftheta_F]_t)$ of every compact set $X \subset \bbR^n$ is also compact. Furthermore, assume that $F(\bfu; [\bftheta_F]_t)$ is continuously differentiable in its parameters, and that the gradient of the KL divergence with respect to the transformation parameters $\nabla_{[\bftheta_F]_t} D_{KL}(b(\bfu; [\bftheta_u]_t) \| p(\bfu | \bfz; [\bftheta_F]_t) )$ is locally Lipschitz. Under these conditions, the transformed particle dynamics function induced by the time-varying transformation $F(\bfu; [\bftheta_F]_t)$ in \eqref{application: tv_transformation} is given by:
    \begin{equation} \label{application: transformed_pf}
        \bfphi_{F}(\bfx, t) = \left[ \frac{\partial F(\bfu; [\bftheta_F]_t)}{\partial t} + \nabla_{\bfu} F(\bfu; [\bftheta_F]_t) \bfphi_u(\bfu, t) \right] \Bigg|_{\bfu = F^{-1}(\bfx; [\bftheta_F]_t)},
    \end{equation}
    which satisfies the following Liouville equation in the sense of distributions:
    \begin{equation}
        \frac{\partial q(\bfx; \bftheta_t)}{\partial t} = - \nabla_{\bfx} \cdot \left( q(\bfx; \bftheta_t) \bfphi_F(\bfx, t) \right). 
    \end{equation}
\end{proposition}

We defer the proof of the result to Appendix~\ref{app:B}. Proposition~\ref{prop: transformed_pf} indicates that, to retrieve the transformed particles $[\bfx_i]_t$ at time $t$ evolving according to the particle flow \eqref{prelim: particle_dynamics} defined by the transformed particle dynamics function \eqref{application: transformed_pf}, it suffices to propagate the base variational density particles $\bfu_i$ through the particle flow \eqref{prelim: particle_dynamics} governed by the particle dynamics function associated with the base variational density $\bfphi_u(\bfu, t)$ and apply the transformation $F(\bfu_i; [\bftheta_F]_t)$. Specifically, 
\begin{equation}
    [\bfx_i]_t = F([\bfu_i]_t; [\bftheta_F]_t), \quad \frac{\partial [\bfu_i]_t}{\partial t} = \bfphi_u([\bfu_i]_t, t). 
\end{equation}
With this formulation, we can evolve the joint parameters according to \eqref{application: joint_parameter_flow} while simultaneously simulating the particle flow associated with the base variational density, as discussed in Sections~\ref{sec: gaussian_flow} and~\ref{sec: gaussian_mixture_approx}. The transformation is then applied at the final time to obtain particles associated with the variational density. Utilizing this property, we proposed a particle flow-based normalizing flow, with its key steps summarized in Algorithm~\ref{alg: pf_nf}.

In the next section, we focus on the case where the base variational density is selected as a Gaussian mixture and present numerical results demonstrating the effectiveness of the proposed integrated particle flow and normalizing flow.

\begin{algorithm}[t]
\caption{Particle Flow-Based Normalizing Flow}
\label{alg: pf_nf}
\begin{algorithmic}[1]
\small
\Require Initial base variational density parameters $[\bftheta_u]_{0}$ defined according to \eqref{gaussian_approx: param_space} or \eqref{gaussian_mixture_approx: natural_parameters}, particles $\{[\bfu_j]_0\}_{j=1}^{M}$ sampled from the initial base variational density, initial transformation parameters $[\bftheta_F]_{0}$, and the joint density $p(\bfx, \bfz)$
\Output Particles $\{\bfx_j\}_{j=1}^{M}$ that approximate the posterior density $p(\bfx | \bfz)$

\Function{$\bff$}{$\{[\bfu_j]_t\}_{j=1}^{M}$, $\bftheta_t$, $t$}
    \State $\delta \bftheta_t \gets$ Evaluate \eqref{application: joint_parameter_flow} with joint parameters $\bftheta_t$
    \For{each base variational density particle \text{$[\bfu_j]_t$}}
    \State $\tilde{A}^{(j)}_t$, $\tilde{\bfb}^{(j)}_t \gets$ Evaluate \eqref{gaussian_approx: approx_particle_flow} or \eqref{gaussian_mixture_approx: component_particle_flow} using the transformed joint density $p(\bfu, \bfz; [\bftheta_F]_t)$
    \State $\tilde{\bfphi}([\bfu_j]_t, t) \gets \tilde{A}^{(j)}_t [\bfu_j]_t + \tilde{\bfb}^{(j)}_t$
    \EndFor
\State \Return $\{\tilde{\bfphi}([\bfx_j]_t, t)\}_{j=1}^{M}$, $\delta \bftheta_t$
\EndFunction

\While{ODE solver running}
    \State $\{[\bfu_j]_{t=T}\}_{j=1}^{M}, \bftheta_T \gets$ SolveODE($\bff(\{[\bfu_j]_t\}_{j=1}^{M}, \bftheta_t, t)$) with initial particles $\{[\bfu_j]_{t=0}\}_{j=1}^{M}$, initial parameters $\bftheta_0$, initial time $t=0$, and termination time $T$
\EndWhile

\For{each particle $\{[\bfu_j]_T\}_{j=1}^{M}$}
\State Get transformed particle $\bfx_j \gets F([\bfu_j]_T, [\bftheta_F]_T)$
\EndFor

\State \Return $\{\bfx_j\}_{j=1}^{M}$
\end{algorithmic}
\end{algorithm}

\subsection{Numerical Results}
We present numerical results demonstrating the effectiveness of our method when combined with the normalizing flow in three scenarios, including the two low-dimensional cases studied in Sections~\ref{sec: eval_gmm} and~\ref{sec: eval_nonlinear} for completeness, as well as an extension of the funnel posterior studied in \citet{blessing2024beyond}.

\paragraph{Low-Dimensional Cases}

For both the Gaussian mixture prior case studied in Section~\ref{sec: eval_gmm} and the nonlinear observation model case studied in Section~\ref{sec: eval_nonlinear}, we use a Gaussian mixture with $5$ components as the initial base variational density $b(\bfu; [\bftheta_u]_{t=0})$. In both cases, the initial weight for the $k$th Gaussian mixture component is set to $\omega^{(k)} = 1 / 5$, and the structure of the transformation is restricted to the composition of $5$ planar transformations \eqref{eq: planar}. 

The initialization of the Gaussian mixture parameter differs slightly between the two cases. For the Gaussian mixture prior case, the initial mean parameter for the $k$th Gaussian mixture component $\bfmu^{(k)}$ is sampled from a Gaussian density centered at the origin with covariance $5I$, i.e., $\bfmu^{(k)} \sim \calN(\mathbf{0}, 5I)$. The initial covariance parameter of the $k$th Gaussian mixture component is set to $\Sigma^{(k)} = 15I$. For the nonlinear observation model case, the initial mean parameter for the $k$th Gaussian mixture component $\bfmu^{(k)}$ is sampled from a Gaussian density centered at the origin with covariance $4I$, i.e., $\bfmu^{(k)} \sim \calN(\mathbf{0}, 4I)$. The initial covariance parameter of the $k$th Gaussian mixture component is set to $\Sigma^{(k)} = I$.

The results for the Gaussian mixture prior case in Section~\ref{sec: eval_gmm} are shown in the first row of Figure~\ref{fig: nf_results}. We show contour plots of the initial variational density $q(\bfx; \bftheta_{t=0})$, the optimized variational density $q(\bfx; \bftheta_{t=T})$, and the optimized base variational density $b(\bfu; [\bftheta_u]_{t=T})$. In this case, there is a significant mismatch between the initial variational density and the posterior density. After applying the proposed particle flow-based normalizing flow, cf. Proposition~\ref{prop: transformed_pf}, the resulting approximation demonstrates substantially improved alignment with the posterior density, leading to a considerable reduction in KL divergence compared with the initial base density. The results for the nonlinear observation model in Section~\ref{sec: eval_nonlinear} are shown in the second row of Figure~\ref{fig: nf_results}. Similar to the Gaussian mixture prior case, the initial variational density differs markedly from the posterior density. Applying the proposed particle flow-based normalizing flow again leads to a noticeably improved approximation. We want to highlight that the accurate approximation is attained through the joint optimization of the base variational density and the transformation. As shown in the last column of Figure~\ref{fig: nf_results}, the optimized base variational density still differs significantly from the posterior density.

\paragraph{High-Dimensional Case}

We consider a high-dimensional extension of the funnel posterior studied in \citet{blessing2024beyond} with $N=30$, defined as:
\begin{equation}\label{eq: funnel}
    p(\bfx) = p_{\calN}(x_1; 0, 9) p_{\calN}(\bfx_{2:N}; \mathbf{0}, \exp(x_1) I). 
\end{equation}
Due to the nonlinearity of the funnel posterior, simple transformations such as the planar transformation~\eqref{eq: planar} and the radial transformation~\eqref{eq: radial} are not expressive enough to capture its structure. We introduce the following triangular transformation, inspired by the conditional dependency structure of the funnel posterior:
\begin{equation} \label{eq: triangular_trans}
    F(\bfu) = B\bfu + \bfb + \exp(L\bfu + \bfl) \odot \bfu, 
\end{equation}
where $\odot$ denotes the Hadamard product and the exponential $\exp(\cdot)$ is applied element-wise. Here, we have $\bfb, \bfl \in \mathbb{R}^{n}$ vector bias parameters, and $B, L \in \mathbb{R}^{n \times n}$ are strictly lower triangular matrix parameters, i.e., lower triangular matrix with zeros on the diagonal. The strict lower triangular structure of $B$ and $L$ induces a triangular Jacobian, thereby encoding the conditional dependency structure of the funnel posterior and enabling efficient computation of the Jacobian determinant.

In this case, we use the same initialization strategy for the base variational density as in the nonlinear observation model case discussed above. We also compare several transformation structures, including a single triangular transformation \eqref{eq: triangular_trans}, as well as compositions of five and ten planar transformations \eqref{eq: planar} and radial transformations \eqref{eq: radial}.

The results are presented in Figure~\ref{fig: funnel_results}. Due to the high dimensionality, visualizing the density is not practical. Instead, we visualize the particle projected to the first two dimensions. The ground-truth particles are shown in the top left figure, where the expected funnel shape is clearly captured. Across all cases, the variational density using a single triangular transformation yields the most accurate samples. The variational density using radial transformations yields more accurate samples than that using planar transformations. For both planar and radial transformation cases, using five compositions is sufficient to generate samples that clearly capture the structure of the funnel posterior, as increasing the number of composed transformations does not lead to noticeable improvement. The results also indicate that optimizing only the transformation, without jointly optimizing the base variational density, does not result in satisfactory sampling performance.

\begin{figure}[ht]
    \centering
    \includegraphics[width=0.33\linewidth]{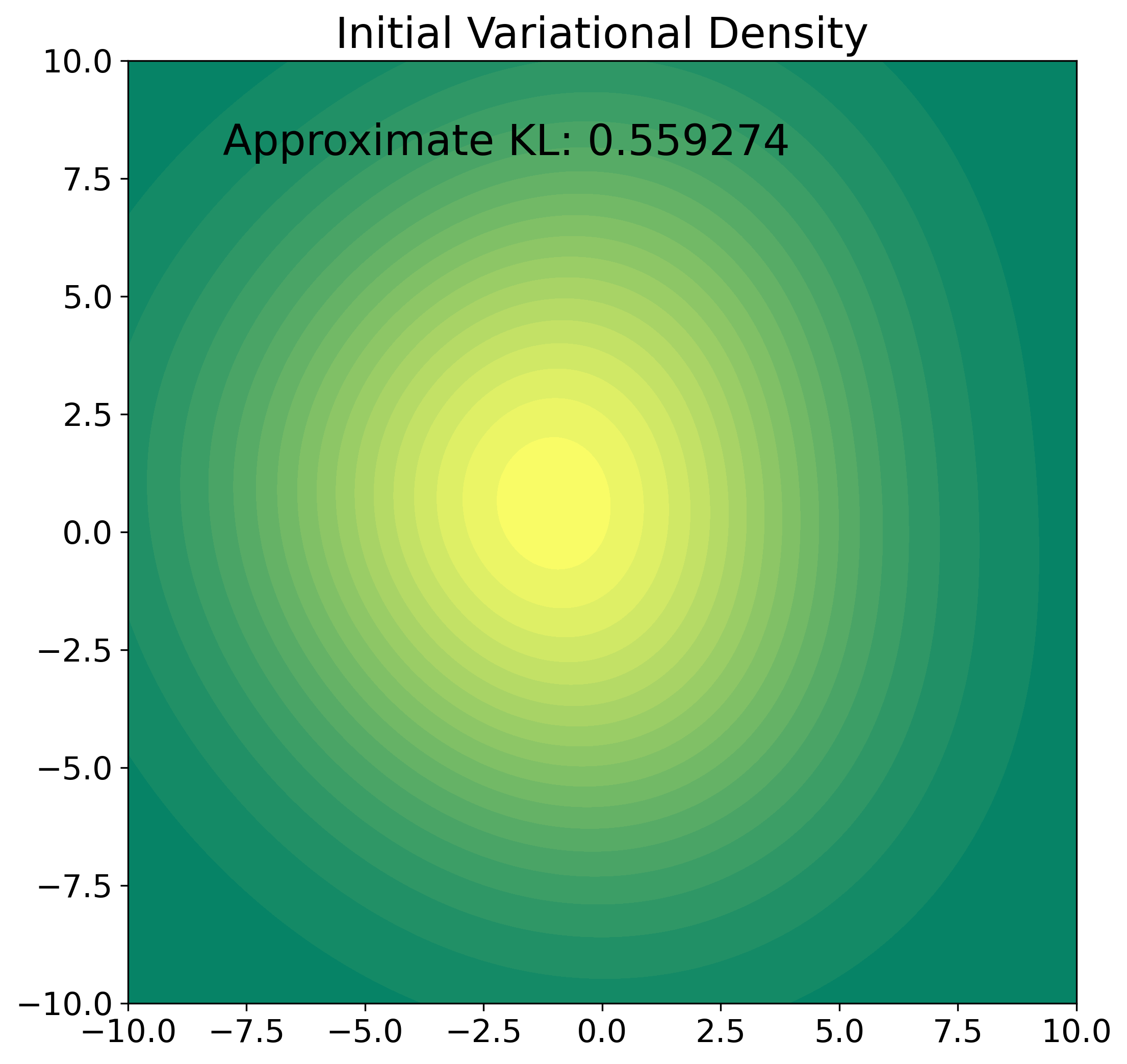}%
    \includegraphics[width=0.33\linewidth]{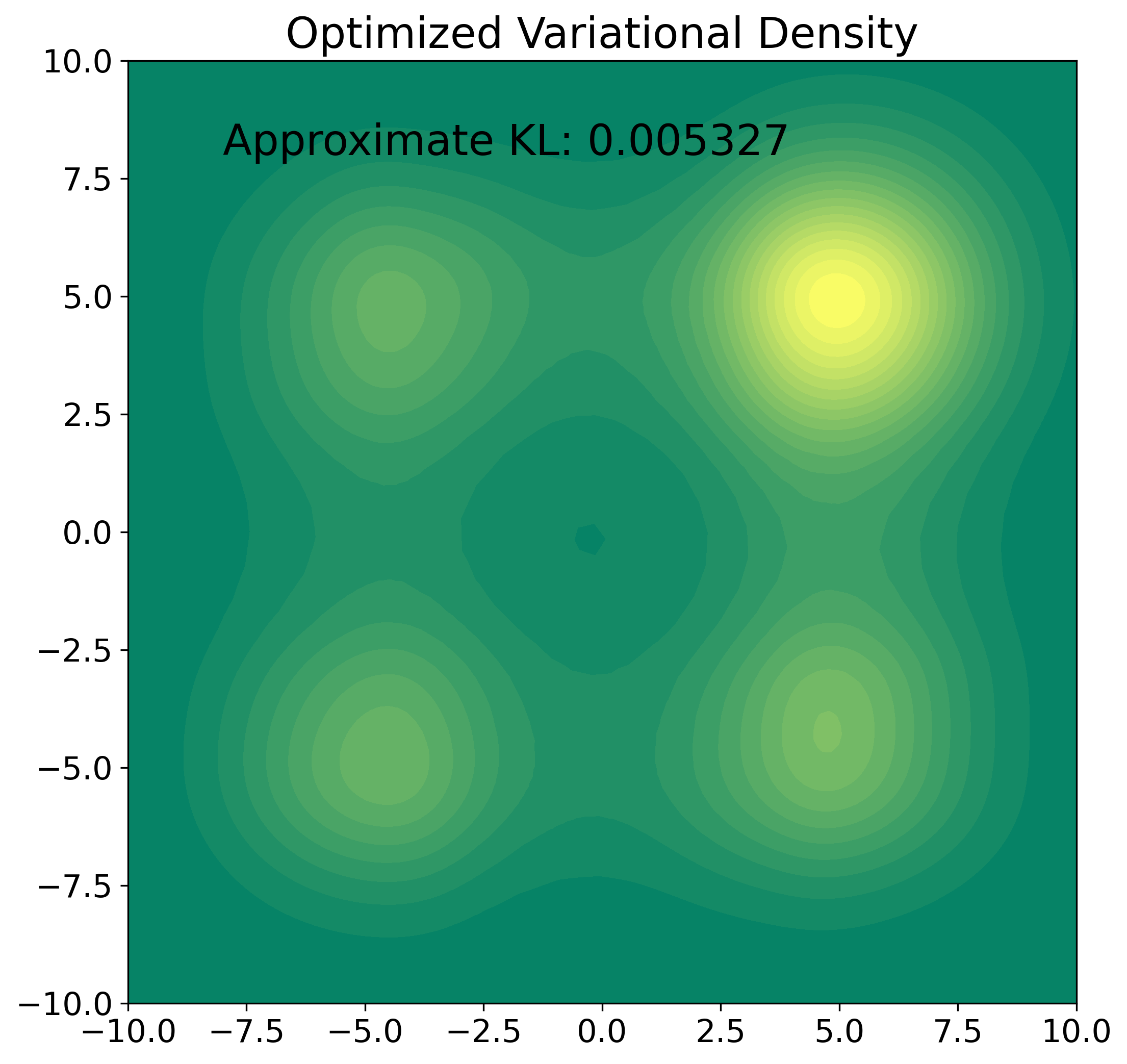}%
    \includegraphics[width=0.33\linewidth]{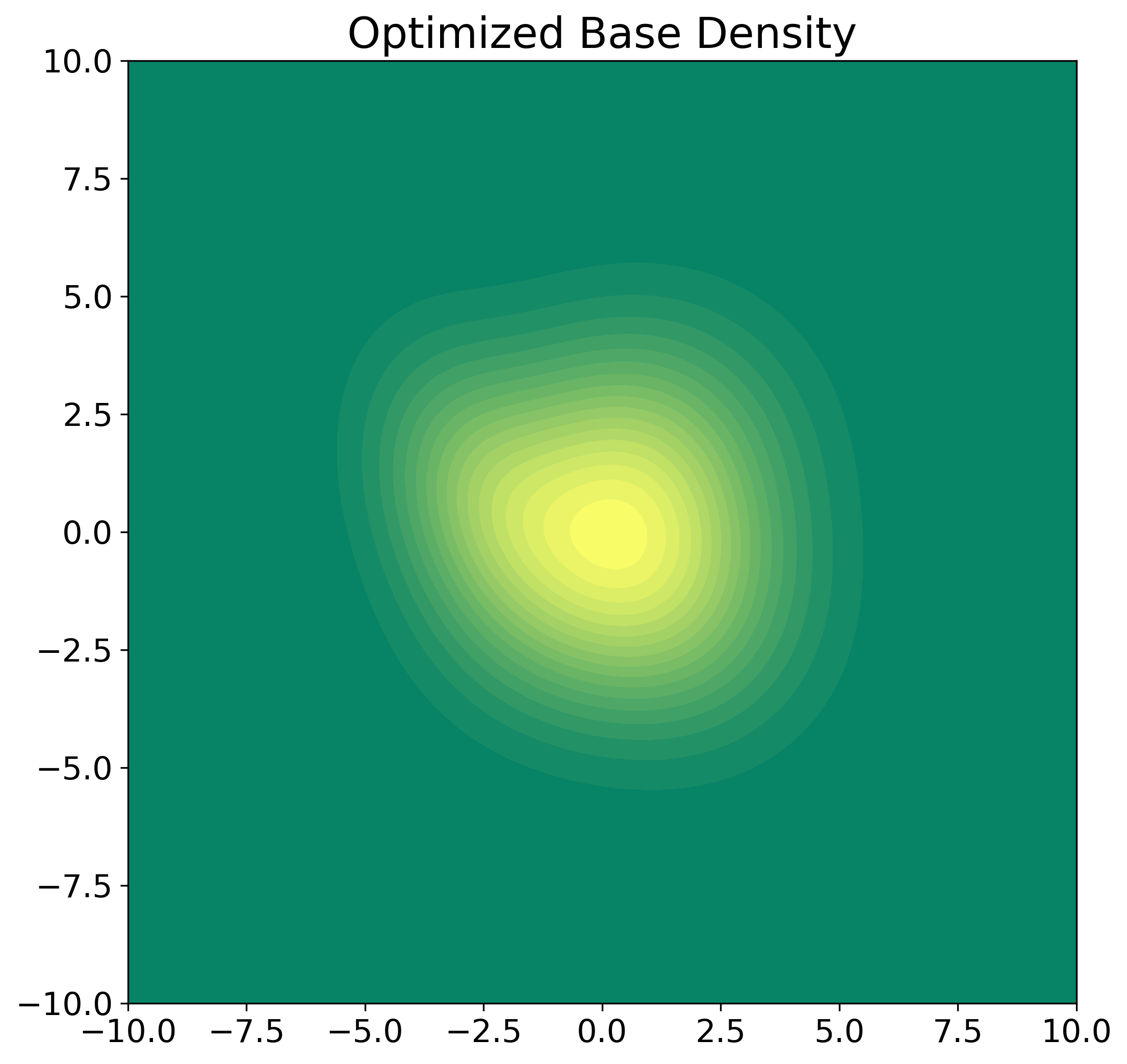}%
    \hfill%
    \includegraphics[width=0.33\linewidth]{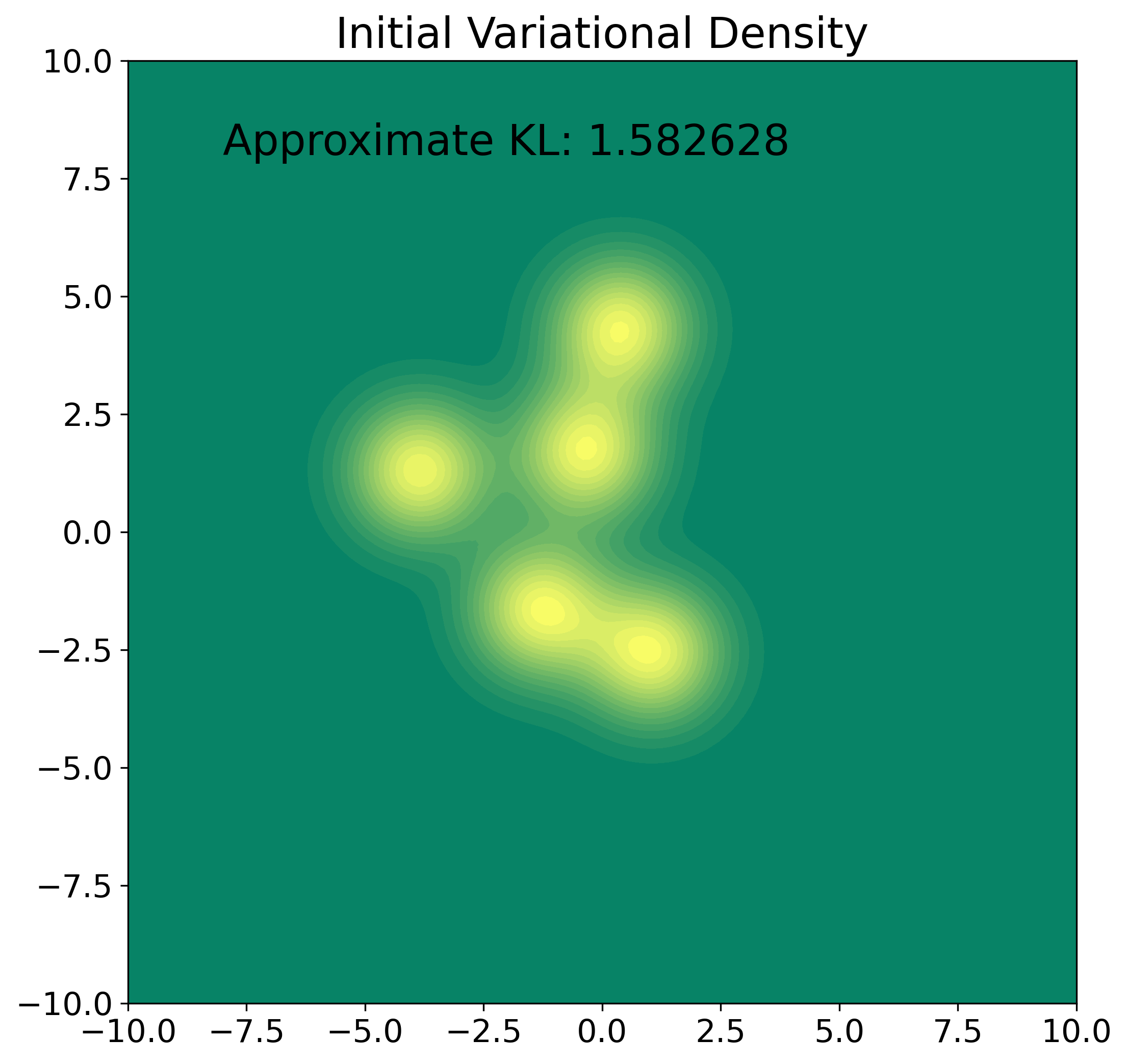}%
    \includegraphics[width=0.33\linewidth]{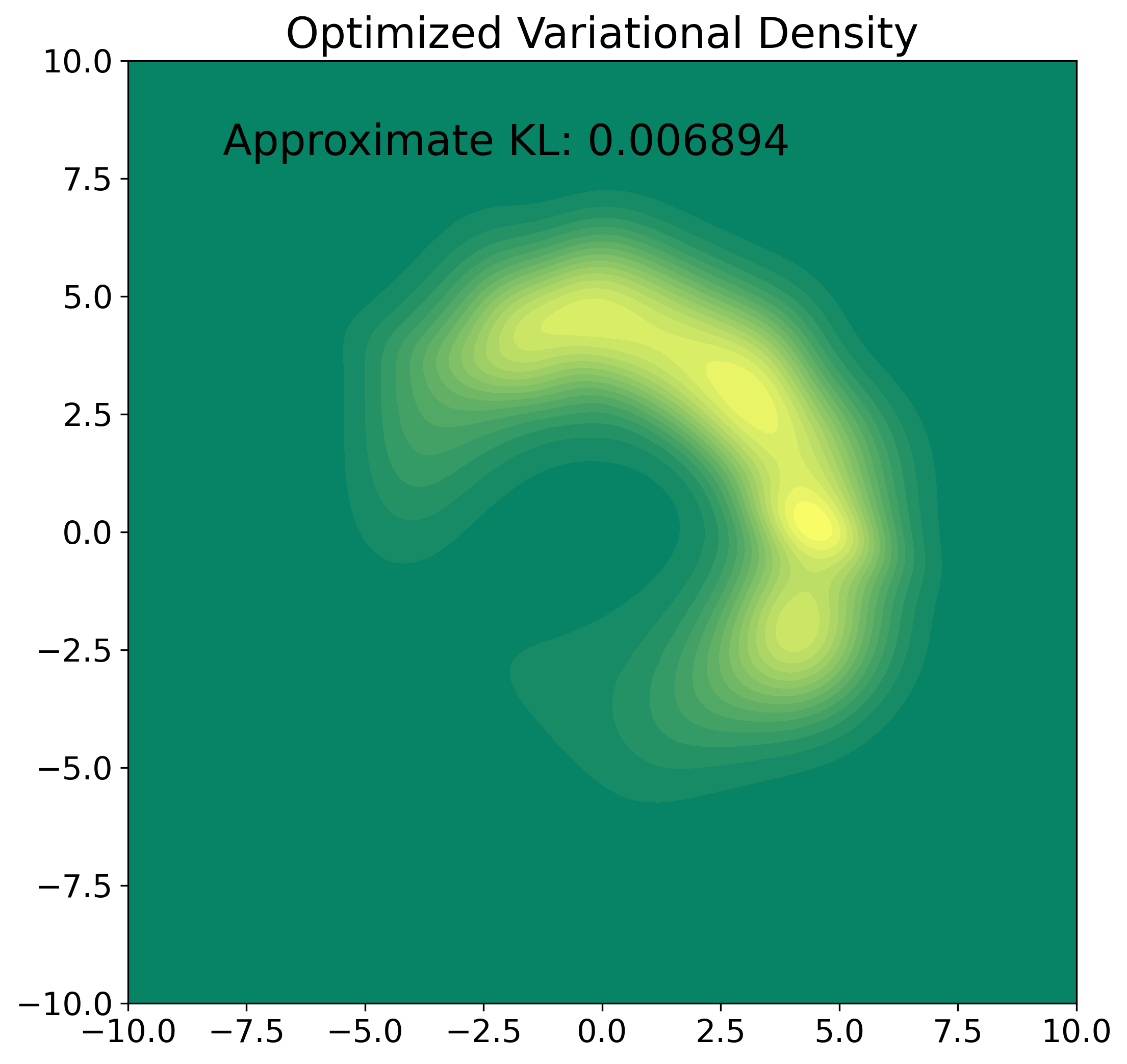}%
    \includegraphics[width=0.33\linewidth]{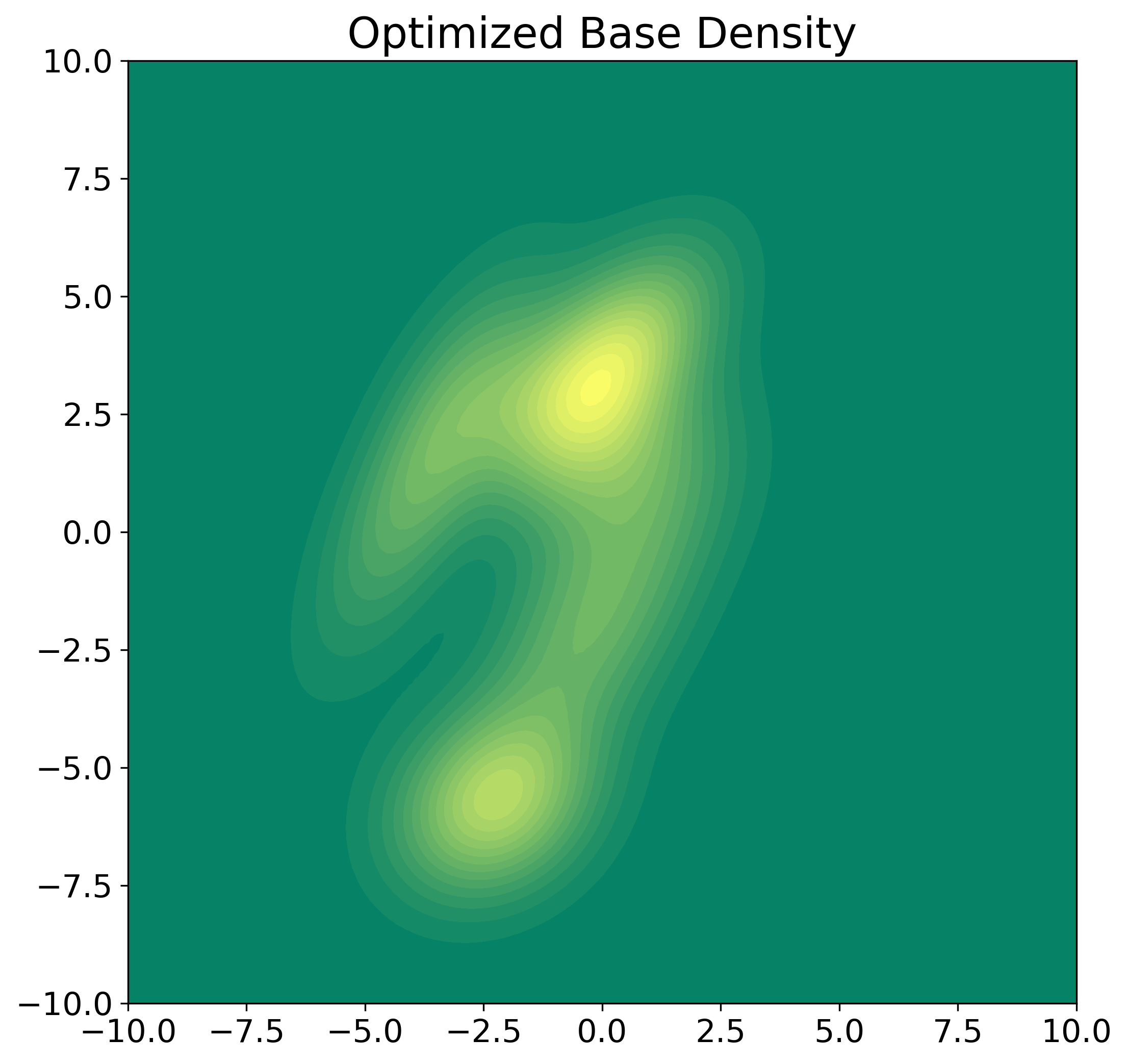}%
    \caption{Low-dimensional results obtained by the proposed particle flow-based normalizing flow \eqref{application: transformed_pf}, associated with the parameter flow \eqref{application: joint_parameter_flow}. The first row corresponds to the Gaussian mixture prior case in Section~\ref{sec: eval_gmm} and the second row corresponds to the nonlinear observation model case in Section~\ref{sec: eval_nonlinear}, respectively. In each row, the left panel shows the initial variational density $q(\bfx; \bftheta_{t=0})$, the middle panel shows the optimized variational density $q(\bfx; \bftheta_{t=T})$, and the right panel shows the optimized base variational density $b(\bfu; [\bftheta_u]_{t=T})$. In both cases, the optimized variational density achieves a good approximation of the posterior, achieving a low KL divergence approximation. The results demonstrate that the accurate approximation is attained through the joint optimization of the base variational density and the transformation. In particular, the optimized base variational density alone still differs significantly from the posterior density.}
    \label{fig: nf_results}
\end{figure}

\begin{figure}[ht]
    \centering
    \includegraphics[width=0.33\linewidth,trim=12mm 16mm 16mm 20mm, clip]{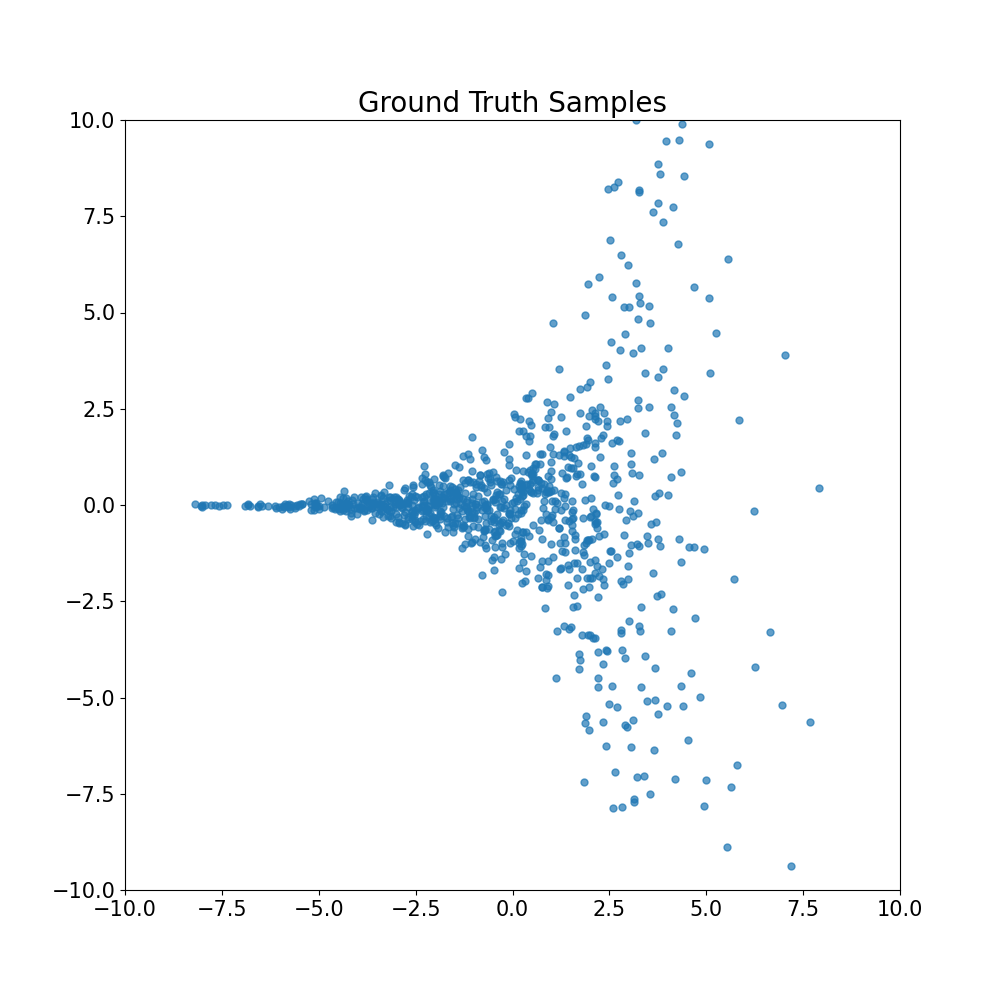}%
    \includegraphics[width=0.33\linewidth,trim=12mm 16mm 16mm 20mm, clip]{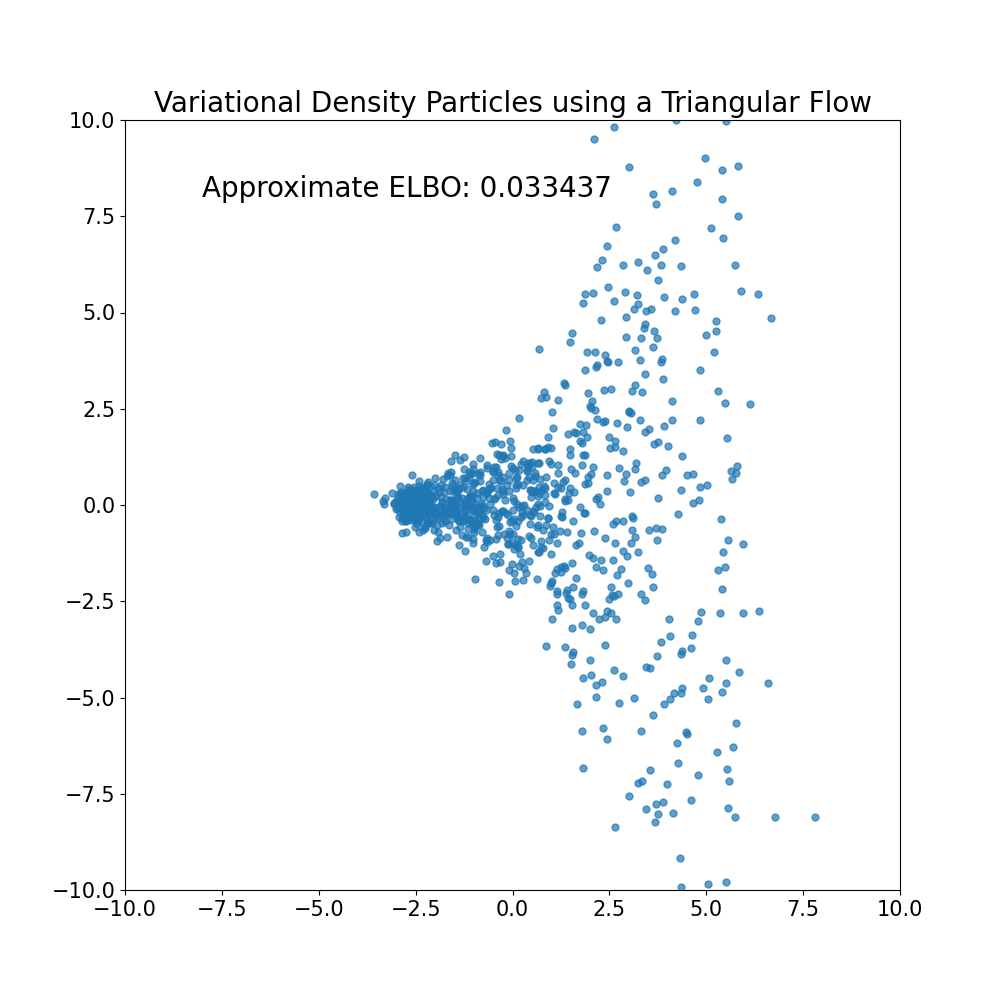}%
    \includegraphics[width=0.33\linewidth,trim=12mm 16mm 16mm 20mm, clip]{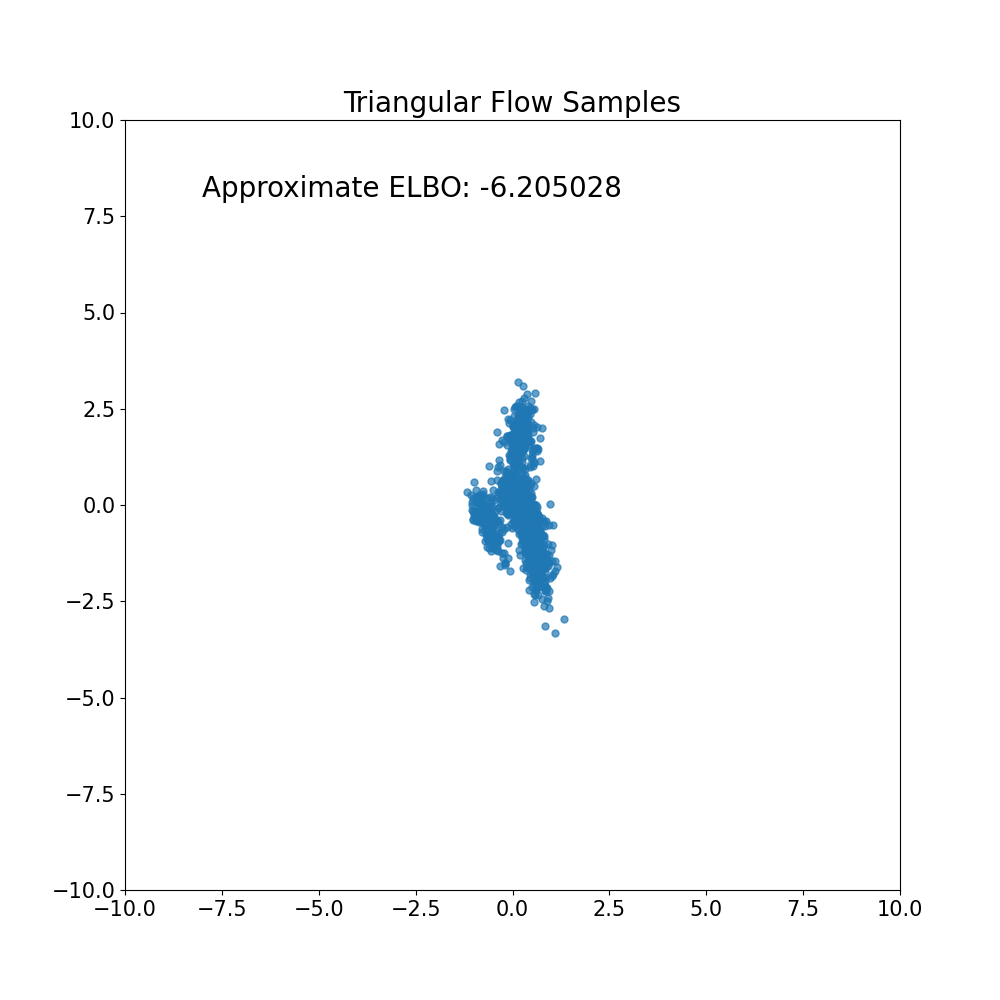}%
    \hfill%
    \includegraphics[width=0.33\linewidth,trim=12mm 16mm 16mm 20mm, clip]{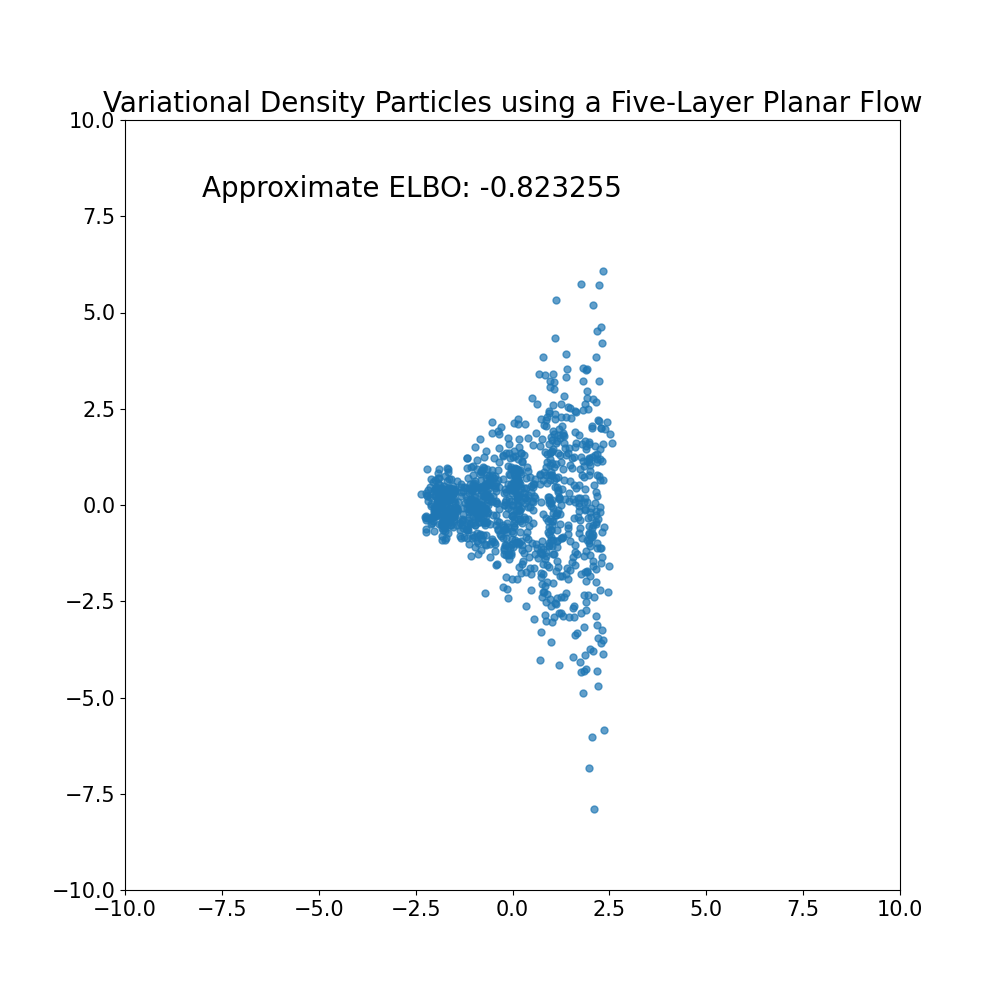}%
    \includegraphics[width=0.33\linewidth,trim=12mm 16mm 16mm 20mm, clip]{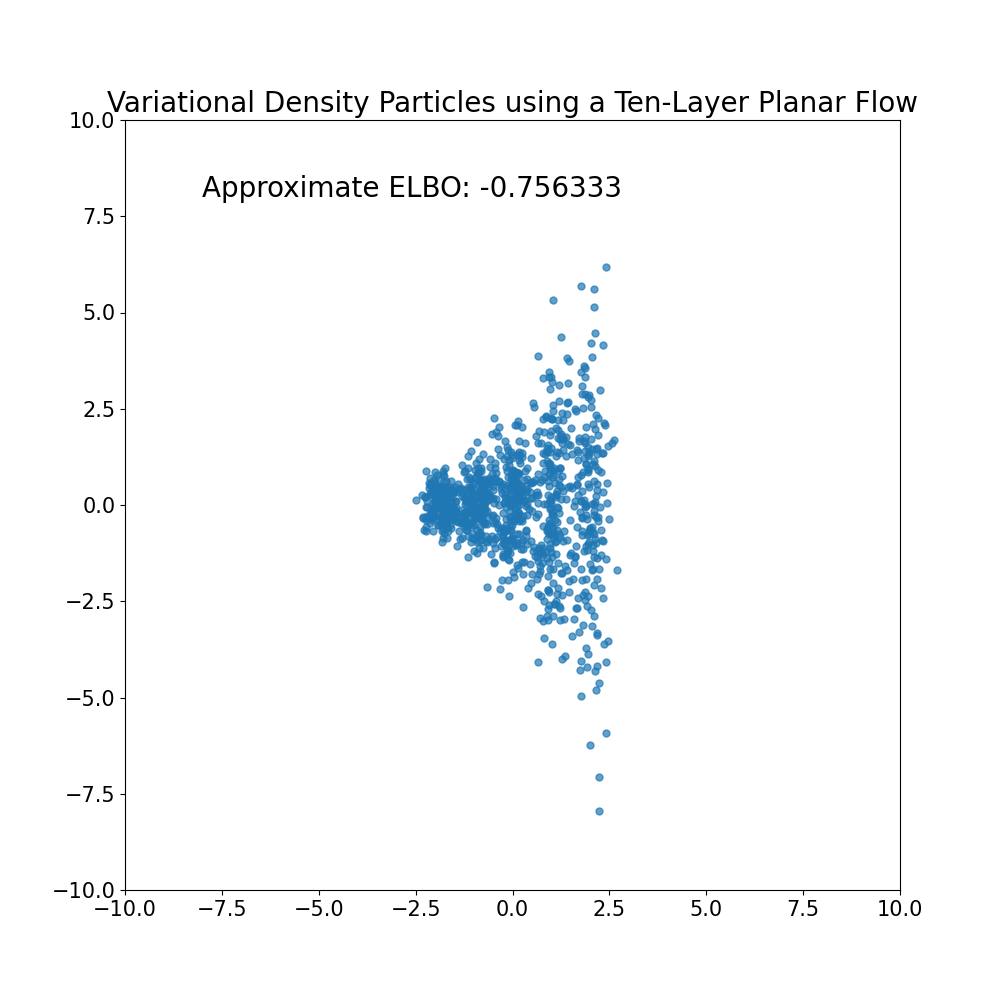}%
    \includegraphics[width=0.33\linewidth,trim=12mm 16mm 16mm 20mm, clip]{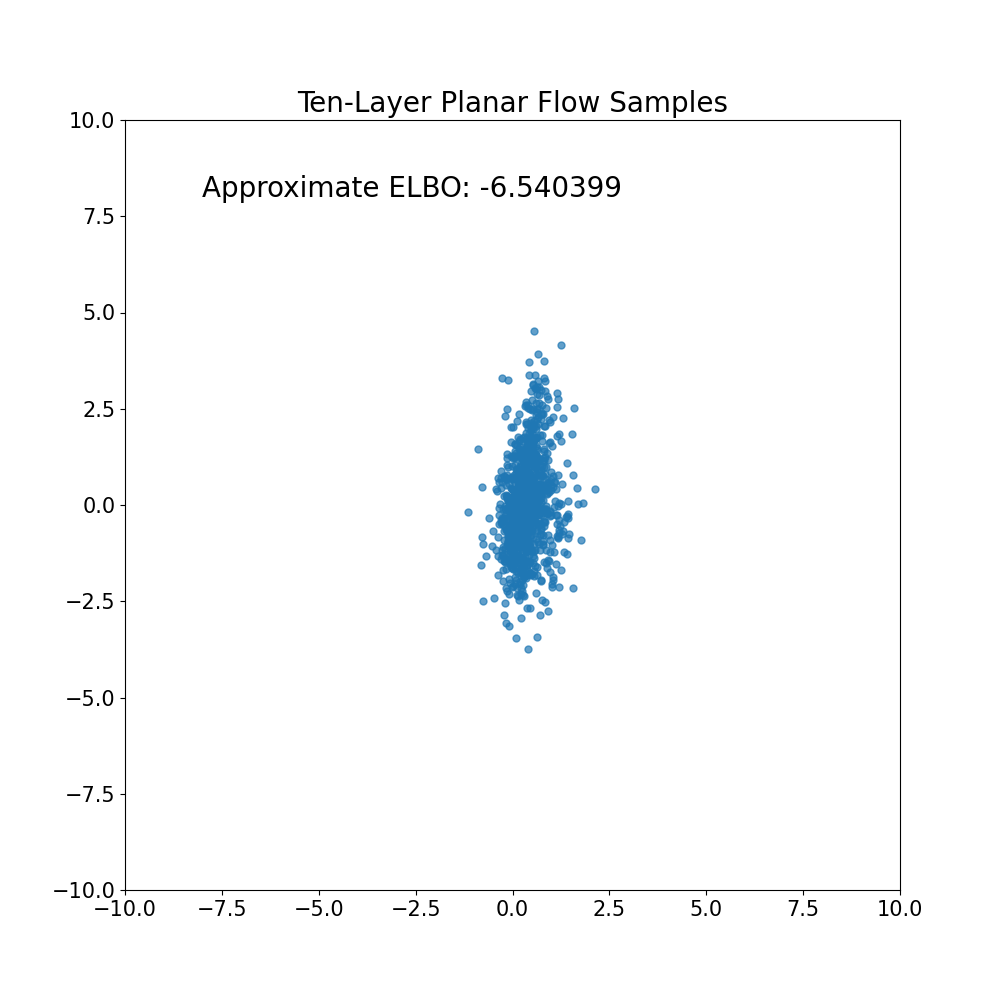}%
    \hfill%
    \includegraphics[width=0.33\linewidth,trim=12mm 16mm 16mm 20mm, clip]{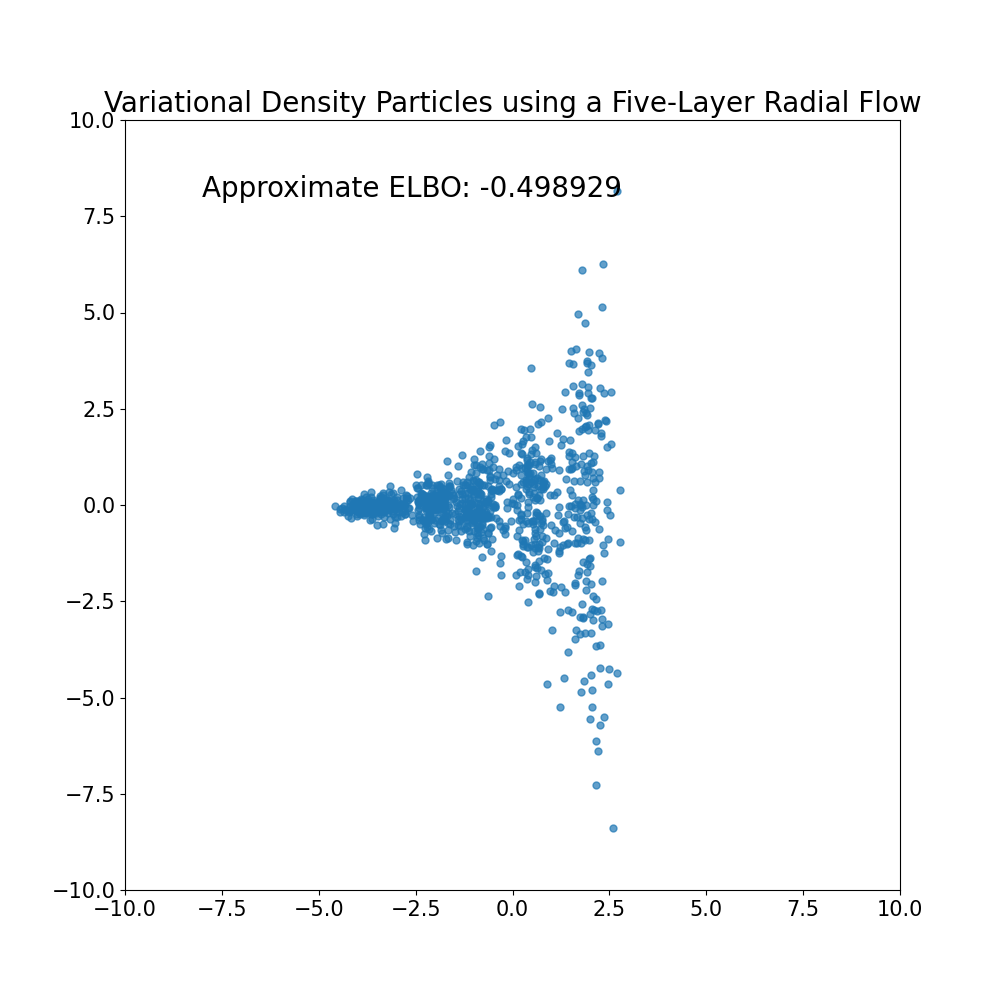}%
    \includegraphics[width=0.33\linewidth,trim=12mm 16mm 16mm 20mm, clip]{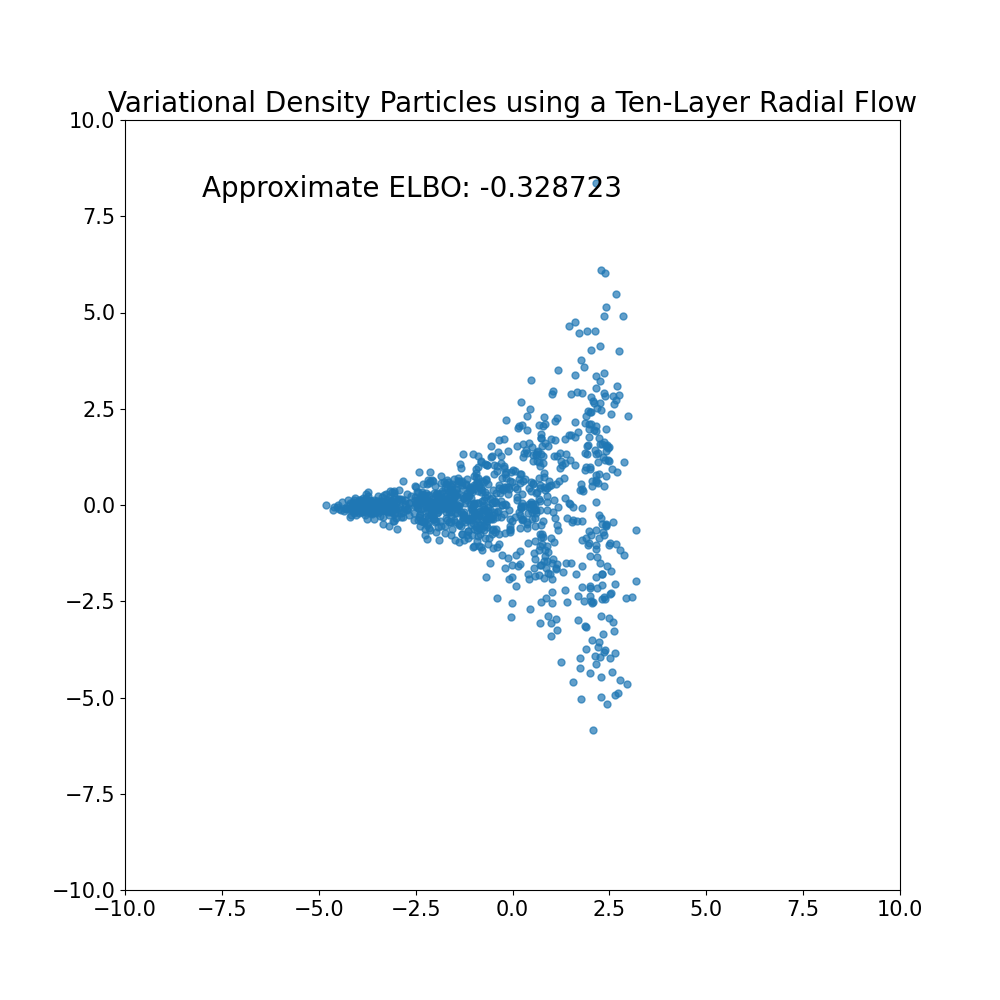}%
    \includegraphics[width=0.33\linewidth,trim=12mm 16mm 16mm 20mm, clip]{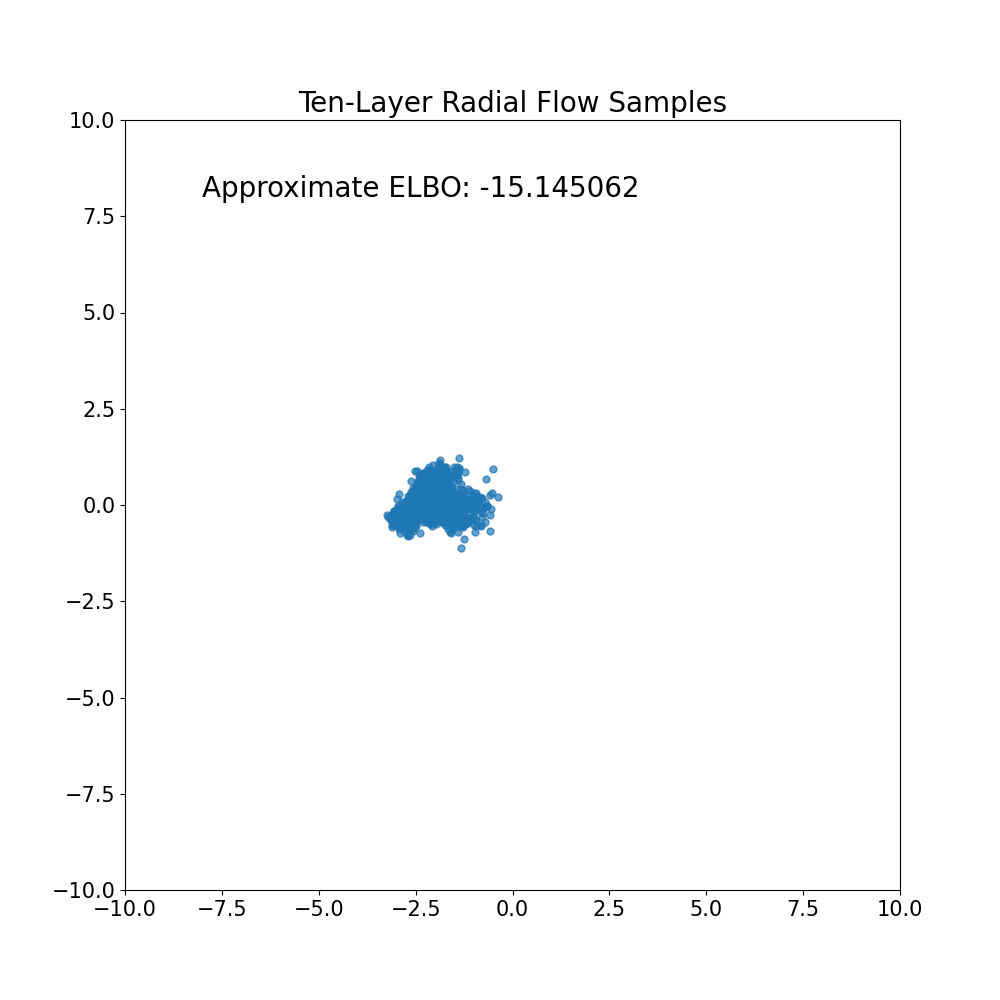}%
    \caption{Results obtained for the funnel posterior case in \eqref{eq: funnel} using the proposed particle flow-based normalizing flow \eqref{application: transformed_pf}, associated with the parameter flow \eqref{application: joint_parameter_flow}, where we show particles projected to the first two dimensions. The top left figure shows the ground-truth particles. The top middle figure shows the results obtained by the variational density $q(\bfx; \bftheta_{t=T})$ using a single triangular transformation. The top right figure shows the results obtained by optimizing only the triangular transformation. The results obtained by planar and radial transformation are shown in the second and third rows, respectively. In each row, the left and middle figures correspond to the results obtained from $q(\bfx; \bftheta_{t=T})$ using compositions of five and ten transformations, respectively, while the right figure shows the results obtained by optimizing only the composition of ten transformations.}
    \label{fig: funnel_results}
\end{figure}

\section{Conclusions}

We developed a variational formulation of the particle flow particle filter. We showed in Theorem~\ref{theorem: fr_sol} that the transient density used to derive the particle flow particle filter is a time-scaled solution to the Fisher-Rao gradient flow. Based on this observation, we first derived a Gaussian Fisher-Rao particle flow, which reduces to the Exact Daum and Huang (EDH) flow \citep{daum2010exact} under linear Gaussian assumptions. Next, we derived an approximated Gaussian mixture Fisher-Rao particle flow to approximate multi-modal behavior in the posterior density. We presented implementation strategies that make the computation of the flow parameters more efficient and ensure stable propagation of the Fisher-Rao flows. In the simulations, we illustrated the equivalence between the Gaussian Fisher-Rao particle flow and the EDH flow under linear Gaussian assumptions, that the approximated Gaussian mixture Fisher-Rao particle flow captures multi-modal behavior in the posterior, and that our Fisher-Rao particle flow methods achieve good accuracy in a Bayesian logistic regression task. To generalize the variational density beyond Gaussian families, we combined our Fisher–Rao particle flows with normalizing flow \citep{rezende2015variational}, resulting in a particle flow-based normalizing flow method. Through numerical simulations, we demonstrated that the proposed method achieves accurate sampling performance on a funnel posterior inference task. Future work will explore the application of the Fisher-Rao flows to robot state estimation problems and their extension to Lie groups, exploiting the geometry of the configuration space of robot systems.



\acks{The authors declare no competing interests. We gratefully acknowledge support from ONR Award N00014-23-1-2353 and NSF FRR CAREER 2045945.
}

\newpage

\bibliography{ref/VariationalParticleFlow}
\clearpage
\newpage
\appendix

\section{}\label{app:A}
In this appendix, we prove Lemma~12 from Section~8.2: \\
\\
{\bf Lemma (Valid Descent Direction) }{\it
    Consider the parameter flow $(67)$. If $\frac{\partial \bftheta_t}{\partial t} \neq \mathbf{0}$, then there exists $\epsilon > 0$ such that the following condition holds:
    \begin{equation}
        D_{KL}\left( q\left( \bfx; \bftheta_t + \alpha \frac{\partial \bftheta_t}{\partial t} \right) \middle\| p(\bfx | \bfz) \right) < D_{KL}\left( q\left( \bfx; \bftheta_t\right) \| p(\bfx | \bfz) \right), \quad \forall \alpha \in (0, \epsilon).
    \end{equation}
}
\begin{proof}
    The parameter flow $(67)$ can be expressed in the following compact form:
    \begin{equation}
        \frac{\partial \bftheta_t}{\partial t} = - W \nabla_{\bftheta_t} D_{KL}\left( q\left( \bfx; \bftheta_t\right) \| p(\bfx | \bfz) \right), \quad W = \begin{bmatrix} \calI^{-1}([\bftheta_u]_t) & \mathbf{0} \\ \mathbf{0} & \gamma I \end{bmatrix}.
    \end{equation}
    Since the Fisher information matrices in (23) and (39) are positive definite, and the scaling parameter $\gamma$ is strictly positive, the block-diagonal matrix $W$ is also positive definite. As a result, whenever $\frac{\partial \bftheta_t}{\partial t} \neq \mathbf{0}$, we have
    \begin{equation}
        \nabla^{\top}_{\bftheta_t} D_{KL}\left( q\left( \bfx; \bftheta_t\right) \| p(\bfx | \bfz) \right) \frac{\partial \bftheta_t}{\partial t} < 0.
    \end{equation}
    The desired result then follows directly from the standard descent-direction argument (see, e.g., Theorem 2.2 in \citet{nocedal2006numerical}).
\end{proof}

\section{}\label{app:B}
In this appendix, we prove Proposition~14 from Section~8.2: \\
\\
{\bf Proposition (Transformed Particle Dynamics Function) }{\it
    Let $\bfphi_u(\bfu, t)$ be a particle dynamics function for the base variational density satisfying:
    \begin{equation}\label{appendix: base_liouville}
        \nabla_t b(\bfu; [\bftheta_u]_t) = \frac{\partial b(\bfu; [\bftheta_u]_t)}{\partial [\bftheta_u]_t} \frac{\partial [\bftheta_u]_t}{\partial t} = - \nabla_{\bfu} \cdot \left(b(\bfu; [\bftheta_u]_t) \bfphi_u(\bfu, t) \right),
    \end{equation}
    for all $\bfu \in \bbR^n$ and $t \in [0, T]$, where the time derivative of the base variation density parameters $[\bftheta_u]_t$ is specified by the parameter flow $(67)$. Suppose the transformation $F(\bfu; [\bftheta_F]_t)$ is proper, i.e., the preimage $U = F^{-1}(X; [\bftheta_F]_t)$ of every compact set $X \subset \bbR^n$ is also compact. Furthermore, assume that $F(\bfu; [\bftheta_F]_t)$ is continuously differentiable in its parameters, and that the gradient of the KL divergence with respect to the transformation parameters $\nabla_{[\bftheta_F]_t} D_{KL}(b(\bfu; [\bftheta_u]_t) \| p(\bfu | \bfz; [\bftheta_F]_t) )$ is locally Lipschitz. Under these conditions, the transformed particle dynamics function induced by the time-varying transformation $F(\bfu; [\bftheta_F]_t)$ in $(68)$ is given by:
    \begin{equation}
        \bfphi_{F}(\bfx, t) = \left[ \frac{\partial F(\bfu; [\bftheta_F]_t)}{\partial t} + \nabla_{\bfu} F(\bfu; [\bftheta_F]_t) \bfphi_u(\bfu, t) \right] \Bigg|_{\bfu = F^{-1}(\bfx; [\bftheta_F]_t)},
    \end{equation}
    which satisfies the following Liouville equation in the sense of distributions:
    \begin{equation}
        \frac{\partial q(\bfx; \bftheta_t)}{\partial t} = - \nabla_{\bfx} \cdot \left( q(\bfx; \bftheta_t) \bfphi_F(\bfx, t) \right). 
    \end{equation}
}
\begin{proof}
    We first show that the particle dynamics function $\bfphi_u(\bfu, t)$ satisfies the Liouville equation \eqref{appendix: base_liouville} in the sense of distributions. Consider an arbitrary test function $\varphi(\bfu, t) \in C^{1}_{c}(\bbR^n \times (0, T))$, we have the following condition holds:
    \begin{equation}
        \int_0^T \int_{\bbR^n} \varphi(\bfu, t) \nabla_t b(\bfu; [\bftheta_u]_t) \d \bfu \d t + \int_0^T \int_{\bbR^n} \varphi(\bfu, t) \nabla_{\bfu} \cdot \left(b(\bfu; [\bftheta_u]_t) \bfphi_u(\bfu, t) \right) = 0. 
    \end{equation}
    Consider the first term, sine the test function $\varphi(\bfu, t)$ has a bounded support in $\bbR^n \times (0, T)$, integration by parts leads to:
    \begin{equation}
        \int_0^T \int_{\bbR^n} \varphi(\bfu, t) \nabla_t b(\bfu; [\bftheta_u]_t) \d \bfu \d t = - \int_0^T \int_{\bbR^n} \nabla_t \varphi(\bfu, t) b(\bfu; [\bftheta_u]_t) \d \bfu \d t.
    \end{equation}
    Similarly, applying Green’s theorem \citep{duistermaat_kolk_2004} together with Assumption~1 to the second term yields:
    \begin{equation}
        \int_0^T \int_{\bbR^n} \varphi(\bfu, t) \nabla_{\bfu} \cdot \left(b(\bfu; [\bftheta_u]_t) \bfphi_u(\bfu, t) \right) = -\int_0^T \int_{\bbR^n} b(\bfu; [\bftheta_u]_t) \nabla^{\top}_{\bfu} \varphi(\bfu, t) \bfphi_u(\bfu, t)  \d \bfu \d t.
    \end{equation}
    Combining the two identities yields the weak form of the Liouville equation:
    \begin{equation} \label{appendix: weak_base}
        \int_0^T \int_{\bbR^n} \left(\nabla_t \varphi(\bfu, t) + \bfphi_u^{\top}(\bfu, t) \nabla_{\bfu} \varphi(\bfu, t) \right) b(\bfu; [\bftheta_u]_t) \d \bfu \d t = 0.
    \end{equation}
    Now, consider a test function $\psi(\bfx, t) \in C^{1}_{c}(\bbR^n \times (0, T))$, and define its pullback under the transformation $F(\bfu; [\bftheta_F]_t)$ by $\tilde{\varphi}(\bfu, t) = \psi(F(\bfu; [\bftheta_F]_t), t)$. The pullback test function also belongs to $C^{1}_{c}(\bbR^n \times (0, T))$ since the assumed regularity of the transformation $F(\bfu; [\bftheta_F]_t)$ and the local Lipschitz continuity of the parameter flow $\partial [\bftheta_F]_t / \partial t$ ensure that the composition preserves $C^1$ regularity, while the properness of $F(\bfu; [\bftheta_F]_t)$ guarantees preservation of compact support under the pullback. Since \eqref{appendix: weak_base} holds for arbitrary $\varphi(\bfu, t) \in C^{1}_{c}(\bbR^n \times (0, T))$, we can conclude the following:
    \begin{equation}
        \int_0^T \int_{\bbR^n} \left(\nabla_t \psi(F(\bfu; [\bftheta_F]_t) + \bfphi_u^{\top}(\bfu, t) \nabla_{\bfu} \psi(F(\bfu; [\bftheta_F]_t) \right) b(\bfu; [\bftheta_u]_t) \d \bfu \d t = 0.
    \end{equation}
    According to the chain rule, we have:
    \begin{equation}
    \begin{split}
        &\nabla_t \psi(F(\bfu; [\bftheta_F]_t) = \nabla_t \psi(F(\bfu; [\bftheta_F]_t), t) + \nabla_t F(\bfu; [\bftheta_F]_t) [\nabla_{\bfx} \psi(\bfx, t)] |_{\bfx = F(\bfu; [\bftheta_F]_t)}, \\
        &\nabla_{\bfu} \psi(F(\bfu; [\bftheta_F]_t) = \nabla^{\top}_{\bfu} F(\bfu; [\bftheta_F]_t) [\nabla_{\bfx} \psi(\bfx, t)] |_{\bfx = F(\bfu; [\bftheta_F]_t)}. 
    \end{split}
    \end{equation}
    Substituting the expression above into the integrand of \eqref{appendix: weak_base} yields:
    \begin{equation}
    \begin{split}
        &\nabla_t \varphi(\bfu, t) + \bfphi_u^{\top}(\bfu, t) \nabla_{\bfu} \varphi(\bfu, t) = \nabla_t \psi(F(\bfu; [\bftheta_F]_t), t) \\
        &+ \left( \frac{\partial F(\bfu; [\bftheta_F]_t)}{\partial t} + \nabla_{\bfu} F(\bfu; [\bftheta_F]_t) \bfphi_u(\bfu, t) \right)^{\top} [\nabla_{\bfx} \psi(\bfx, t)] |_{\bfx = F(\bfu; [\bftheta_F]_t)}. 
    \end{split}
    \end{equation}
    Noting that $\d \bfx = | \det(\nabla_{\bfu} F(\bfu; [\bftheta_F]_t)) | \d \bfu$, applying the change of variables formula to \eqref{appendix: weak_base} yields $\int_0^T \int_{\bbR^n} \left(\nabla_t \psi(\bfx, t) + \bfphi_F^{\top}(\bfx, t) \nabla_{\bfx} \psi(\bfx, t) \right) q(\bfx; \bftheta_t) \d \bfx \d t = 0$.
\end{proof}

\section{}\label{app:C}
In this appendix, we discuss some practical considerations related to the approximation of the expectations of the gradient and Hessian of $V(\bfx)$. Recall the definition of $V(\bfx)$:
\begin{equation}
    V(\bfx) = \log(q(\bfx)) - \log(p(\bfx, \bfz)), 
\end{equation}
where $q(\bfx)$ is the variational density and $p(\bfx, \bfz))$ is the joint density. The derivation of the Fisher-Rao parameter flow (26) and (51) relies on estimating the following Gaussian expectations:
\begin{equation}
\begin{split}
    \bbE_{\bfx \sim \calN(\bfmu, \Sigma)}\left[\nabla_{\bfx}V(\bfx)\right] &= \Sigma^{-1} \bbE_{\bfx \sim \calN(\bfmu, \Sigma)}\left[ (\bfx - \bfmu) V(\bfx) \right], \\
    \bbE_{\bfx \sim \calN(\bfmu, \Sigma)}\left[\nabla^2_{\bfx}V(\bfx)\right] &= \Sigma^{-1} \bbE_{\bfx \sim \calN(\bfmu, \Sigma)}\left[ (\bfx - \bfmu)(\bfx - \bfmu)^{\top} V(\bfx) \right] \Sigma^{-1} \\
    & \quad - \Sigma^{-1} \bbE_{\bfx \sim \calN(\bfmu, \Sigma)}\left[ V(\bfx) \right].
\end{split}
\end{equation}
In practice, these expectations are approximated using a finite set of particles. As a result, any bias introduced in the empirical evaluation of the $V(\bfx)$ term directly propagates to the gradient and Hessian estimators, resulting in biased updates and degraded convergence behavior. This issue can be mitigated by introducing a correcting term when evaluating $V(\bfx)$. This correction stabilizes the magnitude of the estimated $V(\bfx)$ terms and improves the robustness of the resulting gradient flow, particularly in high-dimensional regimes where finite-sample effects become pronounced. Specifically, consider the following corrected $V(\bfx)$ term:
\begin{equation}
    \tilde{V}(\bfx) = V(\bfx) + c, 
\end{equation}
where $c$ is a constant correction term. We shall notice that this constant correction term does not influence the Gaussian expectations:
\begin{equation}
\begin{split}
    \bbE_{\bfx \sim \calN(\bfmu, \Sigma)}\left[\nabla_{\bfx}\tilde{V}(\bfx)\right] &= \Sigma^{-1} \bbE_{\bfx \sim \calN(\bfmu, \Sigma)}\left[ (\bfx - \bfmu) \tilde{V}(\bfx) \right] \\
    &= \Sigma^{-1} \bbE_{\bfx \sim \calN(\bfmu, \Sigma)}\left[ (\bfx - \bfmu) V(\bfx) \right] + \Sigma^{-1} \bbE_{\bfx \sim \calN(\bfmu, \Sigma)}\left[ (\bfx - \bfmu) c \right] \\
    &= \Sigma^{-1} \bbE_{\bfx \sim \calN(\bfmu, \Sigma)}\left[ (\bfx - \bfmu) V(\bfx) \right] = \bbE_{\bfx \sim \calN(\bfmu, \Sigma)}\left[\nabla_{\bfx}V(\bfx)\right], \\
    \bbE_{\bfx \sim \calN(\bfmu, \Sigma)}\left[\nabla^2_{\bfx}\tilde{V}(\bfx)\right] &= \Sigma^{-1} \bbE_{\bfx \sim \calN(\bfmu, \Sigma)}\left[ (\bfx - \bfmu)(\bfx - \bfmu)^{\top} \tilde{V}(\bfx) \right] \Sigma^{-1}  - \Sigma^{-1} \bbE_{\bfx \sim \calN(\bfmu, \Sigma)}\left[ \tilde{V}(\bfx) \right] \\
    &= \Sigma^{-1} \bbE_{\bfx \sim \calN(\bfmu, \Sigma)}\left[ (\bfx - \bfmu)(\bfx - \bfmu)^{\top} V(\bfx) \right] \Sigma^{-1} - \Sigma^{-1} \bbE_{\bfx \sim \calN(\bfmu, \Sigma)}\left[ V(\bfx) \right] \\
    &+ \Sigma^{-1} \bbE_{\bfx \sim \calN(\bfmu, \Sigma)}\left[ (\bfx - \bfmu)(\bfx - \bfmu)^{\top} c \right] \Sigma^{-1} - \Sigma^{-1} \bbE_{\bfx \sim \calN(\bfmu, \Sigma)}\left[ c \right] \\
    &= \Sigma^{-1} \bbE_{\bfx \sim \calN(\bfmu, \Sigma)}\left[ (\bfx - \bfmu)(\bfx - \bfmu)^{\top} V(\bfx) \right] \Sigma^{-1} - \Sigma^{-1} \bbE_{\bfx \sim \calN(\bfmu, \Sigma)}\left[ V(\bfx) \right] \\
    &+ c\Sigma^{-1} - c \Sigma^{-1} = \bbE_{\bfx \sim \calN(\bfmu, \Sigma)}\left[\nabla^2_{\bfx}V(\bfx)\right].
\end{split}
\end{equation}
With a finite number of particles $\{ \bfx_i \}_{i=1}^M$, the equations above lead to:
\begin{equation}
\begin{split}
    \bbE_{\bfx \sim \calN(\bfmu, \Sigma)}\left[\nabla_{\bfx}\tilde{V}(\bfx)\right] &\approx \frac{1}{M} \Sigma^{-1} \sum_{i=1}^{M}(\bfx_i - \bfmu) V(\bfx_i) + c \frac{1}{M} \Sigma^{-1} \sum_{i=1}^{M}(\bfx_i - \bfmu) \\
    \bbE_{\bfx \sim \calN(\bfmu, \Sigma)}\left[\nabla^2_{\bfx}\tilde{V}(\bfx)\right] &\approx  \frac{1}{M} \Sigma^{-1} \sum_{i=1}^{M} (\bfx_i - \bfmu)(\bfx_i - \bfmu)^{\top} V(\bfx_i)  \Sigma^{-1} - \frac{1}{M} \Sigma^{-1} \sum_{i=1}^{M} V(\bfx_i) \\
    &+ c\frac{1}{M} \Sigma^{-1} \sum_{i=1}^{M} (\bfx_i - \bfmu)(\bfx_i - \bfmu)^{\top} \Sigma^{-1} - c \Sigma^{-1}.
\end{split}
\end{equation}
Let $\tilde{\bfmu} = \frac{1}{M} \sum_{i=1}^{M} \bfx_i$ and $\tilde{\Sigma} = \frac{1}{M} \sum_{i=1}^{M} (\bfx_i - \bfmu)(\bfx_i - \bfmu)^{\top}$ denote the empirical mean and covariance computed from the particle set $\{ \bfx_i \}_{i=1}^M$. Taking $c = -\frac{1}{M} \sum_{i=1}^{M} V(\bfx_i)$, the previous identities lead to the following approximations:
\begin{equation}
\begin{split}
    \bbE_{\bfx \sim \calN(\bfmu, \Sigma)}\left[\nabla_{\bfx}\tilde{V}(\bfx)\right] &\approx \frac{1}{M} \Sigma^{-1} \sum_{i=1}^{M}(\bfx_i - \tilde{\bfmu}) V(\bfx_i), \\
    \bbE_{\bfx \sim \calN(\bfmu, \Sigma)}\left[\nabla^2_{\bfx}\tilde{V}(\bfx)\right] &\approx \frac{1}{M} \Sigma^{-1} \sum_{i=1}^{M} \left((\bfx_i - \bfmu)(\bfx_i - \bfmu)^{\top} - \tilde{\Sigma} \right)V(\bfx_i)  \Sigma^{-1}
\end{split}
\end{equation}
These expressions show that, with $c = -\frac{1}{M} \sum_{i=1}^{M} V(\bfx_i)$ approximating the expectation terms using a finite number of samples is equivalent to evaluating the difference terms using the empirical mean and covariance. 

\end{document}